\documentclass[journal]{IEEEtai}
\usepackage{url}
\usepackage[utf8]{inputenc}
\usepackage[small]{caption}
\usepackage{graphicx}
\usepackage{booktabs}
\usepackage{algorithm}
\usepackage{cite}
\usepackage{algorithmic}
\usepackage{color}
\usepackage{multirow}
\usepackage{bm}

\usepackage{booktabs}
\usepackage{algorithm,algorithmic}

\usepackage{subfigure}
\usepackage{amsmath,amssymb, amsthm,amsfonts}

\newtheorem{theorem}{Theorem}

\ifCLASSINFOpdf
\else
\fi
\begin{document}
%\title{Collaborative Update and Double Regularized Regressions for Auto-weighted Noisy and Incomplete Multi-view Clustering}
\title{ANIMC: A Soft Approach for Auto-weighted Noisy and Incomplete Multi-view Clustering}

\author{Xiang~Fang,
        Yuchong~Hu,~\IEEEmembership{Member,~IEEE,}
        Pan~Zhou,~\IEEEmembership{Senior Member,~IEEE,}
        and~Dapeng~Oliver~Wu,~\IEEEmembership{Fellow,~IEEE}% <-this % stops a space
\thanks{This work is supported by National Natural Science Foundation of China (NSFC) under grant no. 61972448. (\emph{Corresponding author: Pan Zhou}.)}
\thanks{X. Fang is with the Hubei Engineering Research Center on Big Data Security, School of Cyber Science and Engineering Huazhong University of Science and Technology, also with the School of Computer Science and Technology, Huazhong University of Science and Technology, Wuhan 430074, China (e-mail: xfang9508@gmail.com).}% <-this % stops a space
\thanks{Y. Hu is with the School of Computer Science and Technology, Key Laboratory of Information Storage System Ministry of Education of China, Huazhong University of Science and Technology, Wuhan 430074, China (e-mail: yuchonghu@hust.edu.cn).}% <-this % stops a space
\thanks{P. Zhou is with the Hubei Engineering Research Center on Big Data Security, School of Cyber Science and Engineering, Huazhong University of Science and Technology, Wuhan 430074, China (e-mail: panzhou@hust.edu.cn).}% <-this % stops a space
\thanks{D. O. Wu is with the Department of Electrical and Computer Engineering, University of Florida, Gainesville, FL 32611, USA (e-mail: dpwu@ieee.org).}
\thanks{$\copyright$ 2021 IEEE. Personal use of this material is permitted.  Permission from IEEE must be obtained for all other uses, in any current or future media, including reprinting/republishing this material for advertising or promotional purposes, creating new collective works, for resale or redistribution to servers or lists, or reuse of any copyrighted component of this work in other works.}
}% <-this % stops a space

\maketitle

\begin{abstract}
Multi-view clustering has wide real-world applications because it can process data from multiple sources. However, these data often contain missing instances and noises, which are ignored by most multi-view clustering methods. Missing instances may make these methods difficult to use directly, and noises will lead to unreliable clustering results. In this paper, we propose a novel Auto-weighted Noisy and Incomplete Multi-view Clustering approach (ANIMC) via a \emph{soft} auto-weighted strategy and a doubly \emph{soft} regular regression model. Firstly, by designing adaptive semi-regularized nonnegative matrix factorization (adaptive semi-RNMF), the soft auto-weighted strategy assigns a \emph{proper} weight to each view and adds a soft boundary to balance the influence of noises and incompleteness. Secondly, by proposing $\theta$-norm, the doubly soft regularized regression model adjusts the sparsity of our model by choosing different $\theta$.
Compared with previous methods, ANIMC has three unique advantages: 1) it is a soft algorithm to adjust our approach in different scenarios, thereby improving its generalization ability; 2) it automatically learns a proper weight for each view, thereby reducing the influence of noises; 3) it performs doubly soft regularized regression that aligns the same instances in different views, thereby decreasing the impact of missing instances. Extensive experimental results demonstrate its superior advantages over other state-of-the-art methods.
%We provide codes for all of our experiments in \url{https://drive.google.com/file/d/1OgI9y3Mfro6yDl-TLBqIeW25O5DHa1ma/view?usp=sharing}.
%Compared with existing methods, ANIMC has two unique advantages: 1) it automatically learns a \emph{proper} weight for each noisy and incomplete view and the \emph{optimal} latent feature matrix, thereby reducing the influence of noises; and 2) it performs a doubly regularized regression that aligns the same instances in different views and pushes the common latent feature matrix towards its consensus, thereby decreasing the impact of missing instances. Extensive Experimental experiments demonstrate its advantages over other state-of-the-art methods.
\end{abstract}

%which aims to evaluate the real repair performance of DRC, and implement and
%deploy it atop an HDFS cluster testbed.
%We design two families of explicit DRC, which not only achieve the minimal
%cross-rack repair traffic, but also are practical to be implemented in real
%storage systems because they are constructed on top of reed-solomon codes.
%We show that our DoubleR conforms to DRC theoretical results and achieves cross-rack repair traffic savings when compared to state-of-the-art regenerating codes.
%Data centers are tiered by nature, and the erasure-coded repair performance is bottlenecked by the critical network bandwidth across tiers.
%
%Motivation:
%
%oversubscription: modern data intensive computing clusters typically have oversubscription values ranging from 3:1 to 10:1 from the racks to the core.
%
%DRC: reduce the cross-rack bandwidth.
%
%In this paper:
%
%DRC theory:
%
%1 further reduce the intra-rack bandwidth
%
%2 further tolerate rack failures optimally
%
%3 deterministic DRC design
%
%DRC framework: real repair performance of DRC
%
%Experiments: Hadoop, Ceph, cross-rack, cross-datecenter.

\begin{IEEEImpStatement}
As an effective method to process data from multiple sources, multi-view clustering has attracted more and more attention. However, most previous works ignore missing instances and noises in original multi-view data, which limits the applications of these works. By a soft approach, our proposed ANIMC can effectively reduce the negative influence of missing instances and noises.
Moreover, ANIMC outperforms the state-of-the-art works by about 20\% in representative cases.
With satisfactory performance on multiple real-world datasets, ANIMC has wide potential applications including the analysis of multilingual document and image datasets.
\end{IEEEImpStatement}


\ifCLASSOPTIONpeerreview
 \fi
\IEEEpeerreviewmaketitle

\section{Introduction}\label{section:intro}
\IEEEPARstart{R}{eal-world} data \cite{liu2018fast,yang2020shearlet,fang2021} often come from different sources, which are called \emph{multi-view data} \cite{blum1998combining,gong2016multi,tan2019individuality}.
For example, the collected images \cite{ma2019deep,zhou2017tensor,yan2020deep,9229197} can be represented by different visual descriptors (i.e., different views)~\cite{wang2020weighted,liu2023exploring,wang2025taylor,fang2026towardsicml,kuai2026dynamic,wang2025point,fang2025your,zhang2025monoattack,fang2023hierarchical,liu2024towards,yang2025eood,fang2022multi,fang2026cogniVerse,lei2025exploring,fang2023you,wang2025dypolyseg,fang2025hierarchical,yan2026fit,fang2025adaptive,wang2026topadapter,cai2025imperceptible,fang2026slap,wang2026reasoning,fang2026immuno,wang2026biologically,fang2026disentangling,wang2025reducing,fang2026advancing,fang2026unveiling,wang2026from,liu2023conditional,liu2026attacking,fang2026rethinking,wang2025seeing,fang2026towards,fang2025multi,fang2024fewer,liu2024pandora,fang2024multi,fang2025turing,fang2024not,liu2023hypotheses,fang2024rethinking,liu2024unsupervised,fang2023annotations,xiong2024rethinking,fang2020v,wang2025prototype,zhang2025manipulating,fang2026align,tang2024reparameterization,fang2025adaptivetai,tang2025simplification,fang2021uimc,cai2026towards,fang2020double}, like CTM \cite{chen2011prediction,shao2020intra,chambolle2018stochastic}, GIST \cite{min2017delicious,liu2017robust,jiang2019discrete}, SIFT \cite{zhou2016bilevel}, etc; web pages can be represented by different kinds of features based on text and hyper-links \cite{xu2013survey,sun2013survey,lu2013unified}.
Integrating the information from different views can help us analyze the data in a comprehensive manner \cite{xu2013survey,sun2013survey,yang2020fade}, which motivates multi-view clustering methods \cite{li2017implicit,zhang2017latent,he2020mv,jia2020semi}.
%In the field of multi-view learning, multi-view clustering has attracted more and more attention due to eliminating the high cost of time and money on labeling multi-view data.
The purpose of multi-view clustering is to adaptively partition data into their respective groups by fully integrating the information from multiple views \cite{lu2013unified}. Up to the present, many multi-view clustering methods have been proposed, like space-based methods \cite{zhou2018multiview,zhang2017latent,kang2021structured}, graph-based methods~\cite{ng2002spectral,nie2017self,kang2018self}, nonnegative matrix factorization (NMF) based methods~\cite{liu2013multi,liang2020multi}, etc. Although most multi-view clustering methods can solve some problems in real-world applications, they still face two major problems.

One problem is the missing instances in multi-view data~\cite{liu2018ls}. Most multi-view clustering methods require that each view has no missing instances. But real-world multi-view data always contain missing instances, which leads to the incomplete multi-view clustering problem. For example, in the camera network, some cameras may temporarily fail due to a power outage, which will result in missing instances.
%blood tests and images scanned by the magnetic resonance can be regarded as two views of each person in disease diagnosis, and it often happens that some individuals would like to only perform one test.
As such, this incompleteness may lead to the lack of columns or rows in the view matrix, which will result in the degradation or failure of previous methods.
%In most applications, the characteristics of incomplete multi-view data are as follows: (i) the common instances presented in all views can be used to extract the consistent information from different views and to integrate these views; (ii) these instances that only exist in partial views can help learn the unique information of the corresponding views.

The other problem is the noises in multi-view data. For instance, real-world images often contain some noises~\cite{chang2000wavelet,zhu2020structured,el2020blind}. For instance, many landscape images often contain fog and rain, which are common noises in image processing.
If we directly cluster these landscape images, these noises may cause deviations in the calculation, which will damage the performance.

Up to now, there is no effective method to solve these two problems at the same time.
To tackle the first problem, several incomplete multi-view clustering methods have been proposed~\cite{shao2015multiple,hu2018doubly,wen2019unified}.
But they ignore the influence of noises (i.e., the second problem), which are ubiquitous in the real-world clustering tasks~\cite{nie2018auto,gu2017empirical}. Specifically, these methods directly conduct the procedure by constructing a basis matrix for each view and a common latent subspace for all views that are rarely modified. But real-world datasets always contain noises that result in unreliable basis matrices and unavailable latent subspace. These noises will affect the clustering calculation, which may lead to inaccurate clustering results.
For the second problem, a simple solution is to assign a proper weight to each view~\cite{nie2017self,zhu2020self,huang2020auto,shi2020auto}. However, previous incomplete multi-view clustering methods are difficult to effectively weight different views due to the following challenges~\cite{li2014partial,zhao2016incomplete,shao2015multiple,hu2018doubly,wen2019unified}:
\begin{itemize}
  \item The combined effects of noises and incompleteness make these incomplete multi-view clustering methods obtain unsatisfactory clustering results. When we cluster these multi-view datasets, both missing instances and noises will influence the weights of these views. Previous methods are difficult to consider the effects of missing instances and noises simultaneously. Besides, when the missing rate and the noisy rate changes, these methods are difficult to balance this effect.
  \item As the missing rate increases, the availability of each view also changes. If we weigh each view based on parameters, the selection of parameters for each case will have a high cost of time and it will be difficult for us to assign a proper weight to each view. To our best knowledge, it is still an open problem to select the optimal values for these parameters in different clustering tasks, which limits the application of parameter-weighted incomplete multi-view clustering methods.
  \item Defined in \cite{smola2001regularized}, the generalization ability of clustering methods is valued by many works \cite{tao2005new,deng2021tri,carbonnelle2020intraclass}. Strong generalization ability often means that a method can keep satisfactory performance in different datasets.
      To effectively cluster different data sets, we often need different objective functions. A clustering method with strong generalization ability often designs different objective functions rather than a single objective function. Thus, a series of effective objective functions should be designed to improve the generalization ability. The generalization ability of previous methods is limited because these methods are based on a single objective function. The single objective function often performs well only in some cases.
\end{itemize}
Therefore, multi-view clustering still faces significant challenges.
%An effective way to solve the problem is to distinguish the availability of different views to reduce the impact of noisy views.
In this paper, we propose a novel Auto-weighted Noisy and Incomplete Multi-view Clustering (ANIMC) approach to meet these challenges. ANIMC is a joint of a \emph{soft} auto-weighted strategy and a doubly \emph{soft} regularized regression model\footnote{A soft model is a series of functions with similar structures.}.
First, by designing adaptive semi-regularized nonnegative matrix factorization (adaptive semi-RNMF), we propose a soft auto-weighted strategy to assign a \emph{proper} weight to each view.
%Also, a soft bound is added to the strategy.
%In the strategy, a soft bound is adopted to balance the influence of noises and incompleteness.
Second, by devising $\theta$-norm, we perform doubly soft regularized regression to align the same instances in different views.
%Moreover, different $\theta$ can be chosen to adjust the sparsity of our model.
%First, ANIMC performs adaptive semi regularized nonnegative matrix factorization (adaptive semi-RNMF) to construct four matrices (common feature matrix, consensus feature matrix, basis matrix, and consensus basis matrix), and with the help of Lagrangian function, both the optimal common feature matrix and each view's proper weight are learned automatically and updated adaptively.
%Further, ANIMC adopts a doubly regularized regression based on $\theta$-norm to align the same instances in different views.
%(i) $L_F$-Norm regularized regression based on the common feature matrix and the consensus feature matrix for all the views, and (ii) $L_F$-Norm regularized regression based on the basis matrix and the consensus basis matrix for each view.
The main contributions of our proposed ANIMC approach are summarized as follows:
\begin{itemize}
  \item To the best of our knowledge, ANIMC is the \emph{first} auto-weighted soft approach for noisy and incomplete multi-view clustering. The soft approach can keep high-level clustering performance on different datasets, which verifies its strong generalization ability.
  \item By proposing adaptive semi-RNMF, ANIMC adaptively assigns a proper weight to each view, which diminishes the effect of noises. Also, ANIMC adds a soft boundary to the soft auto-weighted strategy, which balances the influence of noises and missing instances.
%  To diminish the effect of noises, it learns an optimal common latent feature matrix and the proper weights simultaneously and updates them collaboratively via adaptive semi-RNMF.
  \item By performing a doubly soft regularized regression model based on $\theta$-norm, ANIMC aligns the same instances in all views, which decreases the impact of missing instances. Besides, different $\theta$ can be chosen to adjust the sparsity of the model.
  \item Experimental results on six real-world datasets show that ANIMC significantly outperforms state-of-the-art methods. Moreover, by making the falling direction of the objective function closer to the gradient direction, ANIMC tremendously accelerates the convergence speed by a four-step alternating iteration procedure.
%      Compared with existing incomplete multi-view clustering methods, it tremendously accelerates the convergence speed by a four-step alternating iteration procedure. The reason is that ANIMC can make the falling direction of the objective function closer to the gradient direction.
\end{itemize}

The rest of the paper is organized as follows. Related work is covered in Section~\ref{section:related}.
%By combining RMF and semi-NMF,
The background is described in Section~\ref{section:back}. The proposed ANIMC approach is designed in Section~\ref{section:meth}.
Performance evaluations are reported in Section~\ref{section:exp}. We conclude the paper in Section~\ref{section:con}.

\section{Related Work}  \label{section:related}
%Since the most relevant work for our research is low-rank matrix factorization (MF), we will review these algorithms in this section. MF is a common technique to learn the geometry of data for clustering. Up to now, many matrix factorization methods have been proposed, such as regularized matrix factorization (RMF) and semi-nonnegative matrix factorization (semi-NMF), etc. Among them, RMF and semi-NMF are two simple and effective methods, which makes them widely used in the field of matrix completion.
Since the related work for our approach is incomplete multi-view clustering, we will review some incomplete multi-view clustering algorithms in this section.
Up to now, several incomplete multi-view clustering algorithms have been proposed~\cite{li2014partial,zhao2016incomplete,shao2015multiple,hu2018doubly,wen2019unified}. In this section, we describe some algorithms that are most relevant to our approach. Besides ignoring the influence of noises (see Section~\ref{section:intro}), each of these algorithms has distinct drawbacks. These drawbacks often lead to unsatisfactory clustering results, which limits their real-world applications.
%These algorithms can be classified into two categories: incomplete two-view clustering and incomplete multi-view clustering with more than two views.

%(i) Incomplete two-view clustering aims to cluster incomplete data with two views and the related algorithms are PVC and IMG.
%As the first work for incomplete multi-view clustering, \cite{li2014partial} proposes PVC to learn a common latent space from two incomplete views for clustering. Based on the common latent space, PVC can cluster incomplete two-view datasets.
%%via NMF~\cite{lee1999learning,cai2008sparse} and $L_{1}$-norm regularization.
%%But PVC simply projects instances from each view into a common subspace and overlooks the global information among the two views.
%By extending PVC, \cite{zhao2016incomplete} proposes IMG to remove the nonnegative constraint for better clustering results.
%%To improve the clustering performance, IMG~\cite{zhao2016incomplete} extends PVC and removes the nonnegative constraint to simplify optimization.
%However, PVC and IMG can only integrate two incomplete views, which limits their application in incomplete multi-view data.

%But both PVC and IMG can only solve the problem of two-view incomplete multi-view clustering, which makes these two algorithms limited when solving incomplete multi-view clustering problems.
%(ii)
%Incomplete multi-view clustering aims to cluster incomplete data with multiple views and the most relevant algorithms are MIC, DAIMC and UEAF.
To integrate more than two incomplete views, \cite{shao2015multiple} proposes MIC by filling the missing instances in each view with average feature values. After this filling, MIC can learn a common latent subspace via weighted NMF and $L_{2,1}$-norm regularization from multiple incomplete views.
%To solve the problem of incomplete clustering based on more than two views, MIC~\cite{shao2015multiple} first fills the missing instances in each incomplete view with average feature values, then learns a common latent subspace based on weighted NMF and $L_{2,1}$-norm regularization.
However, if we cluster multi-view datasets with relatively large missing rates, this simply filling will result in unsatisfactory results.
To align the instance information, \cite{hu2018doubly} proposes DAIMC by extending MIC via weighted semi-NMF~\cite{ding2010convex} and $L_{2,1}$-norm regularized regression. To obtain the robust clustering results, \cite{wen2019unified} proposes UEAF by the unified common embedding aligned with incomplete views inferring approach.
But DAIMC and UEAF rely too much on alignment information, which limits their application in the datasets without enough alignment information. To simultaneously clustering and imputing the incomplete base clustering matrices, \cite{liu2019efficient} proposes EE-IMVC by solving the resultant optimization problem. To obtain better clustering performance, \cite{liu2020efficient} proposes EE-R-IMVC by extending EE-IMVC. Both EE-IMVC and EE-R-IMVC need to learn a consensus clustering matrix from an incomplete multi-view dataset. When the dataset contains some noises, EE-IMVC and EE-R-IMVC are difficult to learn the correct consensus clustering matrix, which will limit their application.
To learn the consistent information between different views and the unique information of each view, V$^3$H decomposes each subspace into a variation matrix for the corresponding view and a heredity matrix for all the views to represent the unique information and the consistent information respectively. As the first effective method to cluster multiple views with different incompleteness, UIMC is proposed to integrate these unbalanced incomplete views for clustering.

%. When clustering the dataset without enough alignment information, DAIMC and UEAF always obtain unsatisfactory performance because the loss of alignment information will reduce the availability of their models.
%Therefore, the previous method is difficult to obtain good clustering performance.
%assign a proper weight to each view adaptively~\cite{nie2018auto}.
%
%In fact, some incomplete multi-view clustering methods (e.g., MIC and DAIMC) have already adopted weight assignment to views, but \emph{they assign large weights to noisy incomplete views} which causes a large deviation of results. It is because these methods define each view's weight as an indicator matrix that assigns the missing instances zero weights and the presented instances one weights in each view, making the noisy views have large weights in the calculation. Obviously, both MIC and DAIMC assign a proper weight to each view.
%\subsection{Incomplete Multi-view Clustering}
%\label{subsection:incomplete}

%Therefore, our main idea is to design \emph{an adaptive method that assigns proper weights for noisy and incomplete views automatically}, such that the incomplete views with noises have small weights, thereby reducing the impact of noises on incomplete multi-view clustering.

\section{Background}  \label{section:back}
For convenience, we first define some notations throughout the paper. Then, we introduce two matrix factorization methods for single-view clustering: Regularized Matrix Factorization (RMF, in Section~\ref{subsubsection:rmf}) and semi-nonnegative Matrix Factorization (semi-NMF, in Section~\ref{subsubsection:snmf}). By extending RMF and semi-NMF, we design semi-RNMF (in Section~\ref{subsection:semirnmf}) and leverage semi-RNMF for incomplete multi-view clustering (in Section~\ref{subsection:self}).

%We first define some notations throughout the paper.
\noindent\textbf{Notation}: For a multi-view dataset $\{\bm{X}^{(v)}\}_{v=1}^m\in \mathbb{R}^{d_{v} \times n}$ with $n$ instances and $m$ views, we can cluster these instances into $c$ clusters, where $d_{v}$ is the feature dimension of the $v$-th view; $[n]\overset{\text{def}}{=}\{1,2,\ldots,n\}$; $\bm{U}^{(v)}\in \mathbb{R}^{d_{v}\times c}$ is the basis matrix of the $v$-th view, $\bm{V}\in \mathbb{R}^{n\times k}$ is the common latent feature matrix of all the views; for any matrix $\bm{B}$, $\bm{B}_{i,j}$ is the element in its $i$-th row and $j$-th column, and $\bm{B}_{i,:}$ is its $i$-th row; $\bm{I}$ denotes the identity matrix; $\bm{g}^{(v)}\in \mathbb{R}^{n\times 1}$ is an $n$-dimensional vector to indicate the incompleteness of the $v$-th view, and we diagonalize $g^{(v)}$ to the diagonal matrix $\bm{T}^{(v)}\in \mathbb{R}^{n\times n}$; $||\cdot||_F$ is F-norm; $||\cdot||_{2,1}$ is $L_{2,1}$-norm; $||\cdot||_{\theta}$ is $\theta$-norm; for any number $q$, $q^+$ is its right limit; $\alpha$ and $\beta$ are the nonnegative parameters; $\theta$ and $r$ are the parameters adjustable as needed.
\begin{figure*}[t]
\centering
\includegraphics[width=1\textwidth]{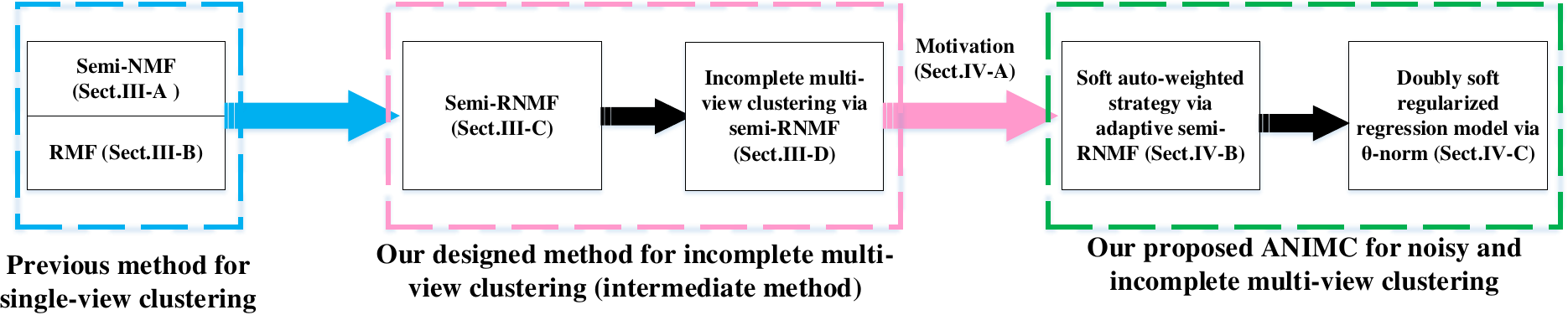}
\caption{The structure of the paper. ``Sect.'' denotes ``Section''.}
\label{fig:framework}
\end{figure*}
%$\bm{1}_n$ is an $n$-dimensional column vector and each element of the vector is 1.
\subsection{Regularized Matrix Factorization}
\label{subsubsection:rmf}
%\subsubsection{Regularized Matrix Factorization}
%\label{subsubsection:rmf}
As a popular latent feature learning method,  regularized matrix factorization (RMF) \cite{li2009relation,gunasekar2017implicit,qi2018unsupervised} outperforms other latent feature learning methods based on KNN or co-clustering. For a data matrix $\bm{X}\in \mathbb{R}^{d\times n}$, RMF approximates $\bm{X}$ with a matrix $\bm{V} \in \mathbb{R}^{n\times c}$ and a matrix $\bm{U} \in \mathbb{R}^{d\times c}$. Therefore, the objective function is as follows:
\begin{align}\label{danrmf}
\min\limits_{\bm{U},\bm{V}} ||\bm{X}-\bm{U}\bm{V}^T||_F^2+\alpha(||\bm{U}||_F^2+||\bm{V}||_F^2),
\end{align}
where $\alpha$ is a nonnegative parameter. For ease of description, we name $\bm{U}$ as the \emph{basis matrix} and $\bm{V}$ as the \emph{latent feature matrix}.

Note that Eq.~\eqref{danrmf} is a biconvex problem. Therefore, it is unrealistic to expect an algorithm to find the global optimal solution. Similar to \cite{li2009relation,gunasekar2017implicit}, we can update $\bm{U}$ and $\bm{V}$ by
%easily get the updating schemes to find the locally optimal solution for this problem as follows:

\noindent\textbf{(i)} Update $\bm{U}$ (while fixing $\bm{V}$) by $\bm{U}=\bm{X}\bm{V}(\alpha\bm{I}+\bm{V}^T\bm{V})^{-1}$.
%using the scheme
%\begin{align}\label{gengurmf}
%\bm{U}=\bm{X}\bm{V}(\alpha\bm{I}+\bm{V}^T\bm{V})^{-1}.
%\end{align}

\noindent\textbf{(ii)} Update $\bm{V}$ (while fixing $\bm{U}$) by $\bm{V}=\bm{X}^T\bm{U}(\alpha\bm{I}+\bm{U}^T\bm{U})^{-1}$.
%\begin{align}\label{gengvrmf}
%\bm{V}=\bm{X}^T\bm{U}(\alpha\bm{I}+\bm{U}^T\bm{U})^{-1}.
%\end{align}
%Semi-NMF based method gives K-means clustering interpretations when $VV^T=I$.
%Regularized matrix factorization is known to be one of the most successful methods for rating prediction outperforming other methods like pearson-correlation based kNN or co-clustering Regularized Matrix Factorization (RMF) is
\subsection{Semi-nonnegative Matrix Factorization}
\label{subsubsection:snmf}
%\subsubsection{Semi-nonnegative Matrix Factorization}
%\label{subsubsection:snmf}
Since semi-nonnegative matrix factorization (semi-NMF)~\cite{ding2010convex} can extract the latent feature information from the data, it has been widely used in single view clustering. For a data matrix $\bm{X}\in \mathbb{R}^{d\times n}$, semi-NMF approximates $\bm{X}$ with a nonnegative latent feature matrix $\bm{V} \in \mathbb{R}^{n\times c}$ and a basis matrix $\bm{U}\in\mathbb{R}^{d\times c}$. Since $\bm{U}$ can be negative, semi-NMF can handle negative input, which extends NMF. The framework of semi-NMF is %as follows:
\begin{align}\label{danseminmf}
&\min\limits_{\bm{U},\bm{V}} ||\bm{X}-\bm{U}\bm{V}^T||_F^2  \nonumber\\
&\mbox{s.t. } \bm{V}_{i,j} \geq 0, i\in[n], j\in[k].
\end{align}
Similar to RMF, the objective function of Eq.~\eqref{danseminmf} is biconvex. \cite{ding2010convex} proposes an iterative updating algorithm to find the locally optimal solution as follows:

\noindent\textbf{(i)} Update $\bm{U}$ (while fixing $\bm{V}$) by $\bm{U}=\bm{X}\bm{V}(\bm{V}^T\bm{V})^{-1}$.
%using the scheme
%\begin{align}\label{gengurmf}
%\bm{U}=\bm{X}\bm{V}(\bm{V}^T\bm{V})^{-1}.
%\end{align}

\noindent\textbf{(ii)} Update $\bm{V}$ (while fixing $\bm{U}$) by
\begin{align}\label{gengvrmf}
\bm{V}_{i,j}=\bm{V}_{i,j}\cdot\sqrt{\frac{(\bm{X}^T\bm{U})_{i,j}^{+}+[\bm{V}(\bm{U}^T\bm{U})]_{i,j}^{-}}{(\bm{X}^T\bm{U})_{i,j}^{-}+[\bm{V}(\bm{U}^T\bm{U})]_{i,j}^{+}}},
\end{align}
where $\bm{B}_{i,j}^{+}=(|\bm{B}_{i,j}|+\bm{B}_{i,j})/2$, $\bm{B}_{i,j}^{-}=(|\bm{B}_{i,j}|-\bm{B}_{i,j})/2$,
%we separate the positive and negative parts of a matrix $\bm{A}$ as:
%\begin{align}\label{A}
%\bm{B}_{i,j}^{+}&=(|\bm{B}_{i,j}|+\bm{B}_{i,j})/2,\bm{B}_{i,j}^{-}&=(|\bm{B}_{i,j}|-\bm{B}_{i,j})/2,
%\end{align}
where $\bm{B}^{+}$ is the positive part of a matrix $\bm{B}$ and $\bm{B}^{-}$ is the negative part of a matrix $\bm{B}$. Note that for any $\bm{B}_{i,j}$, we have $\bm{B}_{i,j}^{+}>0$ and $\bm{B}_{i,j}^{-}$. Thus, the separate can ensure the square root in Eq.~\eqref{gengvrmf} is meaningful. In addition, $\bm{B}_{i,j}^{+}+\bm{B}_{i,j}^{-}=|\bm{B}_{i,j}|$ and $\bm{B}_{i,j}^{+}-\bm{B}_{i,j}^{-}=\bm{B}_{i,j}$.
\subsection{Extending to Semi-RNMF}
\label{subsection:semirnmf}
In real-word applications, we often predict clusters by RMF and extract the latent feature information by semi-NMF. To improve the clustering performance, a natural idea is to combine RMF and semi-NMF, so we propose the semi-RNMF framework as follows:
\begin{align}\label{dansrmf}
&\min\limits_{\bm{U},\bm{V}} ||\bm{X}-\bm{U}\bm{V}^T||_F^2+\alpha(||\bm{U}||_F^2+||\bm{V}||_F^2)  \nonumber\\
&\mbox{s.t. } \bm{V}_{i,j} \geq 0, i\in[n], j\in[k].
\end{align}
We can update $\bm{U}$ (while fixing $\bm{V}$) by $\bm{U}=\bm{X}\bm{V}(\alpha\bm{I}+\bm{V}^T\bm{V})^{-1}$.
%\begin{align}\label{gengusrmf}
%\bm{U}=\bm{X}\bm{V}(\alpha\bm{I}+\bm{V}^T\bm{V})^{-1}.
%\end{align}

Based on the Karush-Kuhn-Tucker (KKT) complementarity condition for the nonnegativity of $\bm{V}$, we can update $\bm{V}$ (while fixing $\bm{U}$) by
\begin{align}\label{gengvsrmf}
\bm{V}_{i,j}=\bm{V}_{i,j}\cdot\sqrt{\frac{(\bm{X}^T\bm{U})_{i,j}^{+}+[\bm{V}(\alpha\bm{I}+\bm{U}^T\bm{U})]_{i,j}^{-}}{(\bm{X}^T\bm{U})_{i,j}^{-}+[\bm{V}(\alpha\bm{I}+\bm{U}^T\bm{U})]_{i,j}^{+}}}.
\end{align}
\subsection{Extending to Incomplete Multi-view Clustering}
\label{subsection:self}
To solve the incomplete multi-view clustering problem, we need to extend Eq.~\eqref{dansrmf}.
For an incomplete multi-view dataset $\{\bm{X}^{(v)}\}_{v=1}^{m}$, we assume that different views have distinct basis matrices $\{\bm{U}^{(v)}\}_{v=1}^{m}$ and a common latent feature matrix $\bm{V}$.

First, we define an $n$-dimensional column vector $\bm{g}^{(v)}$ to indicate the incompleteness:
%. If the $i$-th instance is in the $v$-th view, $\bm{g}_{i}^{(v)}=1$; otherwise $\bm{g}_{i}^{(v)}=0$.
\begin{align}\label{g}
  \bm{g}_{i}^{(v)} \!=\! \left\{ \begin{array}{ll}
                     1, & \textrm{if the }i\textrm{-th instance is in the }v\!\textrm{-th view}; \\
                     0, & \textrm{otherwise.}
                   \end{array}
  \right.
\end{align}
To facilitate matrix operations, we extend $\bm{g}^{(v)}$ into an incomplete indicator matrix $\bm{G}^{(v)}\in\mathbb{R}^{d_v\times n}$:
% and extend Eq.~\eqref{danseminmf} to the incomplete multi-view clustering as follows: we define a diagonal matrix $\bm{G}^{(v)}\in\mathbb{R}^{n\times n}$ to indicate incompleteness of the $v$-th view:
\begin{align}\label{G}
  \bm{G}_{:,i}^{(v)} \!=\! \left\{ \begin{array}{ll}
                     1, & \textrm{if the }i\textrm{-th instance is in the }v\!\textrm{-th view}; \\
                     0, & \textrm{otherwise,}
                   \end{array}
  \right.
\end{align}
where $\bm{G}_{:,i}^{(v)}=1$ denotes that the elements in the $i$-th column of matrix $\bm{G}^{(v)}$ are all 1. Note that when the $v$-th view contains all the instances, the matrix $\bm{G}^{(v)}$ is an all-one matrix. If the $v$-th view miss some instances, the view matrix $\bm{X}^{(v)}$ will miss the corresponding columns, and the corresponding columns in $\bm{G}^{(v)}$ will become 0, i.e., $\sum_{i=1}^n\bm{G}_{t,i}^{(v)}<n$ (where $t\in [d_v]$).

%For a complete multi-view data, we define a complete indicator matrix $\{\bm{C}^{(v)}\}_{v=1}^{m}\in\mathbb{R}^{d_v\times n}$. In addition, $\bm{C}^{(v)}$ is an all-1 matrix. Note that
Second, we extend Eq.~\eqref{dansrmf} to
%We extend Eq.~\eqref{dansrmf} to the following incomplete multi-view clustering framework:
\begin{align}\label{incompseminmf}
\min\limits_{\bm{U}^{(v)},\bm{V}}\sum_v&(||\bm{G}^{(v)}\odot(\bm{X}^{(v)}-\bm{U}^{(v)}\bm{V}^T)||_F^2\nonumber\\
&+\alpha(||\bm{U}^{(v)}||_F^2+||\bm{V}||_F^2))  \nonumber\\
\mbox{s.t. }\bm{V}_{i,j} \geq& 0, i\in[n], j\in[k],
\end{align}
where $\bm{U}^{(v)} \in \mathbb{R}^{d_{v}\times c}$, $\bm{V}\in\mathbb{R}^{n\times c}$, and $\odot$ is the operation that multiplies two matrices by multiplying corresponding elements.
%Particularly, $\bm{V}$ often serves as the cluster index matrix when clustering.
But this extension is not satisfactory, we will propose a better method in the next section.

\section{Proposed ANIMC Approach}   \label{section:meth}
%To solve the auto-weighted problem of incomplete multi-views clustering, we present our Auto-weighted Incomplete Multi-view Clustering(SwIMC) in this section.
%\subsection{Auto-weighted Incomplete Multi-View Clustering}
%Given an incomplete multi-view dataset, let $X^{(1)},X^{(2)},\ldots,X^{(n)}$ be the data matrix for each view. The data matrix $X_v\in \mathbb{R}^d\times n$ for the $v$th view. The objective function using the Incomplete Multi-View Clustering via Semi-NMF method is:
%\begin{align}\label{imvcsNMF}
%\min ||(X^{(v)}-U^{(v)}V^{(v)^T})G^{(v)}||_F^2 \nonumber\\
%s.t. V^{(v)}\geqslant 0
%\end{align}
%where $||\cdot||_F$ is the Frobenius norm, $U^{(v)}\in \mathbb{R}^{d_v\times K}$ ,$V^{(v)}\in \mathbb{R}^{n\times K}$,and $K$ is the dimension of subspace.
%
%Eq.~\eqref{imvcsNMF} only considers the diversity of different views, and does not consider the consistency of between views. To solve this problem, different views are assumed to share the same feature space($i.e. V^{(1)}=V^{(2)}=\cdots=V^{(v)}=V$). To ensure the local consistency of the feature space, letting $V^{\ast}$ be the consensus latent feature matrix, we need to align $V$ and $V^{\ast}$ when constructing the objective function. Then Eq.~\eqref{imvcsNMF} is rewritten as follows:
%\begin{align}\label{rewrnmf}
%\min ||(X^{(v)}-U^{(v)}V^{(v)^T})G^{(v)}||_F^2+\mu||(V^{\ast}-V)||_F^2 \nonumber\\
%\mbox{s.t.} V \geqslant 0
%\end{align}
%where $\mu$ is the hyper-parameter.
%%ÎÒÃÇ¼ÌÐøÑØÓÃÕâ¸öËã·¨£¬½«Æä×÷ÎªÎÒÃÇµÄ»ù´¡²¿·Ö£¬±í´ïÐÎÊ½Îª£º
By showing the drawback of the direct extension to incomplete multi-view clustering, we first present the motivation of our proposed Auto-weighted Noisy and Incomplete Multi-view Clustering (ANIMC) approach. Then we model ANIMC as the joint of a soft auto-weighted strategy and a doubly soft regularized regression model.
%We first show the drawback of simple extending to incomplete multi-view clustering, which is the motivation of our.
%propose an adaptive semi-RNMF to automatically assign an optimal weight to each view, and then we design doubly regularized regressions which designs two $L_F$-Norm regularized regressions to improve the clustering performance.
\subsection{Motivation} \label{subsection:motivation}
Note that Eq.~\eqref{incompseminmf} is the root function to integrate different incomplete views. For different tasks, we may need different functions with stronger generalization ability than Eq.~\eqref{incompseminmf}.
%for stronger generalization ability than Eq.~\eqref{incompseminmf}.
A feasible idea is extending Eq.~\eqref{incompseminmf} to a series of exponential functions. Thus, we can rewrite Eq.~\eqref{incompseminmf} as
\begin{align}\label{incompseminmfduo}
\min\limits_{\bm{U}^{(v)},\bm{V}}\sum_v&(||\bm{G}^{(v)}\odot(\bm{X}^{(v)}-\bm{U}^{(v)}\bm{V}^T)||_F^r \nonumber\\
&+\alpha(||\bm{U}^{(v)}||_F^2+||\bm{V}||_F^2))  \nonumber\\
\mbox{s.t. } \bm{V}_{i,j} \geq & 0, i\in[n], j\in[k],
\end{align}
where $0<r\leq 2$. As $r$ changes, we can obtain different exponential functions. By choosing different exponential functions, Eq.~\eqref{incompseminmfduo} can integrate all views on different tasks.
%Since each view shares the same latent feature matrix, the goal of assigning each instance to the most suitable cluster in each view and ensuring clustering consistency across views is achieved.
%Note that there is no weight factor explicitly defined in Eq.~\eqref{incompseminmfduo}.

However, Eq.~\eqref{incompseminmfduo} cannot distinguish the availability of different views because Eq.~\eqref{incompseminmfduo} does not define the weight factor (i.e., importance) for each view. Intuitively, Eq.~\eqref{incompseminmfduo} only simply fills the missing instances into each view, which cannot effectively leverage the consistent information between views. If we use $\bm{E}^{(v)}=\bm{X}^{(v)}-\bm{U}^{(v)}\bm{V}^T$ to represent the noises, and $(\bm{G}^{(v)}\odot\bm{E}^{(v)})$ can represent the combination of noises and incompleteness. For a noisy and incomplete multi-view dataset, we assume that a view with more noises will have a larger $\bm{E}^{(v)}$. Eq.~\eqref{incompseminmfduo} is difficult to deal with noises effectively because these noisy views have a greater impact on the objective function.
What is worse, Eq.~\eqref{incompseminmfduo} pushes the noisy view with the lowest clustering availability hardest, which
hurts the clustering performance.
%causes the degradation of performance.
Therefore, we need a soft auto-weighted model to adaptively weight each view.
%an adaptive semi-RNMF model that can adaptively adjust the weights according to the clustering capacity of different views is desirable. Doubly regularized regression is used to bridge all the inter-view instances.
%pushes the view which has lowest clustering capacity towards zero harder than the other views, which degrades the performance drastically.
\subsection{Soft Auto-weighted Strategy Via Adaptive Semi-RNMF} \label{subsection:adaptive}
Different views have distinct availability, but the common latent feature matrix $\bm{V}$ does not directly contain the information about the availability of each view. The information can distinguish different views and will also affect clustering results. Therefore, a clever way is to obtain the optimal $\bm{V}$ and distinguish the availability of different views simultaneously through one iteration. Although Eq.~\eqref{incompseminmfduo} can integrate multiple views, it is difficult to distinguish the availability of different views after obtaining $\bm{V}$. It is because Eq.~\eqref{incompseminmfduo} does not assign weights to these views.
%In addition, Eq.~\eqref{incompseminmfduo} does not treat different instances equally because each row of $\bm{V}$ sums up to different values.

Therefore, we propose a novel model named adaptive semi-RNMF to assign a proper weight to each view and learn optimal $\bm{V}$, simultaneously. It relies on the following two intuitive assumptions for a noisy and incomplete multi-view dataset $\{\bm{X}^{(v)}\}_{v=1}^{m}$: (i) $\bm{X}^{(v)}$ is the perturbation of $\bm{U}^{(v)}\bm{V}^T$ due to noises; and (ii) incomplete views with more noises should have smaller weights.
%To treat each instance equally, we add a constraint $\bm{V}_i\bm{1}_n=1$ to ensure that each row of $\bm{V}_i$ sums up to 1.
To learn optimal latent feature matrix $\bm{V}$, we can design the Lagrangian function as follows:
\begin{align}\label{lap}
\sum_v||\bm{G}^{(v)}\odot(\bm{X}^{(v)}-\bm{U}^{(v)}\bm{V}^T)||_F^r+\alpha||\bm{V}||_F^2+\zeta(\Lambda,\bm{V}),
\end{align}
where $\Lambda$ is the Lagrange multiplier and $\zeta(\Lambda,\bm{V})$ is a proxy for the constraints. Setting the derivative of Eq.~\eqref{lap} w.r.t. $\bm{V}$ to zero, we can obtain
\begin{align}\label{jiaw}
\sum_v&w_{v} \frac{\partial ||\bm{G}^{(v)}\odot(\bm{X}^{(v)}-\bm{U}^{(v)}\bm{V}^T)||_F^2}{\partial\bm{V}}+\frac{\partial\alpha||\bm{V}||_F^2}{\partial\bm{V}}\nonumber\\
&+\frac{\partial\zeta(\Lambda,\bm{V})}{\partial \bm{V}}=\bm{0},
\end{align}
where
\begin{align}\label{wgengbian}
w_{v}=0.5r||\bm{G}^{(v)}\odot(\bm{X}^{(v)}-\bm{U}^{(v)}\bm{V}^T)||_F^{0.5r-1}.
\end{align}
In Eq.~\eqref{wgengbian}, $w_{v}$ depends on the target variable $\bm{V}$, so we cannot directly obtain $w_{v}$. To solve Eq.~\eqref{wgengbian}, we first set $w_{v}$ fixed and update $w_{v}$ after obtaining $\bm{U}^{(v)}$ and $\bm{V}$.
If we fix $w_{v}$, Eq.~\eqref{jiaw} can be viewed as the solution to the following problem:
\begin{align}\label{wgengxin}
\min\limits_{w_{v},\bm{U}^{(v)},\bm{V}}\sum_v&(w_{v}||\bm{G}^{(v)}\odot(\bm{X}^{(v)}-\bm{U}^{(v)}\bm{V}^T)||_F^2\nonumber\\
&+\alpha(||\bm{U}^{(v)}||_F^2+||\bm{V}||_F^2))\\
w_{v}=&0.5r||\bm{G}^{(v)}\odot(\bm{X}^{(v)}-\bm{U}^{(v)}\bm{V}^T)||_F^{0.5r-1}.\nonumber
\end{align}
But there is a case that fails Eq.~\eqref{wgengxin}. As the missing rate grows, the zero elements of $\bm{G}^{(v)}$ in Eq.~\eqref{wgengxin} will also increase, and the view weight will enlarge. However, too-large weights will focus too much on the incompleteness and ignore the noises, which is difficult to reflect the impact of noises on clustering.
%If the $v$-th view has a large disagreement on the feature dimension of $\bm{U}^{(v)}$ with most of the other views, it often has a large weight $w_v$. In fact, it does not desire a large weight $w_v$. Moreover, a too large weight will cause us to focus too much on this view and ignore other views.
Therefore, we need to add a soft boundary to $w_v$:
\begin{align}\label{wzuixin}
w_{v}=\min(&0.5r||\bm{G}^{(v)}\odot(\bm{X}^{(v)}-\bm{U}^{(v)}\bm{V}^T)||_F^{0.5r-1},\nonumber\\
&||\bm{X}^{(v)}-\bm{U}^{(v)}\bm{V}^T||_F^{-0.5}),
\end{align}
where $0.5||\bm{X}^{(v)}-\bm{U}^{(v)}\bm{V}^T||_F^{-0.5}$ is the boundary value of Eq.~\eqref{wgengbian}, which is a special case that the dataset is a complete multi-view dataset with $r=1$. Based on function $\min(\cdot)$, $||\bm{X}^{(v)}-\bm{U}^{(v)}\bm{V}^T||_F^{-0.5}$ can reflect the effect of noises on clustering when the missing rate is relatively large.
%is the bound value of Eq.~\eqref{wgengxin}. The bound value is a special case that the dataset is complete multi-view dataset with $r=1$. The bound value can ensure that we
%$\frac{1}{2 d_v||\bm{X}^{(v)}-\bm{U}^{(v)}\bm{V}^T||_F}$ is used to learn the availability of each view, rather than just use feature dimensions to weight the view.
Therefore, the final model of adaptive semi-RNMF is
\begin{align}\label{w}
\min\limits_{w_{v},\bm{U}^{(v)},\bm{V}}\sum_v&(w_{v}||\bm{G}^{(v)}\odot(\bm{X}^{(v)}-\bm{U}^{(v)}\bm{V}^T)||_F^2\nonumber\\
&+\alpha(||\bm{U}^{(v)}||_F^2+||\bm{V}||_F^2)) \nonumber\\
w_{v}=\min(&0.5r||\bm{G}^{(v)}\odot(\bm{X}^{(v)}-\bm{U}^{(v)}\bm{V}^T)||_F^{0.5r-1},\nonumber\\
&||\bm{X}^{(v)}-\bm{U}^{(v)}\bm{V}^T||_F^{-0.5})\\
\mbox{s.t. } \bm{V}_{ij}& \geq 0,i\in[n], j\in[k]. \nonumber
\end{align}
Note that Eq.~\eqref{w} is a soft model and we can learn different weight functions $\{w_v\}_{v=1}^m$ by changing $r$, which enhances its generalization ability.
After fixing $w_{v}$, the Lagrangian function Eq.~\eqref{lap} also applies to Eq.~\eqref{w}. After we obtain $\bm{V}$ from Eq.~\eqref{w}, we can update $w_{v}$, which inspires us to optimize Eq.~\eqref{incompseminmfduo} by an alternative method. After optimization, the updated $\bm{V}$ is at least locally optimal.
%Similarly, we can obtain the updated $w_{v}$, which is exactly the learned weight of the corresponding view. In addition,
For the $v$-th view, $\bm{G}^{(v)}$ indicates the missing rate and perturbation $(\bm{X}^{(v)}-\bm{U}^{(v)}\bm{V}^T)$ represents noises. Through the above two aspects, $w_{v}$ can effectively distinguish different views, which meets our intuitive assumptions about weight. Also, Eq.~\eqref{w} updates both the optimal $\bm{V}$ and $w_{v}$ simultaneously. Moreover, Eq.~\eqref{w} not only learns the optimal $\bm{V}$, but also distinguishes different views automatically. Besides, $w_v$ considers the combined effects of noises and incompleteness. Therefore, we regard $w_{v}$ as the weight of the $v$-th view.
\subsection{Doubly Soft Regularized Regression Model via $\theta$-norm}\label{subsection:doubly}
Note that $\alpha(||\bm{U}^{(v)}||_F^2+||\bm{V}||_F^2)$ (in Eq.~\eqref{w}) serves as the regular term of $\bm{U}^{(v)}$ and $\bm{V}$. As for $w_{v}||\bm{G}^{(v)}\odot(\bm{X}^{(v)}-\bm{U}^{(v)}\bm{V}^T)||_F^2$ (in Eq.~\eqref{w}), it is used to fill missing instances into the view matrix. In most cases, the filling strategy is difficult to achieve good results because it cannot effectively use the information of the presented instances to complete the view matrix.
%cannot decrease the impact of missing instances. As for $w_{v}||\bm{G}^{(v)}\odot(\bm{X}^{(v)}-\bm{U}^{(v)}\bm{V}^T)||_F^2$ (in Eq.~\eqref{w}), it is used to fill missing instances into the view matrix.
In the field of statistics, regularized regression is an efficient tool for matrix completion on real-world incomplete data. To decrease the impact of missing instances, we try to design a doubly regularized regression model to cluster noisy and incomplete multi-view data. Based on the model, we attempt to push the latent feature matrix $\bm{V}$ towards the consensus feature matrix (denoted by $\bm{V}^{\ast}$) and align basis matrices $\{\bm{U}^{(v)}\}_{v=1}^{m}$ to the consensus basis matrix (denoted by $\bm{U}^{\ast}$).

First, we design the F-norm regularized regression to push $V$ closer to $V^{\ast}$. To reduce the disagreement between $V$ and $V^{\ast}$, we leverage the F-norm, which corresponds to Euclidean distance. Thus, we have
\begin{align}\label{regression1}
\min\limits_{\bm{V}}(||\bm{V}-\bm{V}^{\ast}||_F^2+\eta||\bm{V}||_F^2),
\end{align}
where $\eta$ is a nonnegative parameter.
% between sparsity and accuracy of regression for all the views.
%For all the views, they share the same $\bm{V}$ (i.e., $\bm{V}^{(1)}=\bm{V}^{(2)}=\ldots=\bm{V}^{(n_v)}$). Therefore, the index $v$ of the matrix $\bm{V}$  will not present in Eq.~\eqref{regression1}.

Second, to align the basis matrices $\bm{U}^{(v)}$ of the same instance in different views, we adopt the following regularized regression function based on the F-norm:
\begin{align}\label{align}
\min\limits_{\bm{U}^{(v)},\bm{A}^{(v)}}\sum_v(||\bm{A}^{(v)^T}\bm{U}^{(v)}-\bm{U}^{\ast}||_F^2+\beta ||\bm{A}^{(v)}||_F^2),
\end{align}
where $\bm{A}^{(v)}\in \mathbb{R}^{d_v\times c}$ is the regression coefficient matrix of the $v$-th view and $\beta$ is a nonnegative parameter.

For a particular incomplete dataset, its $\bm{V}^{\ast}\in \mathbb{R}^{n\times c}$ and $\bm{U}^{\ast} \in \mathbb{R}^{c\times c}$ are constant. In general, they are low-rank for all the views, and the cluster number $c$ plays an important role in clustering. For most tasks, we need not obtain specific $\bm{V}^{\ast}$ and $\bm{U}^{\ast}$, and we will get satisfactory clustering results by learning correct $c$.
Therefore, we can find simple matrices as substitutes for $\bm{V}^{\ast}$ and $\bm{U}^{\ast}$ to simplify Eq.~\eqref{regression1} and Eq.~\eqref{align}:

\noindent(i) Substitute for $\bm{V}^{\ast}$: since all the views share the same $\bm{V}$, $\bm{V}$ contains the consistent information of all views. The available consistent information is the guarantee of satisfactory clustering results. For different clustering tasks (e.g., different datasets, different missing rates, or different noise rates), we often have different $\bm{V}^{\ast}$. For matrix $\bm{V}$, we assume that due to the influence of noises we have $0<\bm{V}_{i,j}^{\ast}<\bm{V}_{i,j}$.

For example, for a noisy and incomplete multi-view dataset, under different cases (different noises or different incompleteness), we will get different $\bm{V}^{\ast}$. To ensure that we can obtain a suitable regularized regression model in different situations, we need to ensure that the regularized regression model is suitable for different cases.
To improve the generalization ability of Eq.~\eqref{regression1} (i.e., for different $\bm{V}^{\ast}$, we can obtain satisfactory clustering results), we should simplify Eq.~\eqref{regression1} reasonably.
%for real-world multi-view data, the matrix $\bm{V}^{\ast}$ contains a large number of zero elements and we can obtain $\bm{V}-\bm{V}^{\ast}\approx \bm{V}$. In addition, $\bm{V}$ and $\bm{V}^{\ast}$ have the same matrix dimension. Therefore, $\bm{V}$ can serve as the approximate substitutes for $\bm{V}-\bm{V}^{\ast}$, and we can set $||\bm{V}-\bm{V}^{\ast}||_F=||\bm{V}||_F$ to simplify Eq.~\eqref{regression1}.
%\begin{theorem}\label{vjianhuadingli}
For any matrix $\bm{V}^{\ast}$ ($0<\bm{V}^{\ast}_{i,j}<\bm{V}_{i,j}$), we have $\lim\limits_{\bm{V}^{\ast}_{i,j}\rightarrow 0^{+}}||\bm{V}-\bm{V}^{\ast}||_F=||\bm{V}||_F$. When $\bm{V}^{\ast}\rightarrow \bm{V}$, we have $(||\bm{V}-\bm{V}^{\ast}||_F^2+\eta||\bm{V}||_F^2)\rightarrow\eta||\bm{V}||_F^2$; when $\bm{V}^{\ast}\rightarrow \bm{0}$, we have $(||\bm{V}-\bm{V}^{\ast}||_F^2+\eta||\bm{V}||_F^2)\rightarrow(1+\eta)||\bm{V}||_F^2$.
Note that $\eta||\bm{V}||_F^2<(||\bm{V}-\bm{V}^{\ast}||_F^2+\eta||\bm{V}||_F^2)<(1+\eta)||\bm{V}||_F^2$. Thus, $(||\bm{V}-\bm{V}^{\ast}||_F^2+\eta||\bm{V}||_F^2)<(1+\eta)||\bm{V}||_F^2$ in all the cases. Obviously, $(1+\eta)||\bm{V}||_F^2$ is the upper limit of $(||\bm{V}-\bm{V}^{\ast}||_F^2+\eta||\bm{V}||_F^2)$.  When we minimize $(||\bm{V}-\bm{V}^{\ast}||_F^2+\eta||\bm{V}||_F^2)$, we only need to minimize its upper limit to ensure the availability of our method in different datasets with different $\bm{V}^{\ast}$, which illustrates the generalization ability of our algorithm.  Therefore, we can shrink Eq.~\eqref{regression1} to
%can be transformed into
\begin{align}\label{regression1jianhua}
\min\limits_{\bm{V}}(1+\eta)||\bm{V}||_F^2.
\end{align}
Since $\eta$ is nonnegative, we can transform Eq.~\eqref{regression1jianhua} into
\begin{align}\label{vjian}
\min\limits_{\bm{V}}||\bm{V}||_F^2.
\end{align}
%Eq.~\eqref{vjian}, which proves Theorem~\ref{vduiqidingli}.
%For ease of description, we set $\tau=||\bm{V}||_F$, where $\tau\geq 0$. Thus, we can transform solving Eq.~\eqref{vzhengming} into minimizing:
%\begin{align}\label{vzhengminggai}
%f(\tau)=\tau^2+\eta\tau.
%\end{align}
%Deriving $f(\tau)$ w.r.t. $\tau$, we can have
%\begin{align}\label{vzhengmingpian}
%\frac{\partial f(\tau)}{\partial\tau}=2\tau+\eta.
%\end{align}
%Since $\tau\geq 0$ and $\eta\geq 0$, $\frac{\partial f(\tau)}{\partial\tau}\geq 0$ in Eq.~\eqref{vzhengmingpian} and $f(\tau)$ is monotonically increasing. In addition, as $\tau$ increases, $2\tau$ and $f(\tau)$ will approach each other because $\eta$ is constant. Therefore, $\min||\bm{V}||_F^2+\eta||\bm{V}||_F$ and $\min||\bm{V}||_F^2$ are mathematically equivalent.
%\end{proof}
%Therefore, we rewrite Eq.~\eqref{regression1} based on Theorem~\ref{vduiqidingli}:
%\begin{align}\label{vjian}
%\min\limits_{\bm{V}}||\bm{V}||_F^2.
%\end{align}
\noindent(ii) Substitute for $\bm{U}^{\ast}$: for most real-world applications, we need to integrate the consistent information of different views for clustering. Fortunately, we can maximize the consistent information by pushing the latent feature matrix towards the consensus feature matrix. When performing the regularized regression on $\bm{U}^{\ast}$, we only need to cluster $n$ instances into $c$ clusters. Thus, we prefer to obtain an effective $c$ rather than a specific $\bm{U}^{\ast}$. Since $\bm{U}^{\ast}$ is a $c$ dimensional square matrix, we can leverage a $c$-dimensional identity matrix $\bm{I}$ as an alternative to $\bm{U}^{\ast}$\footnote{The forms of different alternatives have little effect on the clustering performance and we choose $\bm{I}$ as the alternative for simplicity.}.
Therefore, we can rewrite Eq.~\eqref{align} as
\begin{align}\label{align1}
\min\limits_{\bm{U}^{(v)},\bm{A}^{(v)}}\sum_v(||\bm{A}^{(v)^T}\bm{U}^{(v)}-\bm{I}||_F^2+\beta||\bm{A}^{(v)}||_F^2).
\end{align}
However, the regularized regression based on F-norm is difficult to learn the sparse $\bm{U}^{(v)}$ and $\bm{V}$ for clustering. Because F-norm cannot select features across all data points with joint sparsity. To obtain a more robust regularized regression model, we design a doubly soft regularized regression model by designing a novel $\theta$-norm.
%extend the F-norm regularized regression to $\theta$-norm regularized regression.

For any matrix $\bm{B}$, its F-norm is defined as
\begin{align}\label{ffanshudingyi}
||\bm{B}||_F=\sqrt{\sum_i\sum_j\bm{B}_{i,j}^2}.
\end{align}
%Inspired by the definition of the F-norm,
Inspired by Eq.~\eqref{ffanshudingyi}, we define the $\theta$-norm of matrix $\bm{B}$ as
\begin{align}\label{sigema}
||\bm{B}||_{\theta}=\sum_i\frac{(1+\theta)(\sum_j\bm{B}_{i,j}^2)^2}{1+\theta\sum_j\bm{B}_{i,j}^2},
\end{align}
where we can choose the proper $\theta$ according to our needs. Based on matrix $\bm{B}$, we design a diagonal matrix $\bm{D}_B$ defined as
%$\bm{D}_{i,i}$ is
\begin{align}\label{ddingyi}
\bm{D}_{B_{i,i}}=\frac{2+\theta(\sum_j\bm{B}_{i,j}^2+2)+\theta^2\sum_j\bm{B}_{i,j}^2}{\theta^2(\sum_j\bm{B}_{i,j}^2)^2+2\theta\sum_j\bm{B}_{i,j}^2+1}.
\end{align}
\begin{theorem}\label{sigemadingli}
For any $\theta>0$, we have
\begin{align}\label{piandingli}
\frac{\partial ||\bm{B}||_{\theta}}{\partial \bm{B}}=\bm{D}_B\bm{B}.
\end{align}
\end{theorem}
\begin{proof}\label{zhengpiandingli}
\begin{align}\label{piandinglizheng}
&\frac{\partial ||\bm{B}||_{\theta}}{\partial \bm{B}_{i,:}}=\frac{\partial((1+\theta)||\bm{B}_{i,:}||_2^2/(1+\theta||\bm{B}_{i,:}||_2))}{\partial \bm{B}_{i,:}}\nonumber\\
&=\frac{2+\theta(||\bm{B}_{i,:}||_2+2)+\theta^2||\bm{B}_{i,:}||_2}{\theta^2||\bm{B}_{i,:}||_2^2+2\theta||\bm{B}_{i,:}||_2+1}\bm{B}_i=\bm{D}_{B_{i,i}}\bm{B}_{i,:}.
\end{align}
\end{proof}
Our proposed $\theta$-norm has the following characteristics:
%There are some favorable natures on $\theta$-norm:

1) $\theta$-norm is nonnegative and global differentiable;

%2) if $\forall i$ $||\bm{B}_i||_2\ll \frac{1}{\theta}$,  it holds that $||\bm{B}||_{\theta}\rightarrow(\theta+1)||\bm{B}||_F^2$;
%
%3) if $\forall i$ $||\bm{B}_i||_2\gg \frac{1}{\theta}$,  it holds that $||\bm{B}||_{\theta}\rightarrow\frac{(\theta+1)}{\theta}||\bm{B}||_F^2$;

2) when $\theta\rightarrow\infty$, we have $||\bm{B}||_{\theta}\rightarrow||\bm{B}||_{2,1}$ and $\bm{D}_{B_{i,i}}\rightarrow 1/||\bm{B}_{i,:}||_2$;

3) when $\theta\rightarrow 0$, we have $||\bm{B}||_{\theta}\rightarrow||\bm{B}||_F^2$ and $\bm{D}_B\rightarrow\bm{I}$.

As $\theta$ increases, $||\bm{B}||_{\theta}$ is closer to $||\bm{B}||_{2,1}$ and $\bm{B}$ becomes more sparse. Since we can choose different $\theta$ to adjust the sparsity of the matrix, $\theta$-norm can be viewed as a soft norm with a strong generalization ability.
%obtain a more robust regularized regression with strong generalization ability.
Besides, the global differentiability of $\theta$-norm can ensure that we can learn the correct derivative, which is significant for the following optimization (see Section~\ref{subsection:optimization}).
%The sparsity of $\bm{B}$ depends on $\theta$, $\bm{B}$ becomes more sparse
%With this relaxed norm, the sparsity of $\bm{B}$ is adjustable without depending on the regularization parameter. That is to say, when less features need to be selected, the more row-sparse $\bm{B}$ could be achieved with a smaller $\theta$ correspondingly. In addition, the global differentiability of $\theta$-norm also avoids the abnormal situations involving derivative. Since the derivative of regularization term is inevitably required in the optimization of problem, $\theta$-norm is more suitable to be adopted as regularization norm.
Therefore, we transform Eq.~\eqref{vjian} into
\begin{align}\label{vjiangai}
\min\limits_{\bm{V}}||\bm{V}||_{\theta}.
\end{align}
Similarly, we transform Eq.~\eqref{align1} into
\begin{align}\label{align1gai}
\min\limits_{\bm{U}^{(v)},\bm{A}^{(v)}}\sum_v(||\bm{A}^{(v)^T}\bm{U}^{(v)}-\bm{I}||_F^2+\beta||\bm{A}^{(v)}||_{\theta}).
\end{align}
Combining Eq.~\eqref{vjiangai} and Eq.~\eqref{align1gai}, we can obtain the doubly soft regularized regression model as follows
\begin{align}\label{dou}
\min\limits_{\bm{U}^{(v)},\bm{V},\bm{A}^{(v)}}&\sum_{v}(||\bm{A}^{(v)^T}\bm{U}^{(v)}-\bm{I}||_F^2 +\beta||\bm{A}^{(v)}||_{\theta})+||\bm{V}||_{\theta}.
\end{align}
For the soft model Eq.~\eqref{dou}, we can change its sparsity by adjusting $\theta$, which improves its generalization ability.
%By adjusting $\theta$,
% $||B||_{2,1}$ is $L_{2,1}$ norm of $B^{(v)}$, which ensures that $B^{(v)}$ is robust to noises and outliers. Similar to $V^{\ast}$, $U^{\ast}$ is an consensus basis matrix $U^{\ast}\in \mathbb{R}^{p \times K}$. And $\nu$ is a hyper-parameter. The $L_{2,1}$ norm of $B^{(v)}$ is defined as:
%\begin{align}\label{L2,1}
%||B^{(v)}||_{2,1}=\sum_{i=1}^d \sqrt{\sum_{j=1}^K B_{ij}^{(v)^2}}
%\end{align}
\subsection{Objective Function}\label{subsection:obj}
Considering the objective for soft auto-weighted strategy (Eq.~(\ref{w})) and doubly soft regularized regression model (Eq.~(\ref{dou})), we minimize the following problem:
%\begin{align}\label{quan}
%\min &||(X^{(v)}-U^{(v)} V^T)G^{(v)}||_F^2+\alpha(||B^{(v)^T} U^{(v)}-I||_F^2  \nonumber\\
%&+\beta(||V||_F^2+\gamma ||B^{(v)}||_{2,1}))  \nonumber\\
%\mbox{s.t. } &V \geqslant 0
%\end{align}
%In Eq.~\eqref{quan}, the weight of the view which belongs to the global information is not directly defined.
%
%It can be seen that $w_{v}$ is dependent on the target variable $V$. Therefore, Eq.~\eqref{jiaw} cannot be directly solved. But if $w_{v}$ is set stationary, the AwIMC can be formulated as:
\begin{align}\label{wansimc}
L=&\sum\limits_{v}(w_{v}||\bm{G}^{(v)}\odot(\bm{X}^{(v)}-\bm{U}^{(v)}\bm{V}^T)||_F^2  \nonumber\\
&+\alpha||\bm{A}^{(v)^T}\bm{U}^{(v)}-\bm{I}||_F^2+\beta||\bm{A}^{(v)}||_{\theta})+\alpha||\bm{V}||_{\theta}   \nonumber\\
w_{v}=&\min(0.5r||\bm{G}^{(v)}\odot(\bm{X}^{(v)}-\bm{U}^{(v)}\bm{V}^T)||_F^{0.5r-1},\nonumber\\
&||\bm{X}^{(v)}-\bm{U}^{(v)}\bm{V}^T||_F^{-0.5})\nonumber\\
\mbox{s.t. }\bm{V}_{i,j} &\geq 0,i\in[n], j\in[k],
\end{align}
%\begin{align}\label{wansimc}
%L=&\min\limits_{V,w_{v}}\sum\limits_{v}(w_{v}||(\bm{X}^{(v)}-\bm{U}^{(v)}\bm{V}^T)\bm{G}^{(v)}||_F^2+\alpha\mbox{Tr}(\bm{V}^T\bm{L}_S^{(v)}\bm{V}))   \nonumber\\
%&\beta||\bm{S}||_F^2  \nonumber\\
%\mbox{s.t. } &\bm{V}_{i,j} \geq 0,i\in[n], j\in[k],\bm{V}_i\bm{1}_n=1,
%\end{align}
where $\alpha$ and $\beta$ are nonnegative parameters. Next, our experiments will show its effectiveness and efficiency of Eq.~\eqref{wansimc}.
% to control the trade-off between two regularized regressions.
%Although the Eq.~\eqref{wansimc} assumes that $w_{v}$ is stationary, it can still be updated by Eq.~\eqref{wgengxin}, and the results will be more accurate. The updated solution result is at least a local optimal solution, and the weight learned by $w_{v}$ is also a linear superposition of different graphs.
\subsection{Optimization}\label{subsection:optimization}
Since Eq.~\eqref{wansimc} is not convex for all the variables simultaneously, inspired by Augmented Lagrange Multiplier (ALM) method~\cite{hestenes1969multiplier}, we design a four-step alternating iteration procedure to optimize it:

%(1) updating $U^{(v)}$ after fixing the other variables, (2) updating $B^{(v)}$ after fixing the other variables, (3) updating $V$ after fixing the other variables, (4) updating $w_v$ by Eq. \eqref{def1} and $U^{(v)}, B^{(v)}, V$ updated in previous steps.
\begin{algorithm}[t]
    \caption{ANIMC}
    \label{simc_alg}
	\begin{algorithmic}
	 \REQUIRE{Data matrices for noisy and incomplete views $\{\bm{X}^{(v)}\}_{v=1}^m\in \mathbb{R}^{d_v\times n}$, indicator matrix $\{\bm{G}^{(v)}\}_{v=1}^m\in\mathbb{R}^{d_v\times n}$, parameters \{$\alpha,\beta$\}, cluster number $c$.}
\STATE{Choose parameters \{$\theta,r$\} according to needs.}
    \STATE{Initialize regression coefficient matrix $\bm{A}^{(v)}$, basis matrix $\bm{U}^{(v)}$, view weight $w_v = 1/m$ for each view.}
    \STATE{Initialize common latent feature matrix $\bm{V}$.}
    \STATE{Initialize iteration time $i=0$, maximum iteration time $i_{\max}$.}
%	\REPEAT
\WHILE{Eq.~\eqref{wansimc} does not converge $\&\&$ $i\leq i_{\max}$}
%	\For{}
	\STATE{Update $\bm{U}^{(v)}$ by Eq.~\eqref{gengu};}
	\STATE{Update $\bm{A}^{(v)}$ by Eq.~\eqref{byou};}
	\STATE{Update $\bm{V}$ by Eq.~\eqref{gengv};}
	\STATE{Update $w_v$ by Eq.~\eqref{wzuixin};}
    \STATE{$i=i+1$;}
\ENDWHILE
%	\UNTIL{converges}
    \ENSURE{$\{\bm{A}^{(v)}\}_{v=1}^m$, $\{\bm{U}^{(v)}\}_{v=1}^m$ ,$\{w_v\}_{v=1}^m$, $\bm{V}$ and clustering results.}
	\end{algorithmic}
\end{algorithm}

\textbf{Step 1.} Updating $\bm{U}^{(v)}$. Fixing the other variables, we need to minimize the following problem:
\begin{align}\label{ju}
L(\bm{U}^{(v)})=&w_{v}||\bm{G}^{(v)}\odot(\bm{X}^{(v)}-\bm{U}^{(v)}\bm{V}^T)||_F^2\nonumber\\
&+\alpha||\bm{A}^{(v)^T}\bm{U}^{(v)}-\bm{I}||_F^2.
\end{align}
%where $\bm{P_0}^{(v)}=\bm{X}^{(v)}-\bm{U}^{(v)}\bm{V}^T$ and $\bm{P_1}^{(v)}=\bm{A}^{(v)^T}\bm{U}^{(v)}-\bm{I}$.

We can obtain $\bm{U}^{(v)}$ by setting the derivative of $L(\bm{U}^{(v)})$ w.r.t. $\bm{U}^{(v)}$ to zero as follows:
\begin{align}\label{gengu}
\frac{\partial L(\bm{U}^{(v)})}{\partial \bm{U}^{(v)}}=&2w_v\bm{G}^{(v)}\odot(\bm{X}^{(v)}-\bm{U}^{(v)}\bm{V}^T)\bm{V}\nonumber\\
&+2\alpha \bm{A}^{(v)}(\bm{A}^{(v)^T}\bm{U}^{(v)}-\bm{I})=\bm{0}.
\end{align}
%Based on the Karush-Kuhn-Tucker (KKT) complementary condition for the nonnegativity of $\bm{U}^{(v)}$, we obtain
%\begin{align}\label{genguxingshi}
%(-2w_v(\bm{G}^{(v)}\odot\bm{U}^{(v)}\bm{V}^T)\bm{V}+2\alpha\bm{A}^{(v)}(\bm{A}^{(v)^T}\bm{U}^{(v)}-\bm{I}))_{i,j}\bm{U}_{i,j}^{(v)}=\bm{0}.
%\end{align}
%Therefore, we can derive the following updating rule
%\begin{align}\label{gengu}
%\bm{U}_{i,j}^{(v)}=\bm{U}_{i,j}^{(v)}\sqrt{\frac{(w_v\bm{G}^{(v)}\odot\bm{U}^{(v)}(\bm{V}^T\bm{V})+\alpha\bm{A}^{(v)})_{i,j}}{(\alpha(\bm{A}^{(v)}\bm{A}^{(v)^T})\bm{U}^{(v)})_{i,j}}}.
%\end{align}

% $\bm{U}^{(v)}$ by

%where $\bm{P_2}^{(v)}=\bm{G}^{(v)}\bm{G}^{(v)^T}\bm{V}=\bm{G}^{(v)}\bm{V}$.
%Setting Eq.~\eqref{gengu} to $0$, we can get
%\begin{align}\label{pianu0}
%\alpha B^{(v)}(B^{(v)^T}U^{(v)}-I)-w_{v}(X^{(v)}-U^{(v)}V^T)G^{(v)}V=0
%\end{align}
The Eq.~\eqref{gengu} is the Sylvester equation w.r.t. $\bm{U}^{(v)}$ \cite{bartels1972solution}. Based on $d_v$ and $c$,
we can divide Eq.~\eqref{gengu} into two cases to facilitate the solution.
\begin{theorem}\label{tgduiying}
For the $v$-th view, denote $\bm{E}^{(v)}=\bm{X}^{(v)}-\bm{U}^{(v)}\bm{V}^T$, and we have $\bm{G}^{(v)}\odot\bm{E}^{(v)}=\bm{E}^{(v)}\bm{T}^{(v)}$.
\end{theorem}
%For Theorem~\ref{tgduiying}, we can prove it as follows:
\begin{proof}\label{zhengmingduiying}
For $\bm{G}^{(v)}\odot\bm{E}^{(v)}$, we have
\begin{align}\label{GE}
  (\bm{G}^{(v)}\odot\bm{E}^{(v)})_{i,j} \!=\! \left\{ \begin{array}{ll}
                     \bm{E}_{i,j}^{(v)}, & \textrm{if instance }j\textrm{ is in view }v; \\
                     0, & \textrm{otherwise.}
                   \end{array}
  \right.
\end{align}
For $\bm{E}^{(v)}\bm{T}^{(v)}$, we have
\begin{align}\label{ET}
  (\bm{E}^{(v)}\bm{T}^{(v)})_{i,j} \!=\! \left\{ \begin{array}{ll}
                     \bm{E}_{i,j}^{(v)}, & \textrm{if instance }j\textrm{ is not in view }v;\\
                     0, & \textrm{otherwise.}
                   \end{array}
  \right.
\end{align}
Note that $\bm{G}^{(v)}\odot\bm{E}^{(v)}$ and $\bm{E}^{(v)}\bm{T}^{(v)}$ have the same size. Besides, $(\bm{G}^{(v)}\odot\bm{E}^{(v)})_{i,j}=(\bm{E}^{(v)}\bm{T}^{(v)})_{i,j}$ in all cases, which proves Theorem~\ref{tgduiying}.
\end{proof}
\textbf{(i)} When both $d_v$ and $c$ are small, based on Theorem~\ref{tgduiying}, we can update $\bm{U}^{(v)}$ by
\begin{align}\label{vecu}
vec(\bm{U}^{(v)})=&[\bm{I}\otimes((w_{v}\bm{V}^T\bm{T}^{(v)}\bm{V})\otimes \bm{I}+\alpha\bm{A}^{(v)}\bm{A}^{(v)^T})]^{-1}  \nonumber\\
&vec(\alpha\bm{A}^{(v)}+w_{v}\bm{X}^{(v)}\bm{T}^{(v)}\bm{V}),
\end{align}
where $vec(\cdot)$ is the vectorization operation; $\otimes$ is the Kronecker product;  $\bm{T}^{(v)}\in\mathbb{R}^{n\times n}$ is the corresponding diagonal matrix of vector $\bm{g}^{(v)}$.

\textbf{(ii)} When both $d_v$ and $c$ are large, the solution is solved by a conjugate gradient \cite{hestenes1952methods}. For ease of calculation, Eq.~\eqref{gengu} can be approximated as the Lyapunov equation \cite{sorensen2002sylvester}.
%, which can be solved by the $lyap$ function of MATLAB.

\textbf{Step 2.}
Updating $\bm{A}^{(v)}$. Fixing the other variables, we need to minimize the following problem:
\begin{align}\label{jb}
L(\bm{A}^{(v)})=\alpha||\bm{A}^{(v)^{T}}\bm{U}^{(v)}-\bm{I}||_F^2+\beta||\bm{A}^{(v)}||_{\theta}.
\end{align}

We can obtain $\bm{A}^{(v)}$ by setting the derivative of $L(\bm{A}^{(v)})$ w.r.t. $\bm{A}^{(v)}$ to zero as follows:
\begin{align}\label{gengb0}
\alpha\bm{U}^{(v)}\bm{U}^{(v)^T}\bm{A}^{(v)}-\bm{U}^{(v)}+\beta\bm{D}_A^{(v)}\bm{A}^{(v)}=\bm{0}.
\end{align}
%where $\bm{H}^{(v)}$ is a diagonal matrix defined as:
%\begin{align}\label{dingd}
%\bm{H}_{ii}^{(v)} = \frac{1}{||\bm{A}_{i,:}^{(v)}||_2},
%\end{align}
%and $\bm{A}_{i,:}^{(v)}$ is the $i$-th row of matrix $\bm{A}^{(v)}$.

By solving Eq.~\eqref{gengb0}, $\bm{A}^{(v)}$ can be updated by
\begin{align}\label{gengb0jian}
\bm{A}^{(v)}=[\alpha\bm{U}^{(v)}\bm{U}^{(v)^T}+\beta\bm{D}_A^{(v)}]^{-1}\bm{U}^{(v)}.
\end{align}

In most image processing applications, $\bm{U}^{(v)}$ has a large number of rows and a small number of columns and $(\bm{U}^{(v)}\bm{U}^{(v)^T}+\beta\bm{D}_A^{(v)})$ is close to singular, so we can use the Woodbury matrix identity \cite{higham2002accuracy} to simplify the calculation. Thus, we have
\begin{align}\label{byou}
\bm{A}^{(v)}=&\frac{\alpha}{\beta}[\bm{D}_A^{(v)^{-1}}-\bm{D}_A^{(v)^{-1}}\bm{U}^{(v)}(\bm{U}^{(v)^T}\bm{D}_A^{(v)^{-1}}\bm{U}^{(v)}\nonumber\\
&+\beta\bm{I})^{-1}\bm{U}^{(v)^T}\bm{D}_A^{(v)^{-1}}]\bm{U}^{(v)}.
\end{align}
%where $\bm{J_1}^{(v)}=\bm{H}^{(v)^{-1}}\bm{U}^{(v)}$, $\bm{J_2}^{(v)}=\bm{U}^{(v)^T}\bm{H}^{(v)^{-1}}$, $\bm{J_3}^{(v)}=(\beta\gamma \bm{I}+2\bm{J_2}^{(v)})^{-1}$ and $\bm{I}$ is the identity matrix.

\textbf{Step 3.}
Updating $\bm{V}$. Fixing the other variables, we need to minimize the following problem:
\begin{align}\label{jv}
L(\bm{V})=\sum\limits_vw_{v}||\bm{G}^{(v)}\odot(\bm{X}^{(v)}-\bm{U}^{(v)}\bm{V}^T)||_F^2+\alpha ||\bm{V}||_{\theta}.
\end{align}
%where $\bm{A_1}^{(v)} = ||(\bm{X}^{(v)}-\bm{U}^{(v)}\bm{V}^T)\bm{G}^{(v)}||_F$.

We can obtain $\bm{V}$ by setting the derivative of $L(\bm{V})$ w.r.t. $\bm{V}$ to zero as follows:
\begin{align}\label{pianv}
\sum\limits_v&w_{v}((\bm{G}^{(v)}\odot\bm{U}^{(v)}\bm{V}^T)^T-(\bm{G}^{(v)}\odot\bm{X}^{(v)})^T)\bm{U}^{(v)} \nonumber\\
&+\beta\bm{D}_V\bm{V}=\bm{0}.
\end{align}
Similar to \cite{ding2010convex}, based on the KKT  complementarity condition for the nonnegativity of $\bm{V}$, we have
\begin{align}\label{pianv}
(\sum\limits_v&w_{v}((\bm{G}^{(v)}\odot\bm{U}^{(v)}\bm{V}^T)^T-(\bm{G}^{(v)}\odot\bm{X}^{(v)})^T)\bm{U}^{(v)} \nonumber\\
&+\beta\bm{D}_V\bm{V})_{i,j}\bm{V}_{i,j}=0.
\end{align}
Therefore, $\bm{V}$ can be updated by
\begin{align}\label{gengv}
\bm{V}_{i,j}\longleftarrow \bm{V}_{i,j}\cdot\sqrt{\frac{[\bm{Z}_1]_{i,j}^{+}+[\bm{Z}_2]_{i,j}^{-}}{[\bm{Z}_1]_{i,j}^{-}+[\bm{Z}_2]_{i,j}^{+}}},
\end{align}
where $\bm{Z}_1=\sum_vw_{v}(\bm{G}^{(v)}\odot\bm{X}^{(v)})^T\bm{U}^{(v)}$ and $\bm{Z}_2=\sum_vw_{v}(\bm{G}^{(v)}\odot\bm{U}^{(v)}\bm{V}^T)^T\bm{U}^{(v)}+\alpha\bm{D}_V\bm{V}$.

In Step 1 and Step 3, we often need to normalize $\bm{U}^{(v)}$ and $\bm{V}$ to ensure the accuracy of the updating rules \cite{shao2015multiple}, so $\bm{U}^{(v)}$ and $\bm{V}$ can be normalized by $\bm{U}^{(v)}\longleftarrow \bm{U}^{(v)}\bm{Q}$, $\bm{V}\longleftarrow \bm{V}\bm{Q}^{-1}$,
%\begin{align}\label{quv}
%\bm{U}^{(v)}\longleftarrow \bm{U}^{(v)}\bm{Q},\bm{V}\longleftarrow \bm{V}\bm{Q}^{-1},
%\end{align}
where $\bm{Q}$ is the diagonal matrix $\bm{Q}_{k,k}^{(v)}=\sum_t{\bm{V}_{t,k}}$.

\textbf{Step 4.}
Updating $w_{v}$. Fixing the other variables, we can update the variable $w_{v}$ by Eq.~\eqref{wzuixin}.

Our proposed ANIMC algorithm is shown in Algorithm~\ref{simc_alg}. We provide its codes in GitHub (\url{https://github.com/ZeusDavide/TAI_2021_ANIMC}).
%After obtaining the latent feature matrix $\bm{V}$,
%we perform K-means clustering on it for the final clustering results.
\subsection{Convergence and Complexity}\label{subsection:fuzadu}
\subsubsection{Convergence Analysis}\label{subsubsection:shoulianfenxi}
To optimize our proposed ANIMC, we need to solve four subproblems in Algorithm \ref{simc_alg}. Each subproblem is convex and has the closed solution w.r.t corresponding variable. Thus, the objective function Eq.~\eqref{wansimc} will reduce monotonically to a stationary point, which ensures that ANIMC can at least find a locally optimal solution.
%As shown in Algorithm~\ref{simc_alg}, the optimization of our proposed ANIMC can be divided into four subproblems, each of which is convex w.r.t corresponding variable. Thus, by learning the optimal solution for each subproblem alternatively, ANIMC can at least find a locally optimal solution.
\subsubsection{Complexity Analysis}\label{subsubsection:fuzadufenxi}
As shown in Algorithm \ref{simc_alg}, the operations of updating four parameters ($\bm{U}^{(v)},\bm{A}^{(v)},\bm{V}$, and $w_v$) determine the computational complexity of the algorithm. Assume $i$ is the number of iterations, the time complexities of updating $\bm{U}^{(v)},\bm{A}^{(v)},\bm{V}$ and $w_v$ are respectively $O(d_v^2c+n(d_v+c)),O(c(d_v^2+c^2)),O(ic(nd_v+cd_v+nc))$ and $O(nc^2)$. Since $c \ll \min(d_v,n)$ in most image processing applications, the total complexity is $O(\max(d_vin(ci+d_v^2)))$.

\section{Performance Evaluation}
 \label{section:exp}
\subsection{Datasets}\label{subsection:dataset}
\begin{table}[ht]
\centering
\caption{Statistics of the datasets}
\begin{tabular}{ccccccc}
\hline
Dataset  & \#. instances & \#. views  & \#. clusters & \#. features \\\hline
BBCSport    & 544  & 2    & 5 &6386  \\
BUAA    & 180  & 2  & 20 &200   \\
Caltech7 & 1474 & 6 &7 &3766\\
Digit       & 2000  & 5    & 10  &585 \\
NH-face & 4660 & 5 & 5 &12054\\
Scene   & 2688  & 4   & 8 &1248\\\hline
\end{tabular}
\label{dataset}
\end{table}
%\noindent\textbf{4.1 Datasets}
%\begin{figure*}[t]
%\centering
%\subfigure[ACCs for BBCSport]{\label{fig:acc_bbcsport} \includegraphics[width=0.315\textwidth]{figure/acc_bbcsport} }
%\subfigure[NMIs for BBCSport]{\label{fig:nmi_bbcsport} \includegraphics[width=0.315\textwidth]{figure/nmi_bbcsport} }
%\subfigure[Purities for BBCSport]{\label{fig:purity_bbcsport} \includegraphics[width=0.315\textwidth]{figure/purity_bbcsport} }
%\subfigure[ACCs for BUAA]{\label{fig:acc_buaa} \includegraphics[width=0.315\textwidth]{figure/acc_buaa} }
%\subfigure[NMIs for BUAA]{\label{fig:nmi_buaa} \includegraphics[width=0.315\textwidth]{figure/nmi_buaa} }
%\subfigure[Purities for BUAA]{\label{fig:purity_buaa} \includegraphics[width=0.315\textwidth]{figure/purity_buaa}}
%\subfigure[ACCs for Digit]{\label{fig:acc_digit} \includegraphics[width=0.315\textwidth]{figure/acc_digit}}
%\subfigure[NMIs for Digit]{\label{fig:nmi_digit} \includegraphics[width=0.315\textwidth]{figure/nmi_digit}}
%\subfigure[Purities for Digit]{\label{fig:purity_digit} \includegraphics[width=0.315\textwidth]{figure/purity_digit}}
%\subfigure[ACCs for Scene]{\label{fig:acc_scene} \includegraphics[width=0.315\textwidth]{figure/acc_scene}}
%\subfigure[NMIs for Scene]{\label{fig:nmi_scene} \includegraphics[width=0.315\textwidth]{figure/nmi_scene} }
%\subfigure[Purities for Scene]{\label{fig:purity_scene} \includegraphics[width=0.315\textwidth]{figure/purity_scene}}
%\caption{Incomplete multi-view clustering results on various four datasets.}
%% \setlength{\abovecaptionskip}{1in}
%\label{fig:AwIMC_four}
%\end{figure*}
\begin{table*}
\centering
\caption{Incomplete multi-view clustering results on various datasets. \textbf{Bold} numbers denote the best results.}
\begin{tabular}{c|c|ccc|ccc|ccccccccccccccccccccccccc}
\hline
\multirow{2}*{Dataset}&\multirow{2}*{Method}& \multicolumn{3}{c}{ACC (\%)}&\multicolumn{3}{|c|}{NMI (\%)}&\multicolumn{3}{c}{Purity (\%)}\\\cline{3-11}
 ~& ~ & PER=0.1&  PER=0.3&  PER=0.5&  PER=0.1& PER=0.3& PER=0.5&  PER=0.1& PER=0.3& PER=0.5 \\\hline
\multirow{12}*{BBCSport} & AGC$\_$IMC& 56.01& 43.64& 39.18& 50.75& 29.47& 17.43& 54.65& 40.85& 39.68\\
~&BSV & 53.19& 42.19& 30.86&47.71& 35.86& 23.75& 63.64 & 52.74& 47.02\\
~&Concat& 50.36&45.81&36.98& 50.90&32.97&19.07&68.91&53.44&44.81\\
~&DAIMC&40.42 &37.81 &33.17 &23.17 &14.08& 11.63&41.16 &36.11& 32.91\\
~&EE-IMVC& 55.97& 43.97& 40.77 & 51.66& 30.15&18.53&50.42&37.18& 30.11\\
~&EE-R-IMVC&50.12&40.03& 34.64& 50.65& 31.50& 19.64&51.96 &38.69 & 29.76\\
~& MIC& 38.71 & 31.94 & 30.15 & 20.41& 18.40& 16.85 & 48.35&40.01& 36.13\\
~&MLAN&47.13& 29.61& 19.46 & 17.42 & 5.26 & 4.80& 40.13&32.96& 26.73 \\
~&NMF-CC& 43.57& 30.01& 21.44& 14.01& 6.13 & 5.08&49.45& 35.26& 29.08  \\
~&UEAF& 58.57 & 47.26 & 45.68 & 50.41 & 32.44& 21.98& 48.97& 43.26 & 31.07\\
~&UIMC&60.03& 47.86 &39.94& 52.74& 30.41& 22.86&53.64 &42.75&40.39\\
~&Our ANIMC& \textbf{65.75} & \textbf{51.62} & \textbf{47.63}& \textbf{54.27}& \textbf{50.08}& \textbf{25.60} &\textbf{60.02} & \textbf{52.64} & \textbf{49.43}\\\hline
\multirow{12}*{BUAA} &AGC$\_$IMC&72.04&56.80&27.93&78.60&64.99&47.01&74.30&52.60&41.67\\
~&BSV&31.17&26.93&10.86&39.07&30.26&25.01& 33.06&20.31&11.48\\
~&{{Concat}}&{{32.09}}&{{27.81}}&{{13.60}}&{{41.07}}&{{32.49}}&{{26.33}}&{{35.12}}&{{20.91}}&{{13.04}}\\
~&{{DAIMC}}&{{41.75}}& {{38.26}}&{{35.53}}&{{57.64}}&{{50.17}}&{{43.10}}&{{46.26}}&{{39.21}}&{{34.52}}\\
~&{{EE-IMVC}}&{{70.01}}&{{57.23}}&{{28.07}}&{{59.00}}&{{50.08}}&{{42.66}}&{{73.01}}&{{51.77}}&{{40.04}}\\
~&{{EE-R-IMVC}}&{{71.08}}&{{58.99}}&{{30.03}}&{{63.71}}&{{52.80}}&{{45.96}}&{{75.99}}&{{53.87}}&{{40.08}}\\
~&{{MIC}}&{{34.11}}&{{28.33}}&{{22.78}}&{{50.61}}&{{46.89}}&{{38.90}}&{{40.00}}&{{28.87}}& {{15.46}}\\
~&{{MLAN}}&{{34.44}}&{{24.06}}&{{11.12}}&{{36.88}}&{{28.19}}&{{6.33}}&{{36.67}}&{{25.56}}& {{11.74}} \\
~&{{NMF-CC}}&{{50.13}}&{{39.41}}&{{22.41}}&{{55.10}}&{{51.64}}&{{25.11}}&{{48.11}}&{{41.71}}& {{25.01}} \\
~&{{UEAF}}&{{38.89}}&{{31.67}}&{{21.08}}&{{52.86}}&{{42.83}}&{{33.86}}&{{40.56}}&{{37.22}}& {{33.33}}\\
~&{{UIMC}}&{{72.83}}&{{60.17}}&{{33.96}}&{{77.71}}&{{65.10}}&{{47.93}}&{{76.04}}&{{58.85}}&{{43.07}}\\
~&{{Our ANIMC}}& {{\textbf{78.23}}} & {{\textbf{63.81}}}& {{\textbf{40.16}}} & {{\textbf{82.07} }}&{{ \textbf{69.57} }}& {{\textbf{51.72} }}&{{\textbf{80.06} }}&{{ \textbf{62.11}}}&{{ \textbf{50.42}}}\\\hline
\multirow{12}*{{Caltech7}}&{{AGC$\_$IMC}}&{{61.83}}&{{56.77}}&{{51.32}}&{{60.76}}&{{58.99}}&{{52.37}}&{{70.42}}&{{63.71}}&{{64.32}}\\
~&{{ BSV}}& {{44.74}}&{{33.61}}&{{30.87}}&{{47.32}}&{{40.87}}&{{36.08}}&{{66.47}}&{{60.83}}&{{53.61}}\\
~&{{Concat}}&{{46.51}}&{{35.70}}&{{32.18}}&{{49.13}}&{{42.41}}&{{37.80}}&{{69.44}}&{{61.09}}&{{55.72 }} \\
~&{{DAIMC}}&  {{51.86}}&{{46.33}}&{{43.85}}&{{56.43}}&{{54.17}}&{{48.26}}&{{73.01}}&{{71.18}}&{{67.64}} \\
~&{{EE-IMVC}}&{{61.26}}&{{54.97}}&{{49.93}}&{{61.47}}&{{58.93}}&{{50.14}}&{{76.88}}&{{67.43}}&{{64.52}}\\
~&{{EE-R-IMVC}}&{{61.77}}&{{55.08}}&{{50.16}}&{{61.87}}&{{59.43}}&{{51.08}}&{{77.15}}&{{67.48}}&{{65.72}}\\
~& {{MIC}}&{{37.52}}&{{34.86}}&{{27.99}}&{{48.53}}&{{41.71}}&{{35.75}}&{{72.33}}&{{69.13}}&{{60.07}}\\
~&{{MLAN}}&{{38.01}}&{{29.33}}&{{19.20}}&{{47.62}}&{{30.98}}&{{21.76}}&{{70.42}}&{{51.71}}&{{37.14}}\\
~&{{NMF-CC}}&{{39.15}}&{{24.84}}&{{18.53}}&{{49.80}}&{{27.91}}&{{20.43}}&{{68.77}}&{{50.38}}&{{35.91}}\\
~&{{UEAF}}&{{62.81}}&{54.88}&{{48.97}}&{{61.73}}&{{60.01}}&{{52.77}}&{{78.41}}&{{70.80}}&{{66.42}}\\
~&{{UIMC}}&{{64.27}}&{{54.40}}&{{47.89}}&{{60.72}}&{{57.38}}&{{51.82}}&{{77.21}}&{{70.49}}&{{67.16}}\\
~&{{Our ANIMC}}&{{\textbf{66.76}}}&{{\textbf{56.79}}}&{{\textbf{51.73}}}&{{\textbf{63.42}}}&{{\textbf{62.88}}}&{{\textbf{54.71}}}&{{\textbf{80.06}}}&{{\textbf{74.53}}}&{{\textbf{68.72}}}\\\hline
\multirow{12}*{{Digit }}&{{AGC$\_$IMC}}&{{85.72}}&{{82.66}}&{{70.33}}&{{80.46}}&{{72.89}}&{{69.48}}&{{84.92}}&{{83.08}}&{{71.26}}\\
~&{{BSV}}&{{61.73}}&{{54.36}}&{{35.72}}&{{59.94}}&{{50.17}}&{{48.73}}&{{63.58}}&{{56.41}}&{{34.55}}\\
~&{{Concat}}&{{63.15}}&{{55.72}}&{{37.46}}&{{60.73}}&{{52.18}}&{{50.24}}&{{63.77}}&{{57.03}}&{{35.49}}\\
~& {{DAIMC}}&{{90.35}}&{{87.25}}&{{85.05}}&{{82.32}}&{{77.91}}&{{73.31}}&{{90.17}}&{{87.16}}&{{84.89}}   \\
~&{{EE-IMVC}}&{{56.48}}&{{49.67}}&{{30.76}}&{{52.14}}&{{48.49}}&{{33.62}}&{{56.98}}&{{40.07}}&{{33.89}}\\
~&{{EE-R-IMVC}}&{{57.06}}&{{50.38}}&{{33.61}}&{{54.82}}&{{49.66}}&{{34.83}}&{{58.72}}&{{42.96}}&{{35.17}}\\
~& {{MIC}}&{{52.75}}&{{20.90}}&{{12.05}}&{{45.99}}&{{12.30}}&{{2.08}}&{{52.05}}&{{19.60}}&{{11.95}}\\
~&{{MLAN}}&{{19.75}}&{{10.21}}&{{4.37}}&{{14.52}}&{{8.71}}&{{3.16}}&{{23.10}}&{{6.04}}&{{1.63}}\\
~&{{NMF-CC}}&{{61.25}}&{{59.45}}&{{34.90}}&{{53.02}}&{{47.53}}&{{31.59}}&{{62.60}}&{{56.45}}&{{41.93}}\\
~&{{UEAF}}&{{55.10}}&{{38.85}}&{{30.42}}&{{49.54}}&{{37.50}}&{{28.32}}&{{55.83}}&{{40.16}}&{{32.55}}\\
~&{{UIMC}}&{{92.15}}&{{88.73}}&{{84.62}}&{{87.01}}&{{78.54}}&{{70.88}}&{{92.45}}&{{84.76}}&{{80.24}}\\
~&{{Our ANIMC}}&{{\textbf{95.00}}}&{{\textbf{90.46}}}&{{\textbf{85.85}}}&{{\textbf{89.48}}}&{{\textbf{82.37}}}&{{\textbf{73.04}}}&{{\textbf{95.13}}}&{{\textbf{90.68}}}&{{\textbf{85.90}}}\\\hline
\multirow{12}*{{NH-face}}&{{AGC$\_$IMC}}&{{79.06}}&{{72.83}}&{{63.16}}&{{68.40}}&{{56.73}}&{{46.25}}&{{78.94}}&{{70.12}}&{{66.31}}\\
~&{{BSV}}&{{65.32}}&{{54.26}}&{{47.38}}&{{54.27}}&{{36.71}}&{{22.84}}&{{71.43}}&{{57.88}}&{{48.26}}\\
~&{{Concat}}&{{76.57}}&{{58.42}}&{{50.73}}&{{73.45}}&{{54.88}}&{{44.39}}&{{80.43}}&{{59.16}}&{{48.73}}\\
~& {{DAIMC}}& {{85.40}}&{{80.26}}&{{72.71}}&{{76.43}}&{{70.54}}&{{65.82}}&{{84.97}}&{{80.53}}&{{73.88}}\\
~&{{EE-IMVC}}&{{79.53}}&{{69.78}}&{{61.54}}&{{66.81}}&{{53.17}}&{{44.06}}&{{78.60}}&{{69.85}}&{{66.13}}\\
~&{{EE-R-IMVC}}&{{81.75}}&{{73.84}}&{{64.73}}&{{67.48}}&{{54.22}}&{{45.18}}&{{79.14}}&{{70.98}}&{{67.41}}\\
~& {{MIC}}&{{74.26}}&{{70.50}}&{{68.35}}&{{70.32}}&{{64.59}}&{{59.63}}&{{80.27}}&{{76.40}}&{{72.18}}\\
~&{{MLAN}}&{{52.17}}&{{36.02}}&{{21.03}}&{{50.28}}&{{28.91}}&{{19.76}}&{{53.88}}&{{34.96}}&{{20.18}}\\
~&{{NMF-CC}}&{{76.59}}&{{70.06}}&{{63.84}}&{{71.87}}&{{68.43}}&{{61.80}}&{{78.46}}&{{77.25}}&{{69.43}}\\
~&{{UEAF}}&{{78.29}}&{{70.16}}&{{62.35}}&{{64.22}}&{{51.63}}&{{42.81}}&{{77.92}}&{{71.46}}&{{68.20}}\\
~&{{UIMC}}&{{85.52}}&{{80.88}}&{{73.79}}&{{77.48}}&{{70.53}}&{{65.59}}&{{86.13}}&{{81.44}}&{{72.96}}\\
~&{{Our ANIMC}}&{{\textbf{89.05}}}&{{\textbf{86.73}}}&{{\textbf{85.46}}}&{{\textbf{82.47}}}&{{\textbf{79.93}}}&{{\textbf{76.44}}}&{{\textbf{90.47}}}&{{\textbf{88.11}}}&{{\textbf{87.83}}}\\\hline
\multirow{12}*{{Scene  }}&{{AGC$\_$IMC}}&{{68.05}}&{{60.49}}&{{52.47}}&{{50.16}}&{{43.95}}&{{39.56}}&{{67.80}}&{{50.93}}&{{45.16}}\\
 ~&{{BSV}}&{{64.88}}&{{50.16}}&{{43.11}}&{{41.60}}&{{36.05}}&{{29.31}}&{{60.18}}&{{57.14}}&{{45.72}}\\
~&{{Concat}}&{{65.43}}&{{52.86}}&{{42.97}}&{{45.13}}&{{37.26}}&{{28.40}}&{{61.09}}&{{55.17}}&{{48.32}}\\
~& {{DAIMC}}&{{59.48}}&{{51.36}}&{{47.90}}&{{48.36}}&{{39.99}}&{{30.72}}&{{58.13}}&{{54.56}}&{{47.26}}   \\
~&{{EE-IMVC}}&{{63.08}}&{{54.74}}&{{51.33}}&{{49.86}}&{{43.62}}&{{37.50}}&{{64.06}}&{{55.18}}&{{48.09}}\\
~&{{EE-R-IMVC}}&{{65.88}}&{{55.06}}&{{50.19}}&{{49.72}}&{{44.03}}&{{36.84}}&{{66.95}}&{{51.46}}&{{47.82}}\\
~& {{MIC}}&{{48.22}}&{{29.42}}&{{23.79}}&{{24.88}}&{{8.69}}&{{1.56}}&{{49.33}}&{{32.12}}&{{24.43}}\\
~&{{MLAN}}&{{45.17}}&{{24.80}}&{{10.92}}&{{22.43}}&{{12.38}}&{{7.52}}&{{44.19}}&{{28.54}}&{{13.06}}\\
~&{{NMF-CC}}&{{71.53}}&{{65.26}}&{{59.87}}&{{53.24}}&{{45.41}}&{{42.10}}&{{70.36}}&{{66.72}}&{{57.25}}\\
~&{{UEAF}}&{{62.38}}&{{53.76}}&{{49.23}}&{{50.40}}&{{48.57}}&{{39.16}}&{{61.73}}&{{50.61}}&{{44.82}}\\
~&{{UIMC}}&{{70.89}}&{{68.94}}&{{65.88}}&{{52.04}}&{{51.76}}&{{40.16}}&{{72.15}}&{{70.48}}&{{62.70}}\\
~&{{Our ANIMC}}&{{\textbf{78.19}}}&{{\textbf{75.57}}}&{{\textbf{66.93}}}&{{\textbf{57.29}}}&{{\textbf{52.38}}}&{{\textbf{41.67}}}&{{\textbf{78.19}}}&{{\textbf{75.57}}}&{{\textbf{66.93}}}\\\hline
\end{tabular}
\label{buwanzhengjuleijieguo}
\end{table*}
\begin{figure*}[t]
\centering
\includegraphics[width=0.95\textwidth]{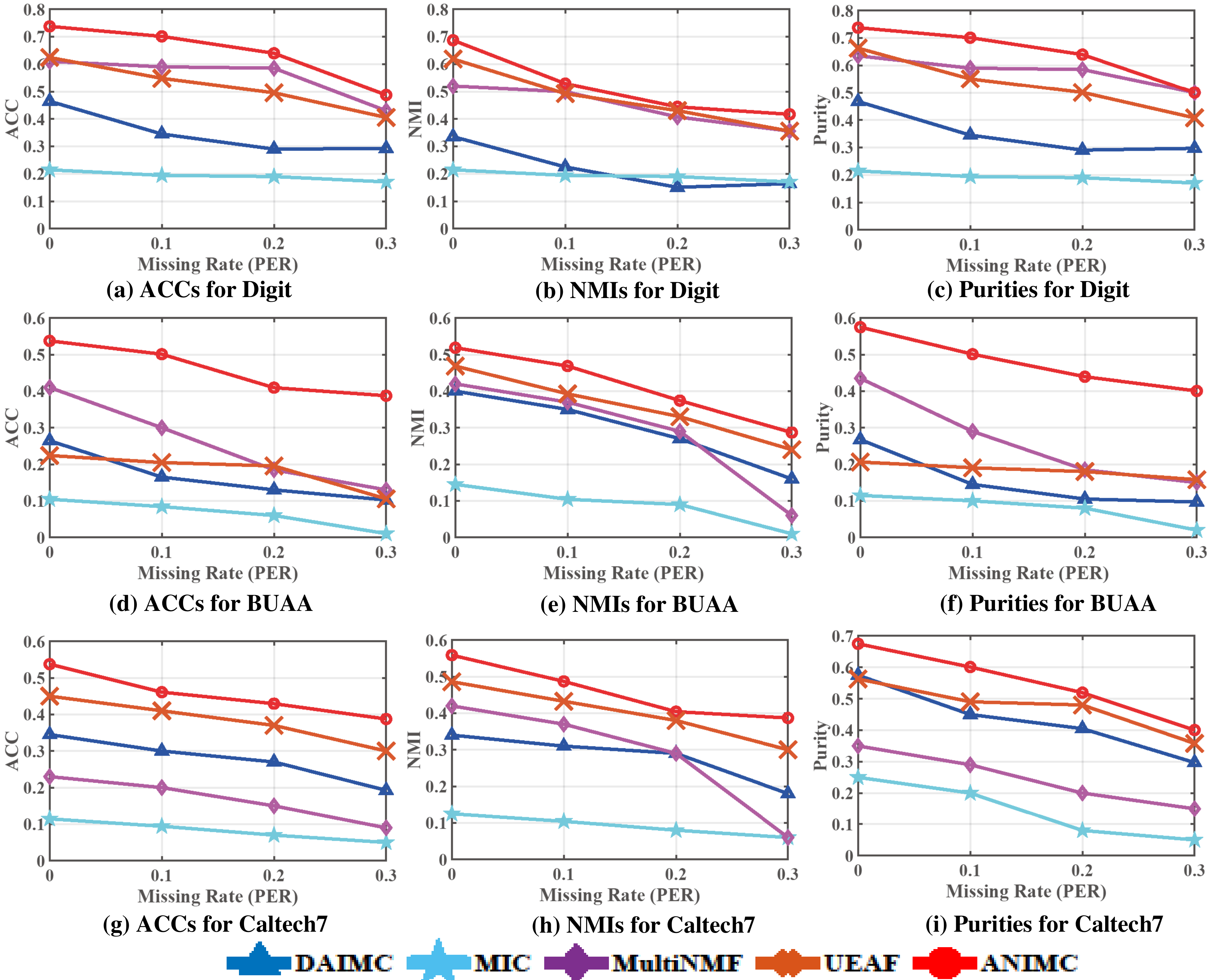}
\caption{Noisy and incomplete multi-view clustering on the Digit, BUAA, and Caltech7 datasets.}
\label{fig:AwIMC_noisy}
\end{figure*}
\begin{figure*}[t]
\centering
\subfigure[$0\leq r\leq 0.6$ (Digit)]{\label{fig:rbiangai1} \includegraphics[width=0.32\textwidth]{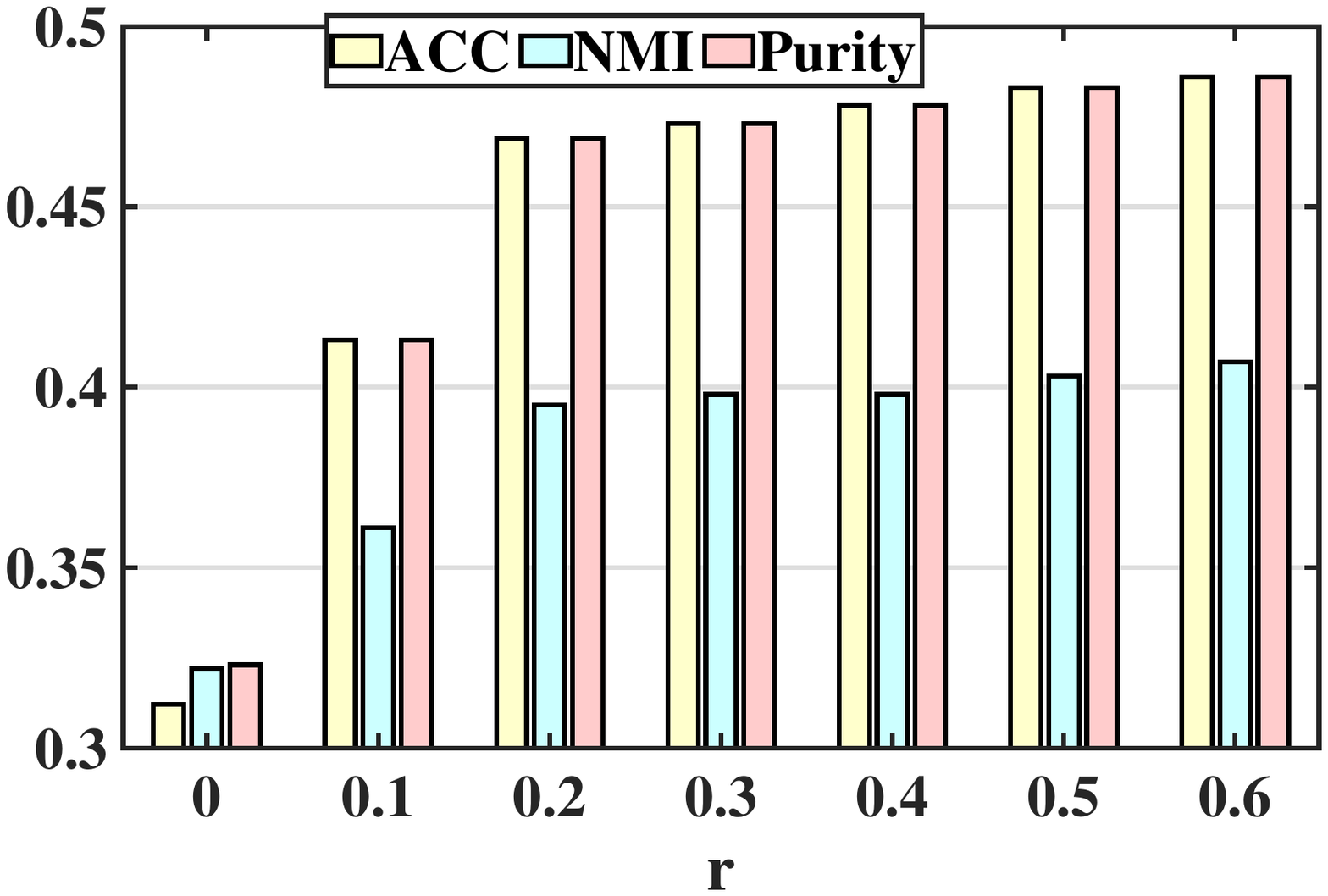}}
\subfigure[$0.7\leq r\leq 1.3$ (Digit)]{\label{fig:rbiangai2} \includegraphics[width=0.32\textwidth]{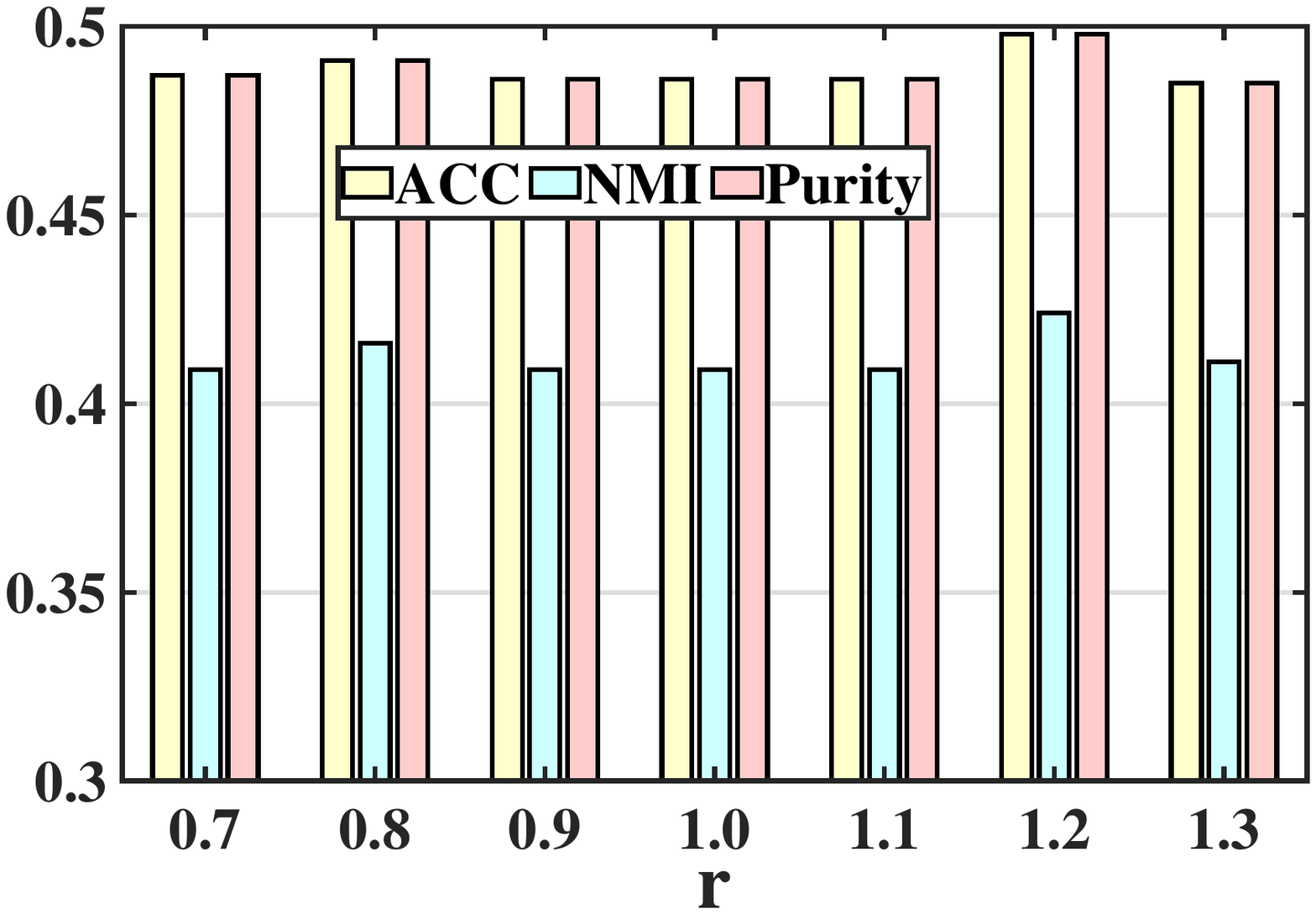}}
\subfigure[$1.4\leq r\leq 2.0$ (Digit)]{\label{fig:rbiangai3} \includegraphics[width=0.32\textwidth]{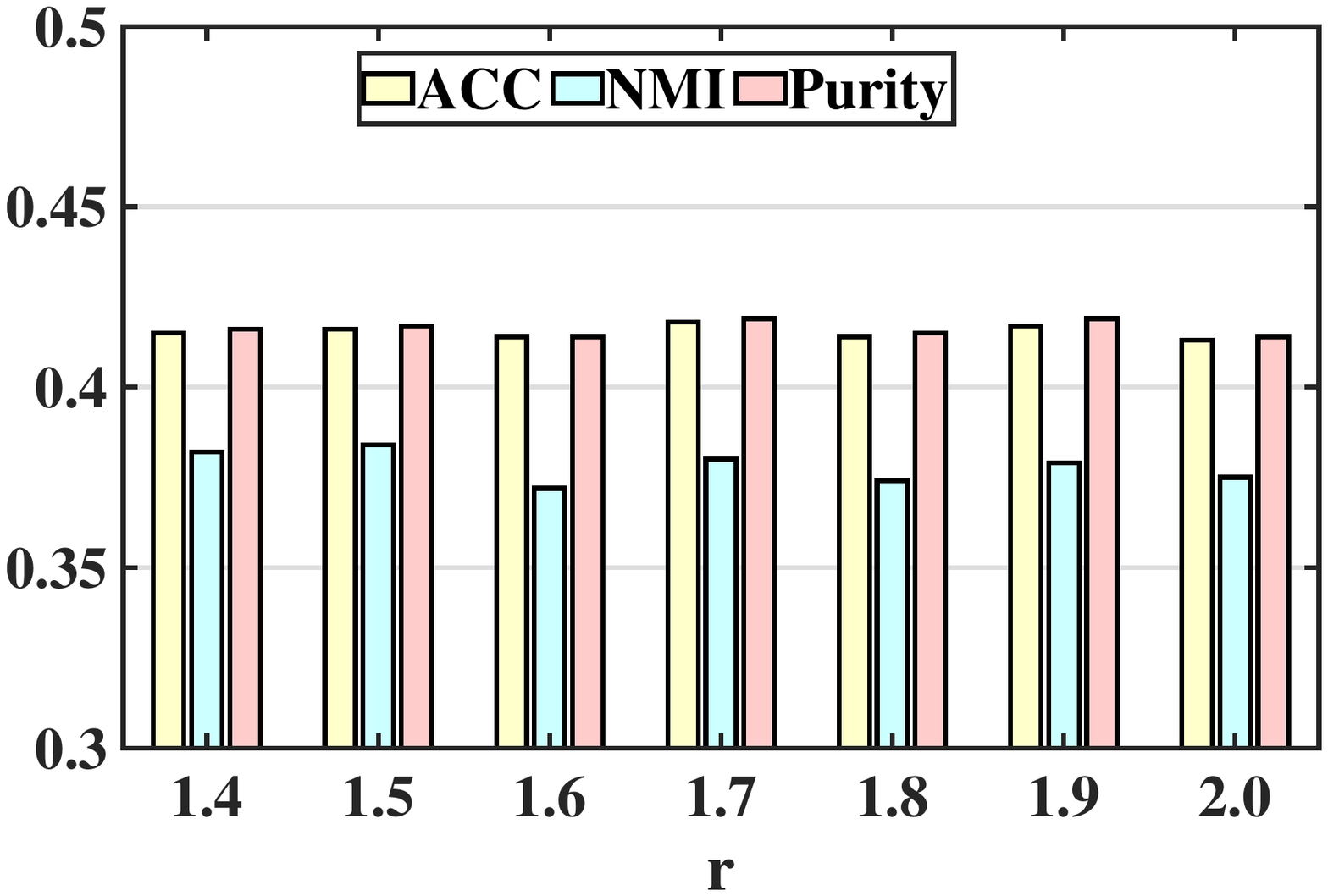}}
\subfigure[$0\leq r\leq 0.6$ (Scene)]{\label{fig:rbiangai1_scene} \includegraphics[width=0.32\textwidth]{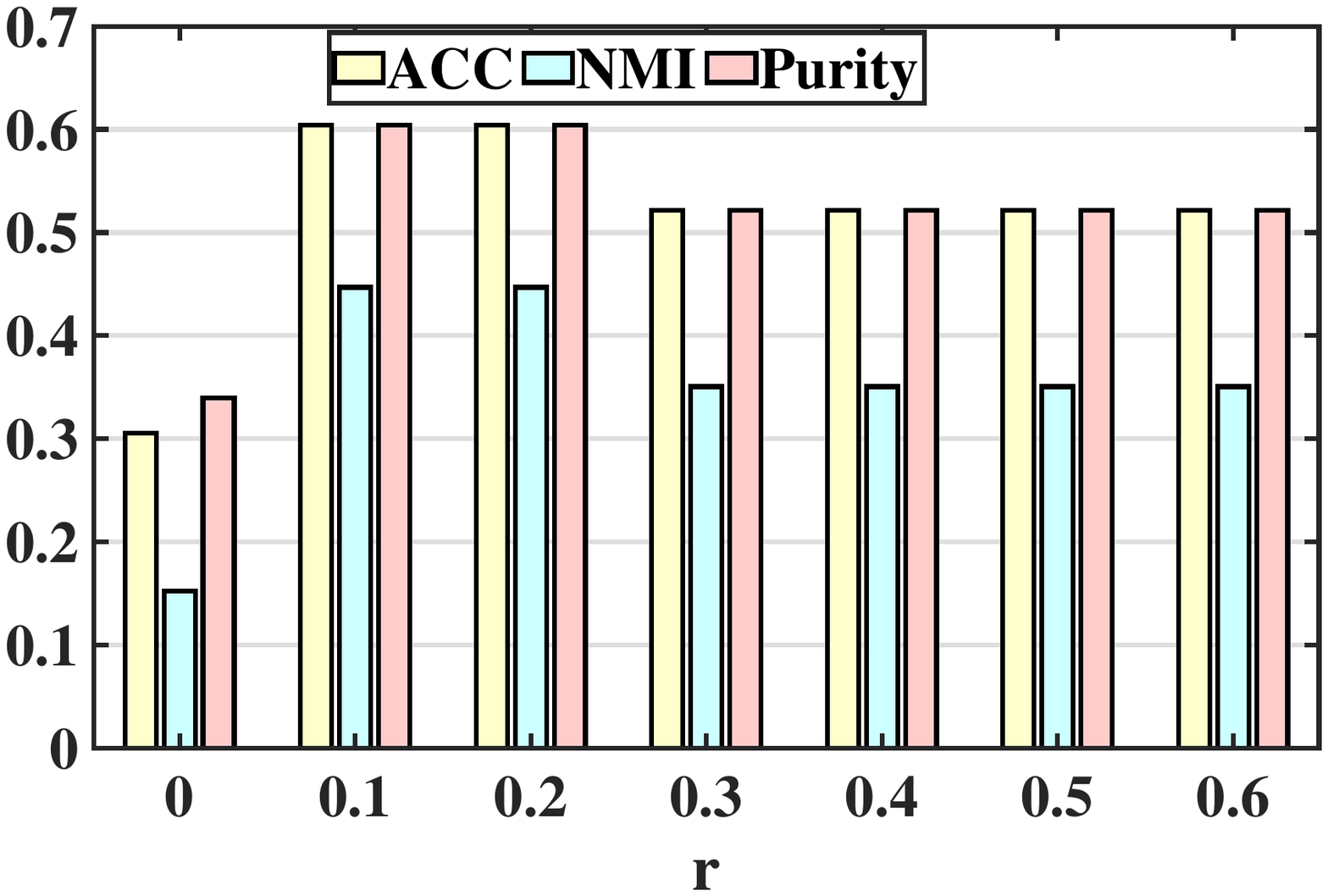}}
\subfigure[$0.7\leq r\leq 1.3$ (Scene)]{\label{fig:rbiangai2_scene} \includegraphics[width=0.32\textwidth]{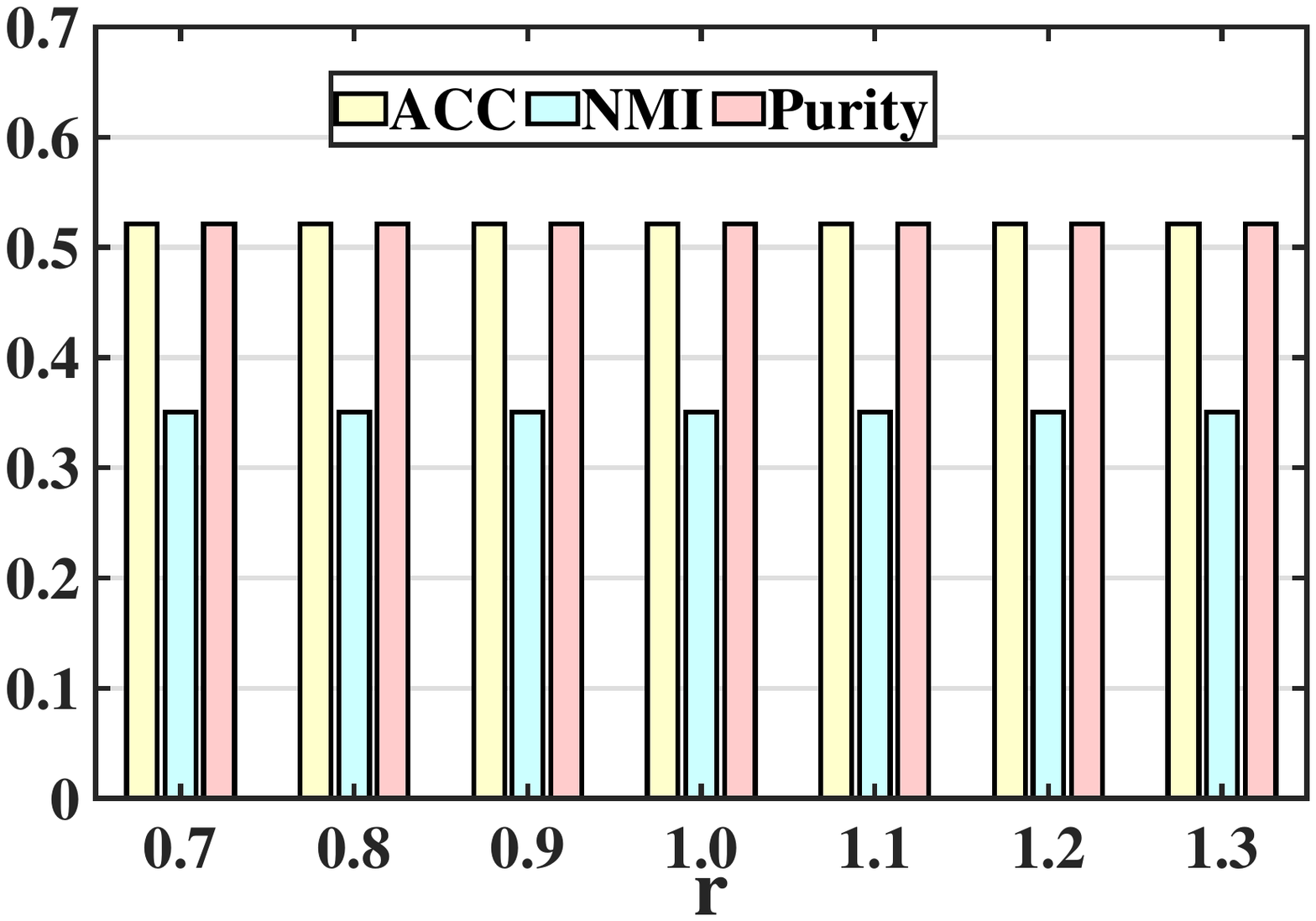}}
\subfigure[$1.4\leq r\leq 2.0$ (Scene)]{\label{fig:rbiangai3_scene} \includegraphics[width=0.32\textwidth]{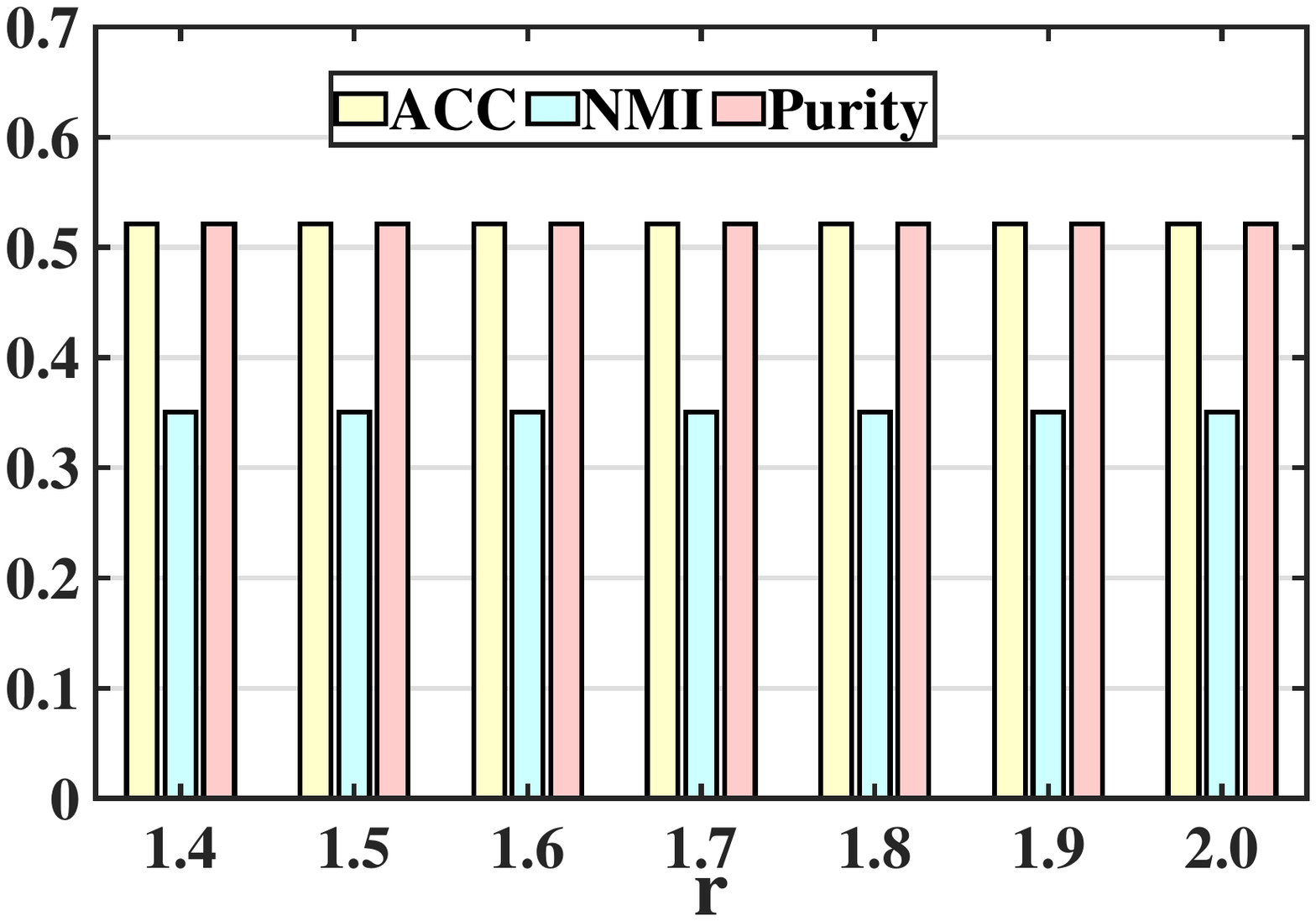}}
\caption{Noisy and incomplete multi-view clustering on the Digit and Scene datasets with different r.}
\label{fig:rbianhua}
\end{figure*}
\begin{figure*}[t]
\centering
\subfigure[PER=0]{\label{fig:shituduibi0} \includegraphics[width=0.32\textwidth]{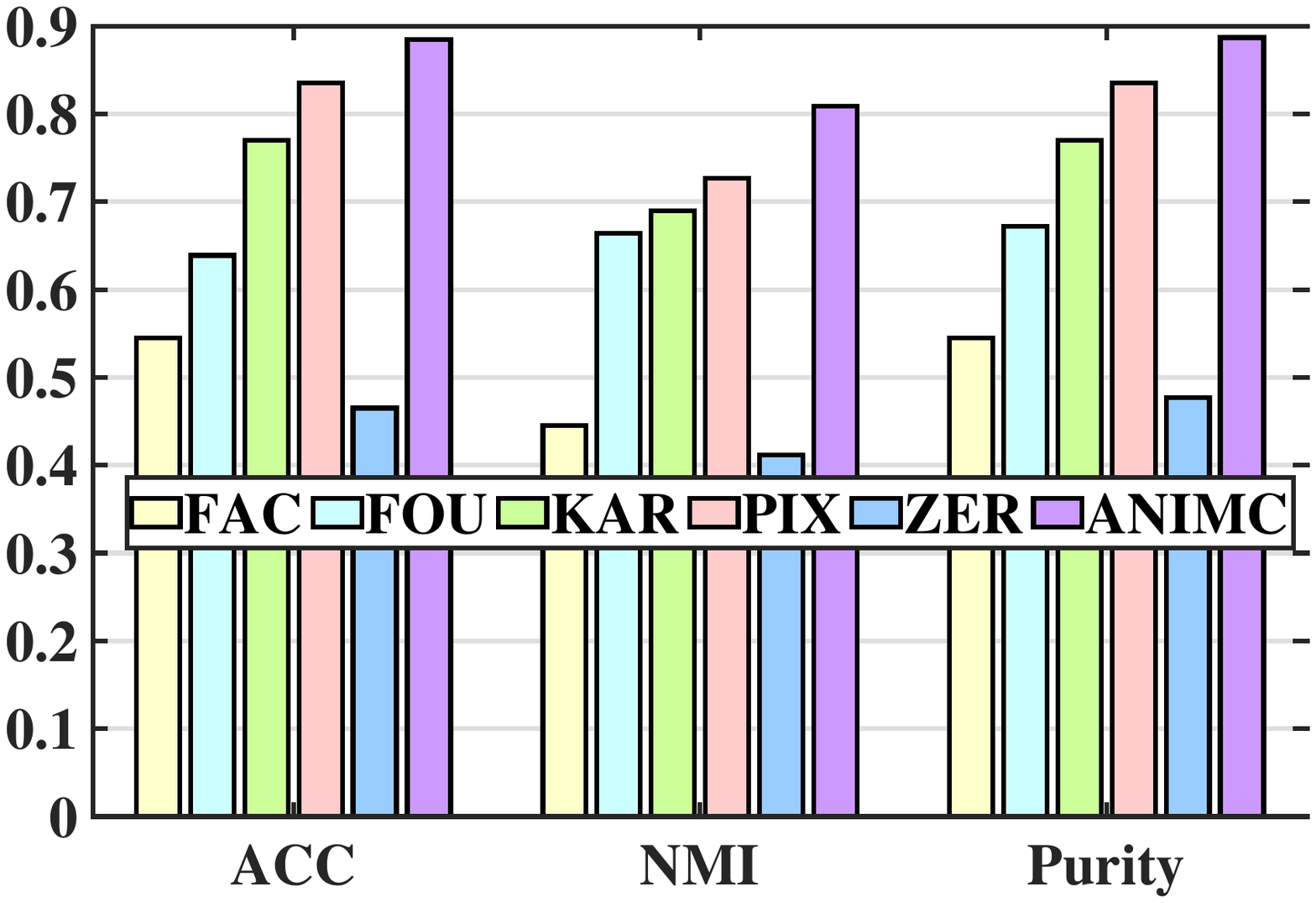}}
\subfigure[PER=0.1]{\label{fig:shituduibi1} \includegraphics[width=0.32\textwidth]{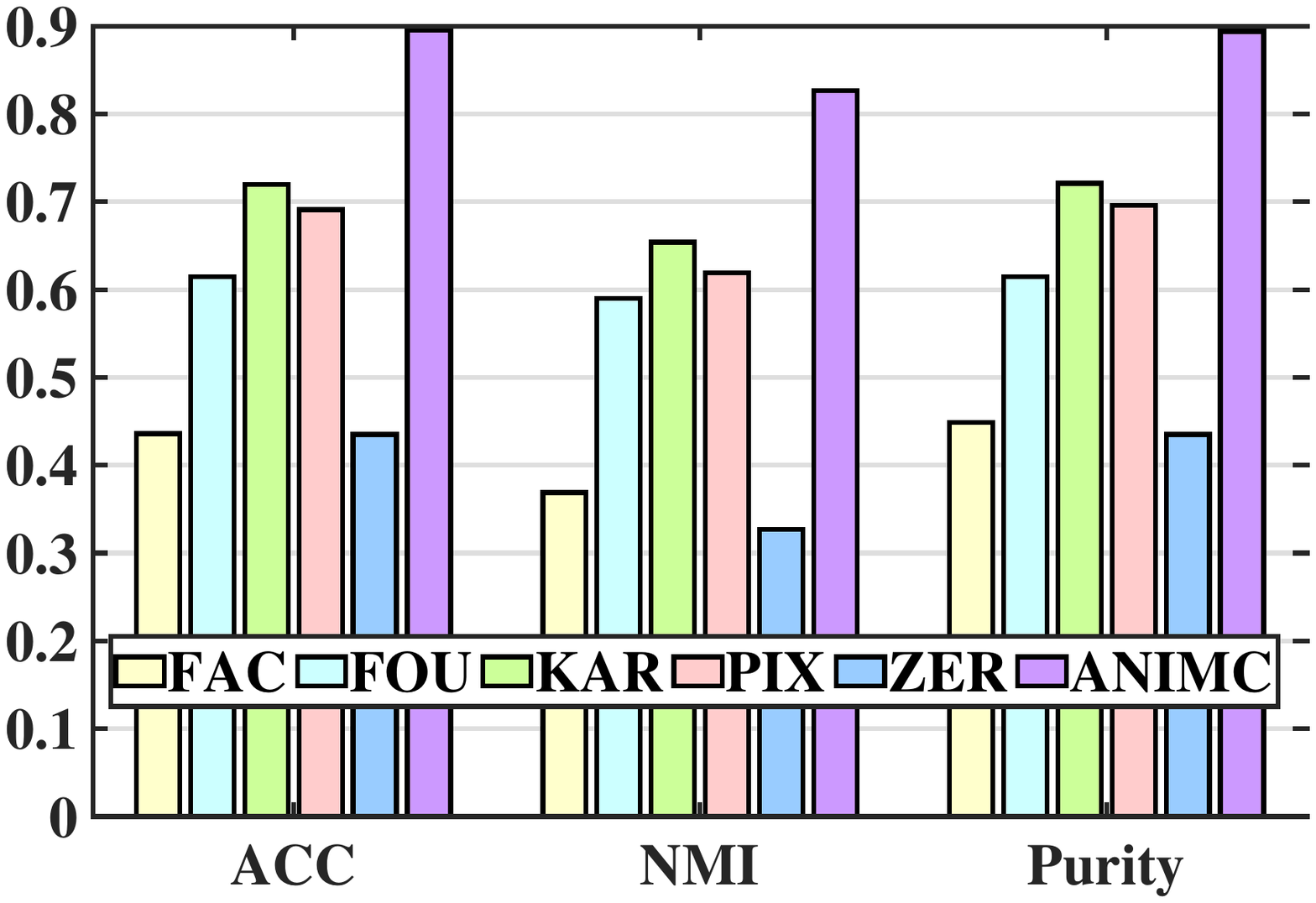}}
\subfigure[PER=0.2]{\label{fig:shituduibi2} \includegraphics[width=0.32\textwidth]{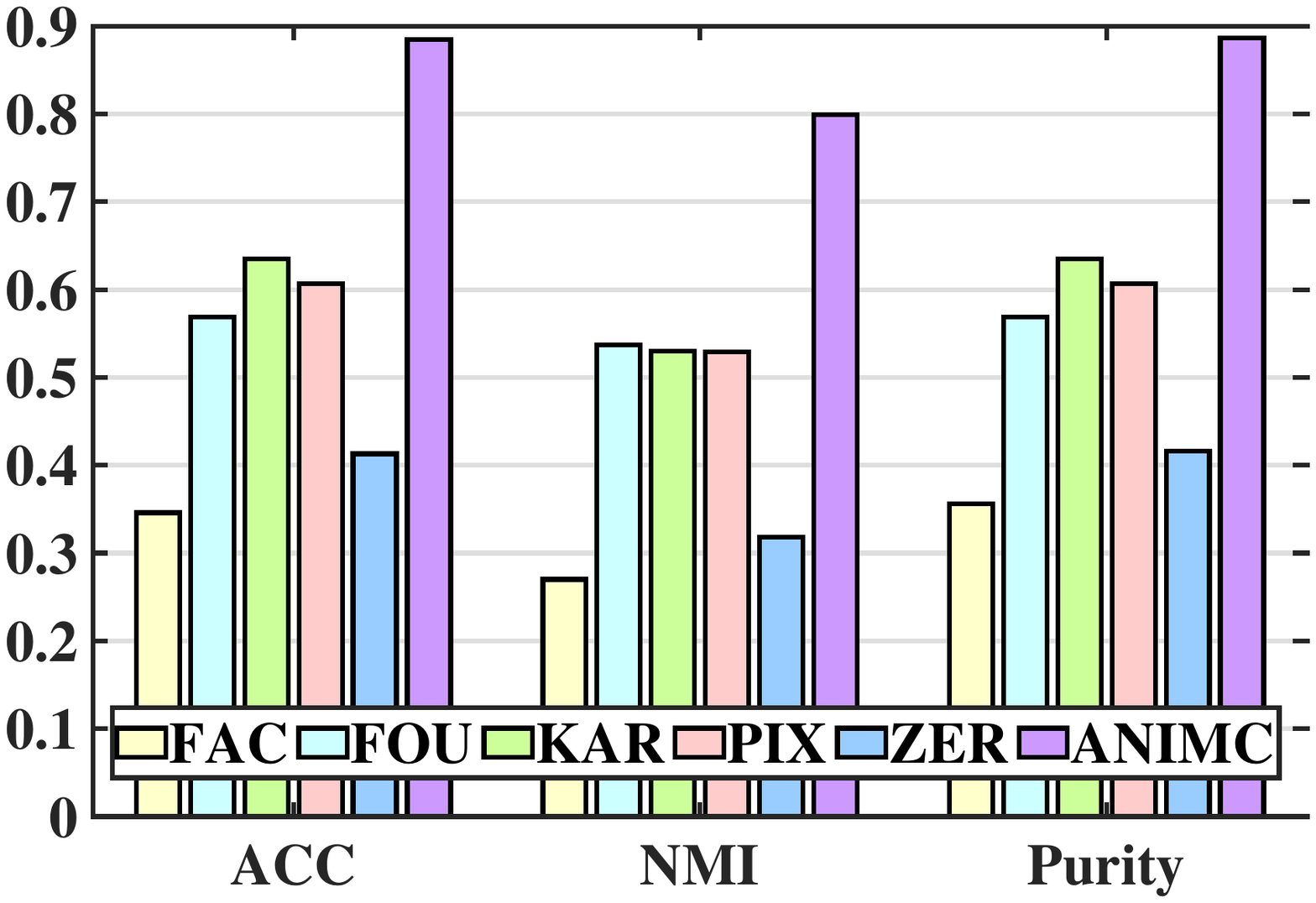}}
\subfigure[PER=0.3]{\label{fig:shituduibi3} \includegraphics[width=0.32\textwidth]{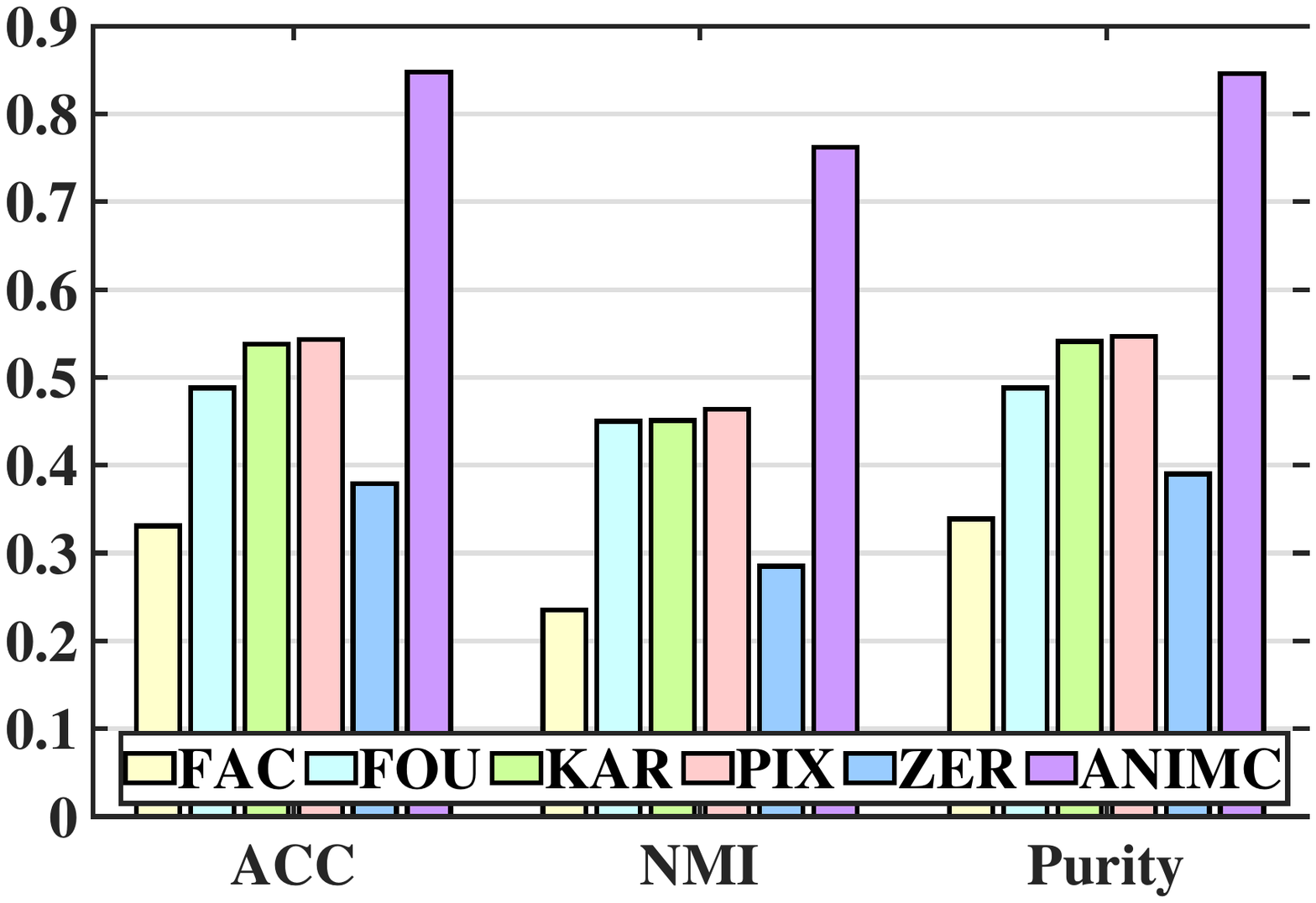}}
\subfigure[PER=0.4]{\label{fig:shituduibi4} \includegraphics[width=0.32\textwidth]{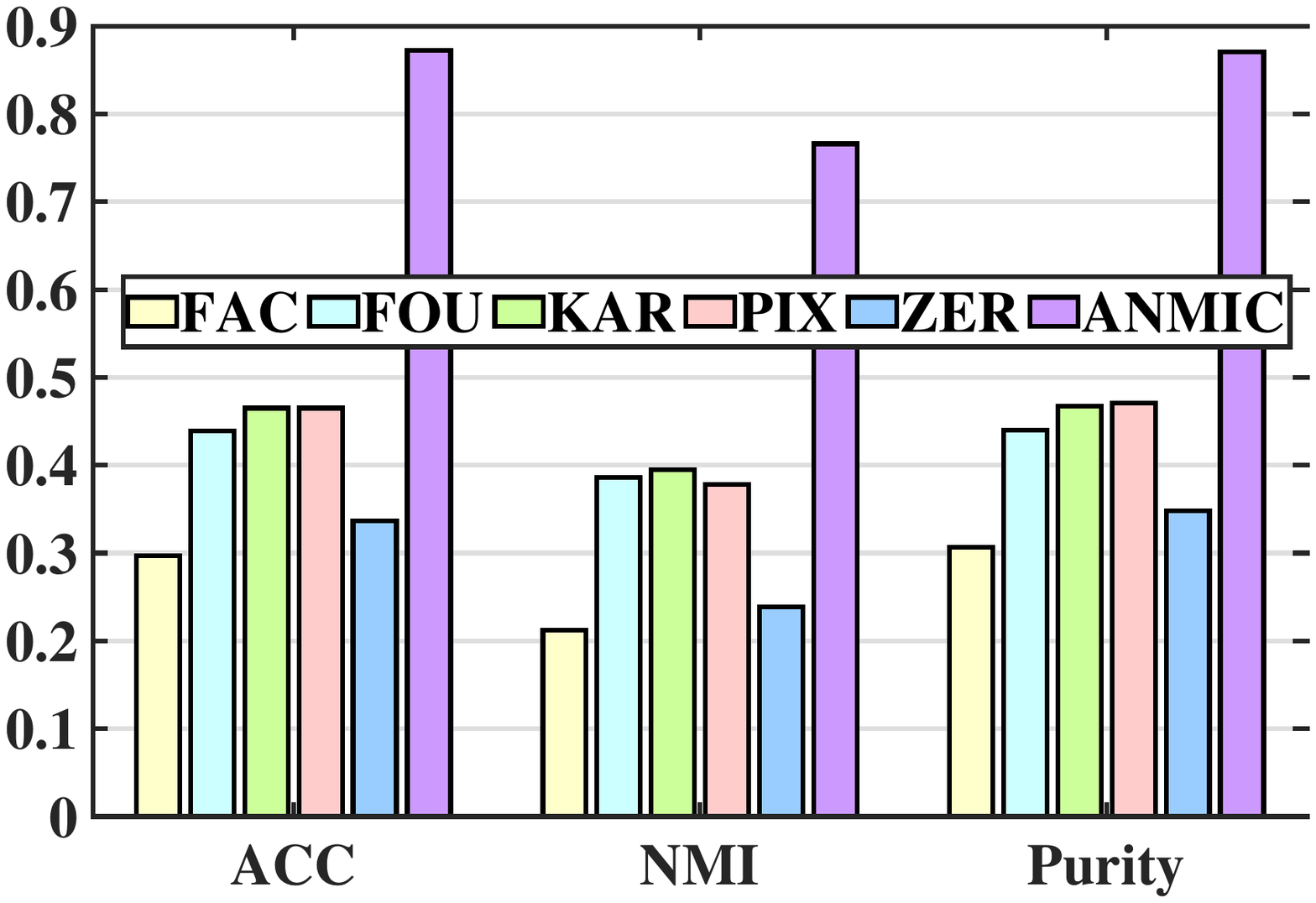}}
\subfigure[PER=0.5]{\label{fig:shituduibi5} \includegraphics[width=0.32\textwidth]{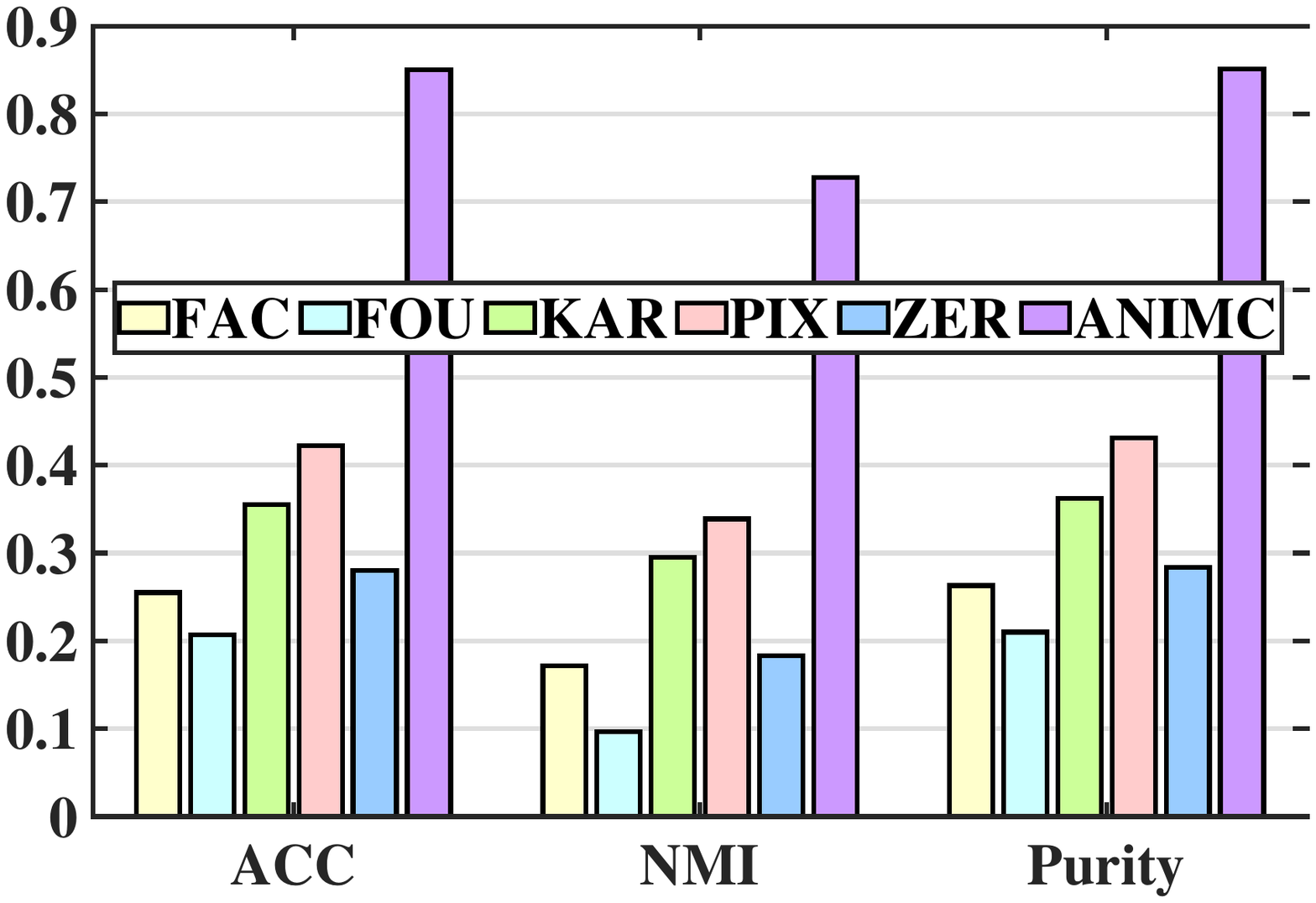}}
\caption{Availability comparison of different views for incomplete multi-view clustering on the Digit dataset.}
\label{fig:usability}
\end{figure*}
\begin{figure*}[t]
\centering
\subfigure[PER=0]{\label{fig:shituduibi0_scene} \includegraphics[width=0.32\textwidth]{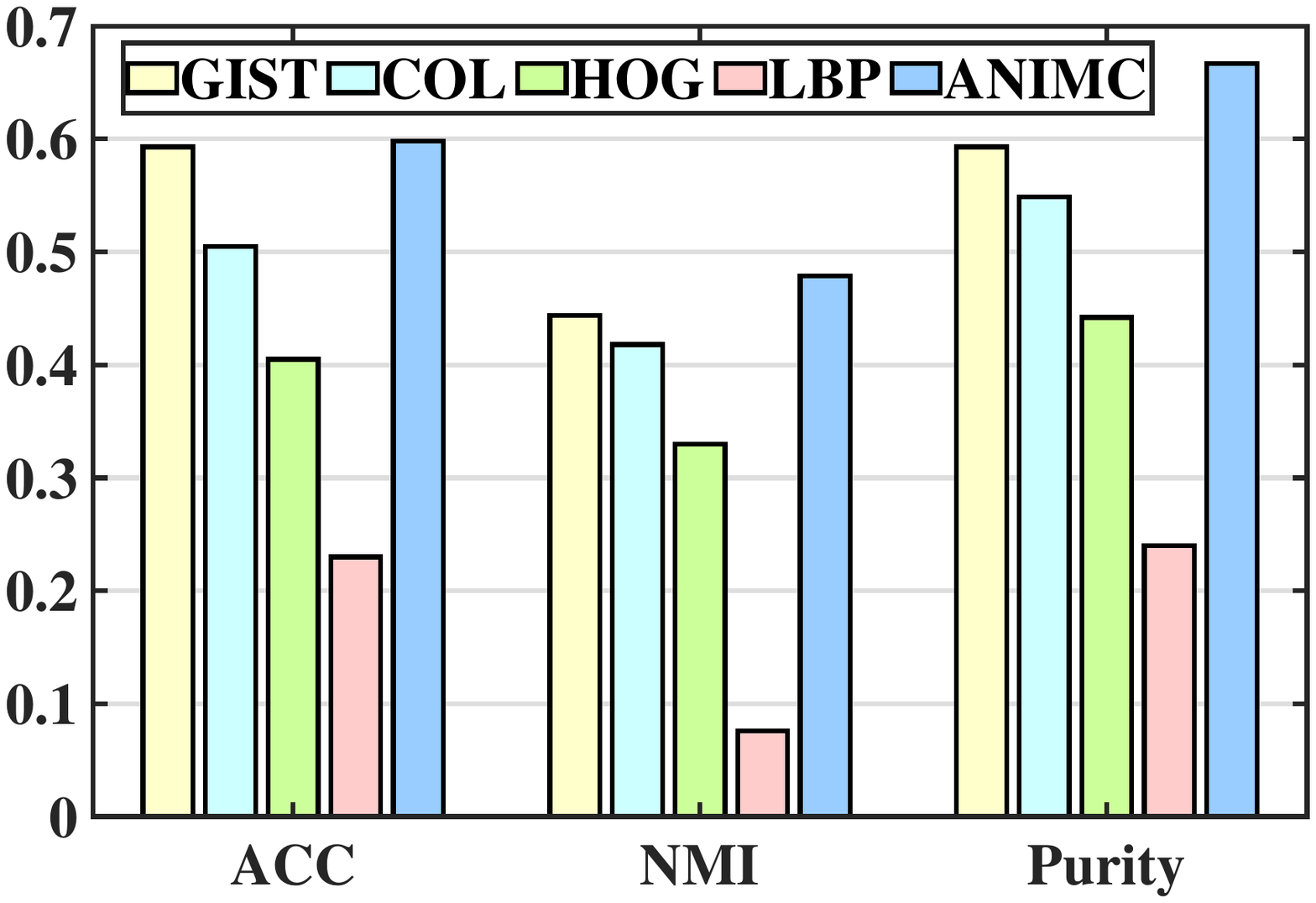}}
\subfigure[PER=0.1]{\label{fig:shituduibi1_scene} \includegraphics[width=0.32\textwidth]{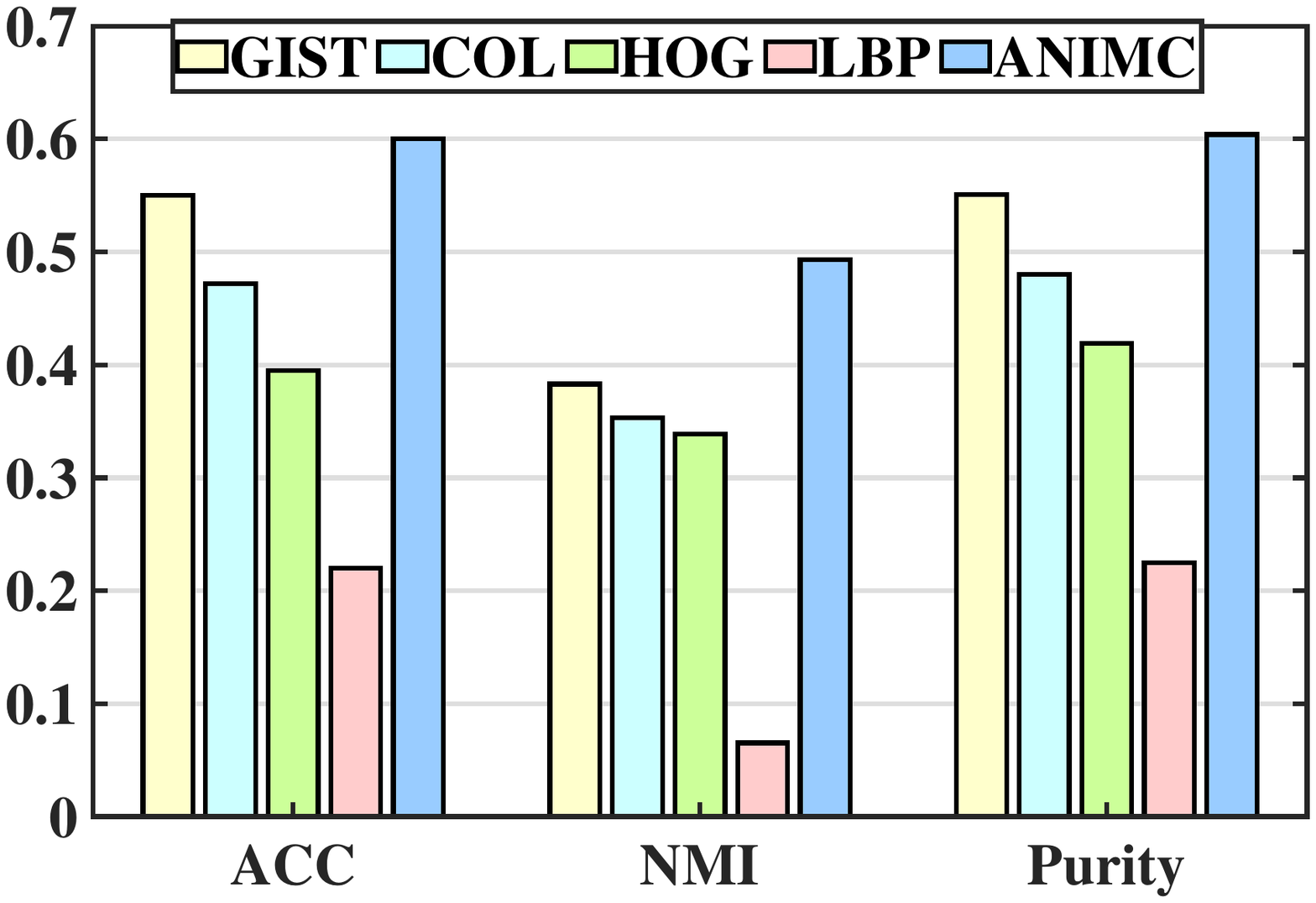}}
\subfigure[PER=0.2]{\label{fig:shituduibi2_scene} \includegraphics[width=0.32\textwidth]{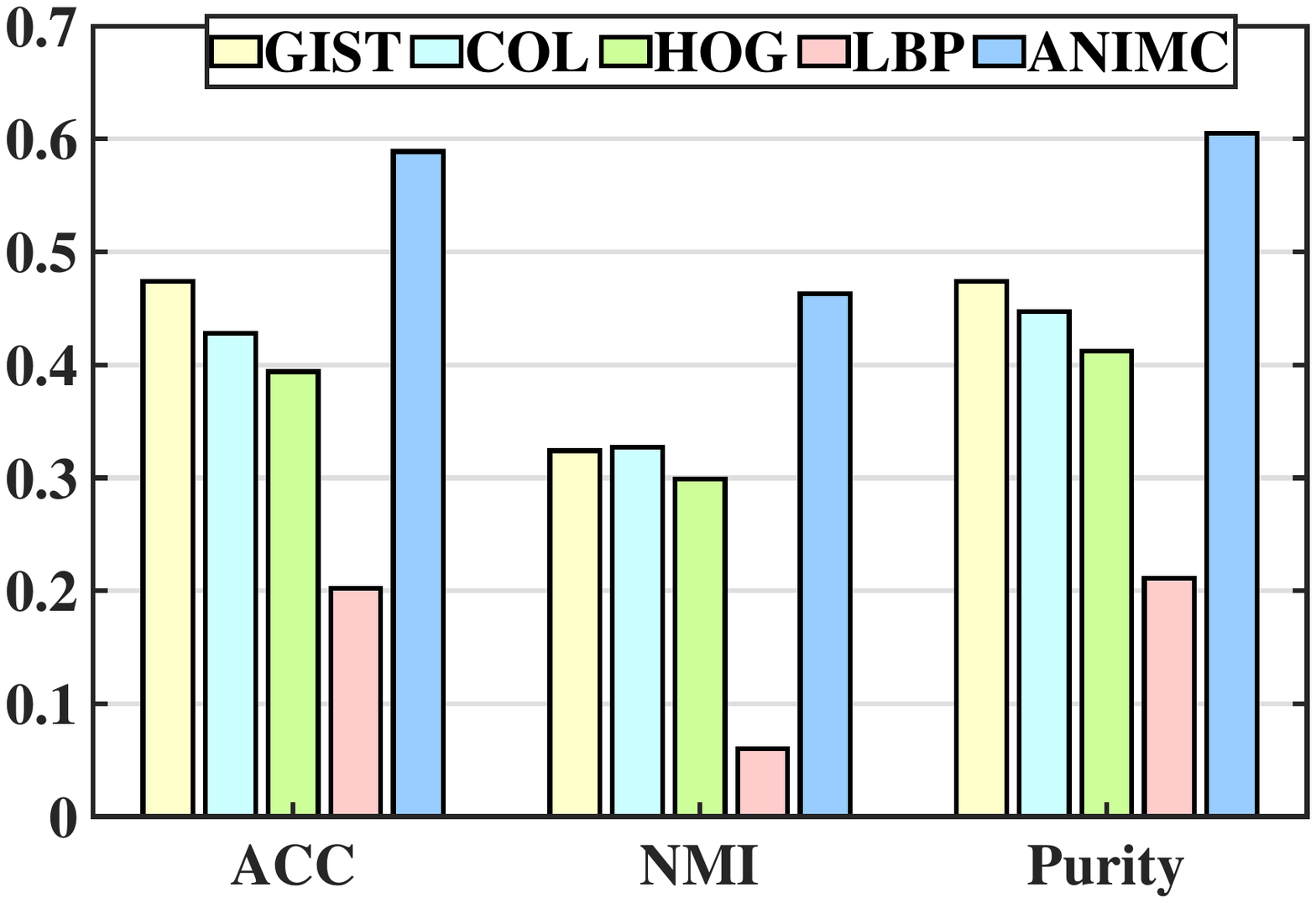}}
\subfigure[PER=0.3]{\label{fig:shituduibi3_scene} \includegraphics[width=0.32\textwidth]{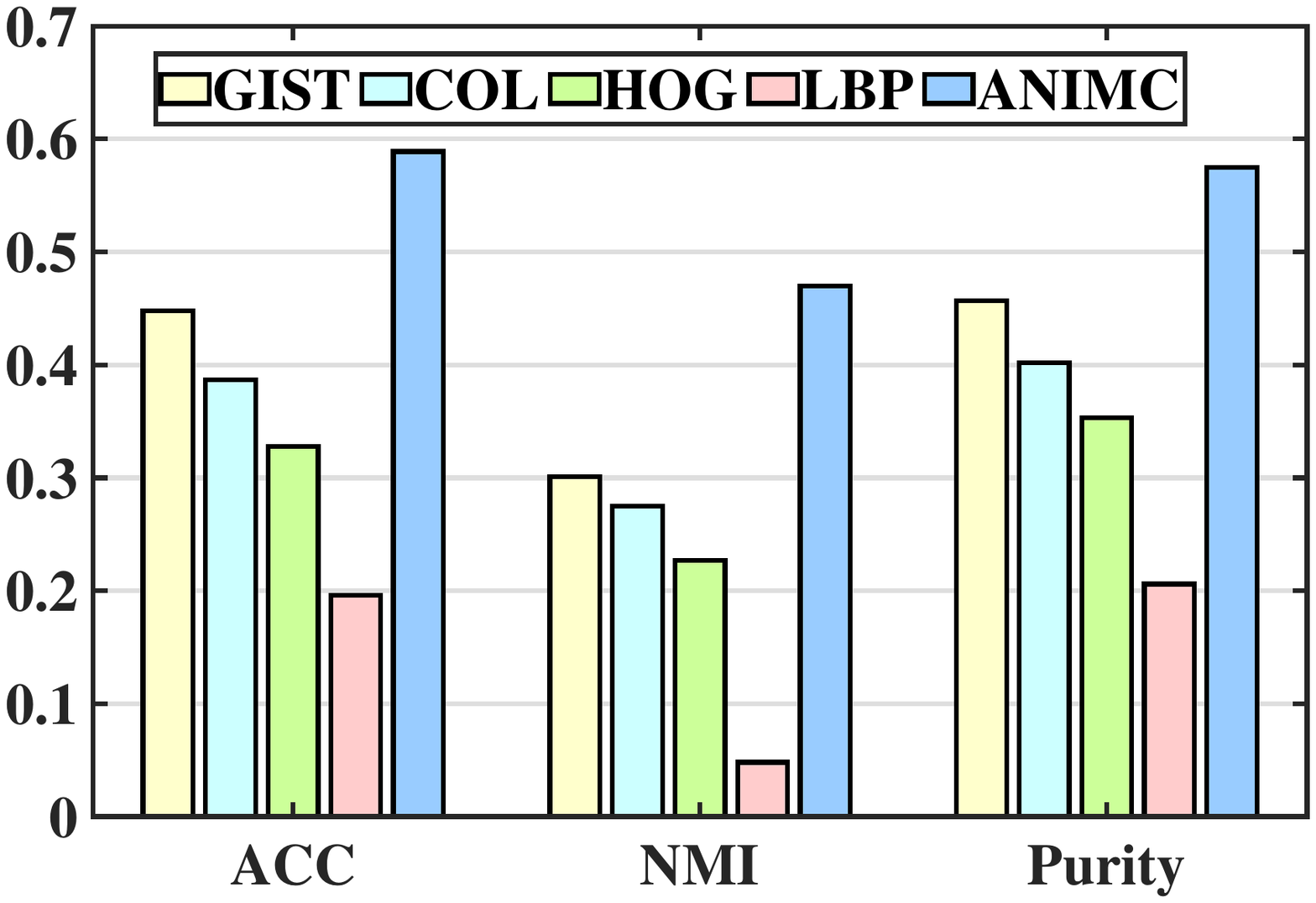}}
\subfigure[PER=0.4]{\label{fig:shituduibi4_scene} \includegraphics[width=0.32\textwidth]{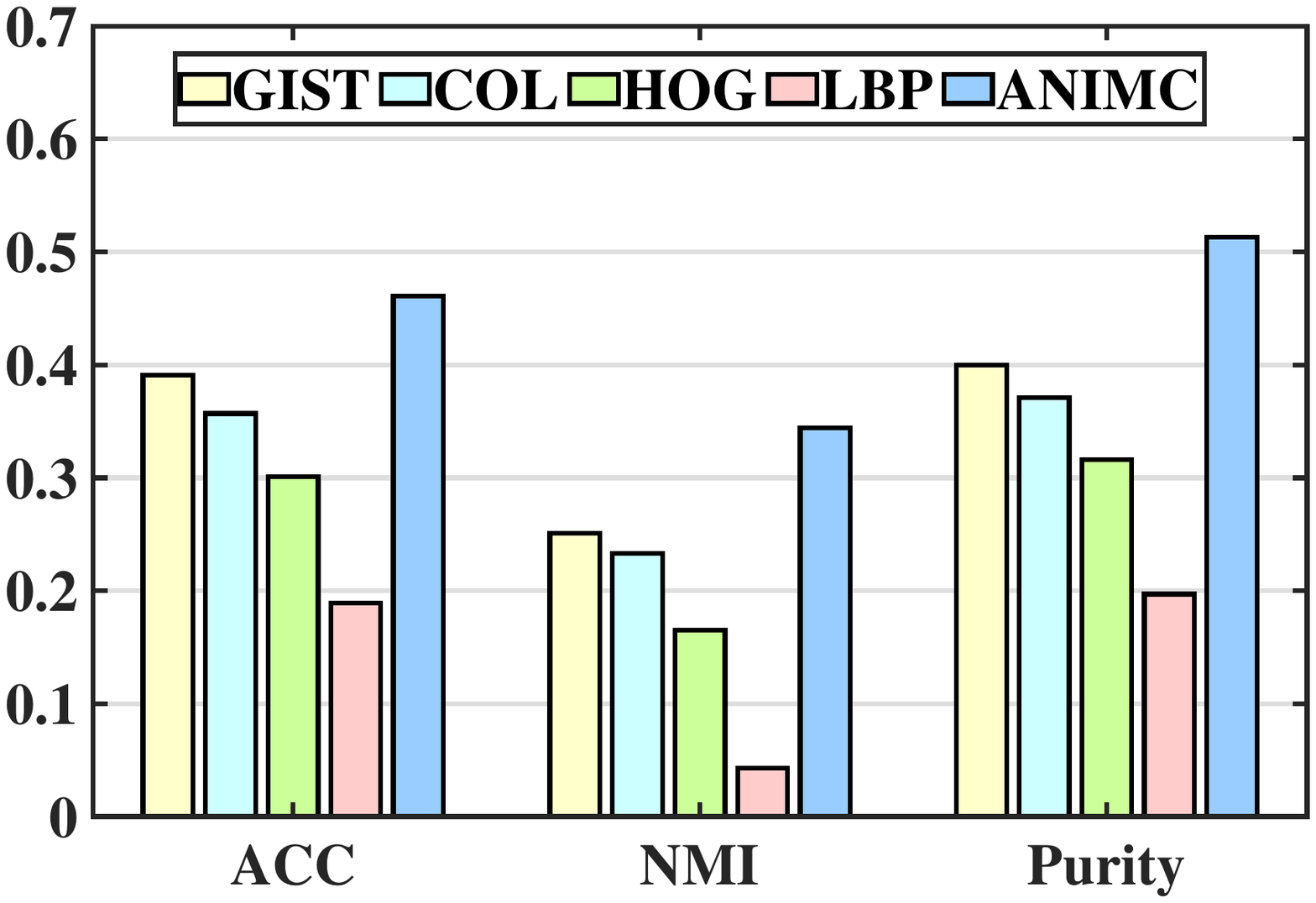}}
\subfigure[PER=0.5]{\label{fig:shituduibi5_scene} \includegraphics[width=0.32\textwidth]{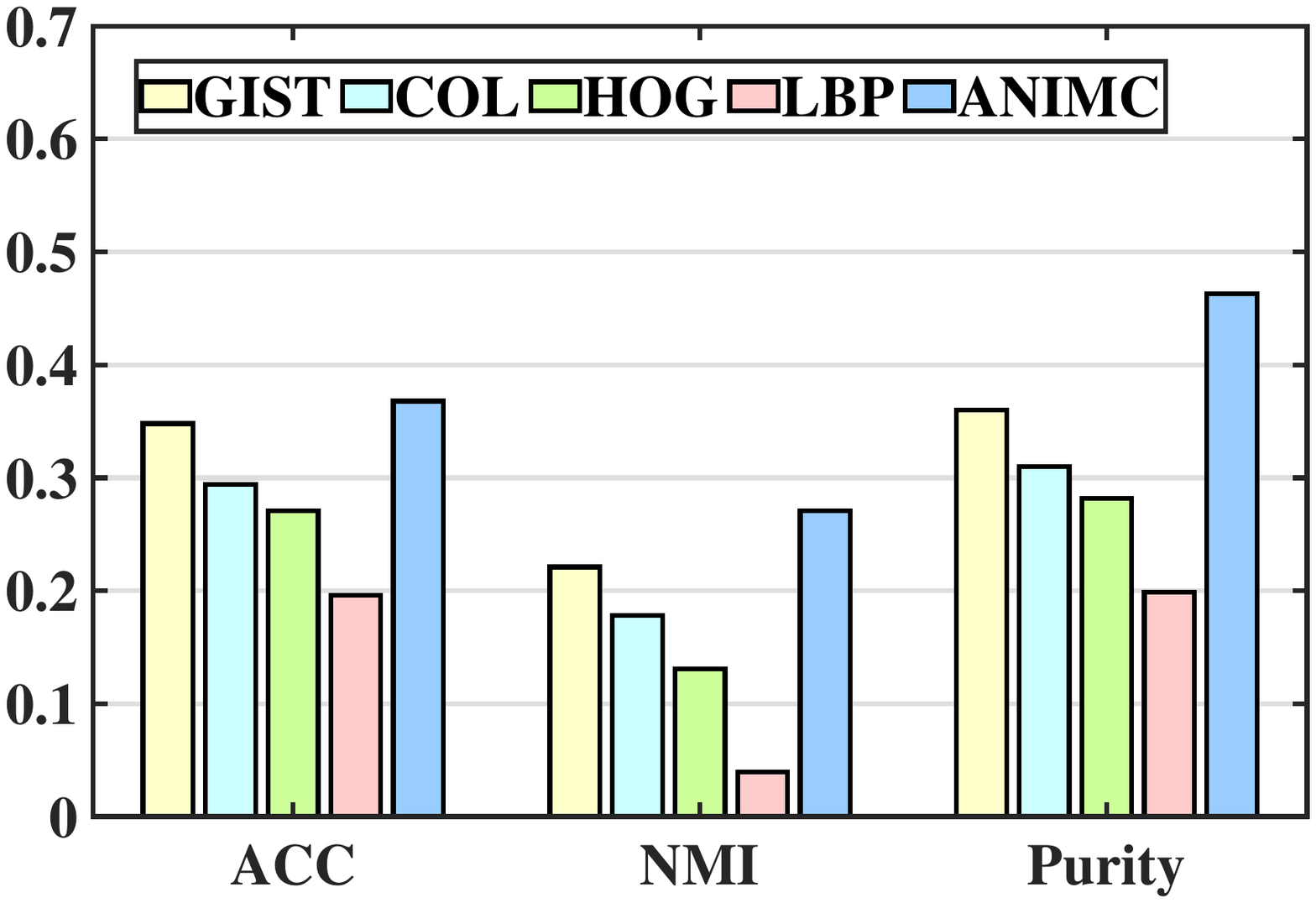}}
\caption{Availability comparison of different views for incomplete multi-view clustering on the Scene dataset.}
\label{fig:usability_scene}
\end{figure*}
To evaluate the effectiveness of our proposed ANIMC approach in real-world tasks, we conduct experiments on six real-world multi-view datasets as follows: BBCSport\footnote{\url{http://mlg.ucd.ie/datasets/segment.html}.} \cite{greene2006practical}, BUAA \cite{huang2012buaa}, Digit\footnote{\url{http://archive.ics.uci.edu/ml/datasets.html}.}, Scene \cite{monadjemi2002experiments}, Caltech7 \cite{cai2013multi}, and NH$\_$face \cite{wu2013constrained}, whose statistics are summarized in Table \ref{dataset}.
%Note that BUAA, Digit and Scene are multi-view image datasets, so we perform clustering on these image datasets to analyze the capabilities of ANIMC in image clustering. BBCSport is a text dataset and we cluster it to analyze the generalization ability of ANIMC.

Most multi-view clustering algorithms often cluster their subsets for simplicity. To compare fairly with these algorithms, we use the same subsets for clustering.
The detailed information of used datasets is as follows:

BBCSport: The original BBCSport dataset contains 737 documents (instances), which are described by 2-4 views and categorized into 5 clusters. Following \cite{wen2019unified}, we choose a subset with 544 instances and 2 views.

BUAA: The dataset is an image dataset, which is described by two views. Following \cite{wen2019unified}, we use its subset with 180 instances and 10 clusters.
%of the BUAA face database, which contains 90 visual images and 90 near infrared images of the first 10 clusters.
%All images were resized into an $10\times 10$ matrix and vectored in advance.

Caltech7: The dataset is a subset of the Caltech101 dataset. The dataset has 1474 instances consisting of 7 clusters. Each instance has 6 views.

Digit: It is a handwritten dataset. Following \cite{shao2015multiple}, we use its subset with 2000 instances, 10 clusters, and 5 views.

NH$\_$face: It is a movie dataset, which contains 4660 faces from 5 persons. Each face has 3 views.

Scene: The dataset is an outdoor Scene dataset. Following \cite{hu2020multi}, the dataset has 2688 images (instances) consisting of 8 clusters. Each image has 4 views.

\subsection{Compared Methods}\label{subsection:comp_meth}
%\noindent\textbf{4.2 Compared Methods}
We compare our proposed ANIMC approach with eleven state-of-the-art methods:
%We compare our proposed \textbf{V$^3$H} with eight state-of-the-art clustering methods:
%1) \textbf{DAIMC}~\cite{hu2018doubly} extends MIC based on weighted semi-NMF and $L_{2,1}$-Norm regularization regression;
%
%2) \textbf{IMG}~\cite{zhao2016incomplete} transforms the collected incomplete multi-view data to a complete representation in a latent space;
%
%3) \textbf{MIC}~\cite{shao2015multiple} extends MultiNMF based on weighted NMF and $L_{2,1}$-Norm regularization;
%
%4) \textbf{MultiNMF}~\cite{liu2013multi} extends NMF to multi-view scenes by jointing these views;
%
%5) \textbf{PVC}~\cite{li2014partial} aligns the same samples in  different views by constructing a latent subspace;
%
%6) \textbf{UEAF}~\cite{wen2019unified} learns a consensus representation for all views by extending MIC;
%where MultiNMF is the baseline method.
%We compare our proposed method AwIMC with the following state-of-the-art multi-view clustering methods:

1) \textbf{AGC$\_$IMC} \cite{wen2020adaptive} develops  a joint  framework  for  graph  completion  and  consensus  representation  learning.

2) \textbf{BSV}~\cite{zhao2016incomplete} first fills all missing instances in the average instance of each view, then implements K-means clustering on every single view, separately. Finally, BSV reports the best performance.

3) \textbf{Concat}~\cite{zhao2016incomplete} first combines all views into a single view by concatenating their view matrices. Then, Concat implements K-means clustering on the single view and reports the clustering results.

4) \textbf{DAIMC} \cite{hu2018doubly} extends MIC \cite{shao2015multiple} by designing weighted semi-NMF and $L_{2,1}$-norm regularized regression.

5) \textbf{EE-IMVC}~\cite{liu2019efficient} proposes a late fusion approach to simultaneously clustering and imputing the incomplete base clustering matrices.

6) \textbf{EE-R-IMVC}~\cite{liu2020efficient} improves EE-IMVC by incorporating prior knowledge to regularize the learned consensus clustering matrix.

7) \textbf{MIC} \cite{shao2015multiple} extends PVC via weighted NMF and $L_{2,1}$-norm regularization.

8) \textbf{MLAN}~\cite{nie2018auto} is a self-weighted framework for complete multi-view clustering. By learning the optimal graph, it performs clustering and local structure learning simultaneously.

9) \textbf{NMF-CC}~\cite{liang2020multi} aims to capture the intra-view diversity and learning many independent basis matrices in turn for a satisfactory clustering representation.

10) \textbf{UEAF}~\cite{wen2019unified} learns a consensus representation for all views by extending MIC.

%{\color{blue}{6) \textbf{MSGLMAIN}~\cite{kang2021structured} simultaneously considers graph structure, scalability, and out-of-sample problems by making use of the anchor idea, bipartite graph, and spectral graph property.}}

11) \textbf{UIMC} \cite{fang2021uimc} is the first effective method to cluster multiple views with different incompleteness.

Since NMF-CC and MLAN cannot directly handle incomplete multi-view data, following~\cite{hu2018doubly}, we fill the missing instances with average feature values. Note that our proposed ANIMC has two parameters ($\alpha$ and $\beta$), and we adjust them to obtain the best performance.
%For the parameter $r$, we set $r=1$ to obtain the incomplete multi-view clustering results and noisy and incomplete multi-view clustering results.

Following~\cite{li2014partial} and \cite{zhao2016incomplete}, we repeat each incomplete multi-view clustering experiment $10$ times to obtain the average performance. Following~\cite{wen2019unified}, we randomly delete some instances from each view to make these views incomplete and set the missing rate (PER) from $0.1$ to $0.5$ (each view has $50\%$ instances missing) with $0.2$ as an interval.

Following~\cite{wen2019unified}, we evaluate the experimental results by three popular metrics: Accuracy (ACC), Normalized Mutual Information (NMI), and Purity.
\cite{liang2020multi} provides the calculation method of these metrics.
For these metrics, the larger value represents better clustering performance.
All results of compared methods are produced by released codes, some of which may be inconsistent with published information due to different pretreatment processes. All the codes in the experiments are implemented in MATLAB R2019b and run on a Windows 10 machine with $3.30$ GHz E3-1225 CPU, $64$ GB main memory.

\subsection{Incomplete Multi-view Clustering}\label{subsection:result_incomp}
Table \ref{buwanzhengjuleijieguo} shows the ACC, NMI and Purity values on these real-world datasets with various PERs. As PER increases, the performance of each method often decreases.
Obviously, our proposed ANIMC significantly performs better than other state-of-the-art methods in all cases, which verifies its effectiveness and strong generalization ability. %Especially, when we cluster the BBCSport dataset with PER=0.3, ANIMC gains large improvements around $ 3.67\%$, $ 20.4\%$, and $9.89\%$ over the second-best algorithm (i.e., UIMC) in terms of ACC, NMI and Purity, respectively.
%Moreover, other state-of-the-art methods are difficult to maintain satisfactory clustering performance in all cases due to weak generalization ability. For example, although DAIMC can obtain a pretty clustering on Digit dataset, it performs poorly on the BUAA dataset.
\begin{figure*}[t]
\centering
\subfigure[Convergence Study]{\label{fig:digit_compare} \includegraphics[width=0.315\textwidth]{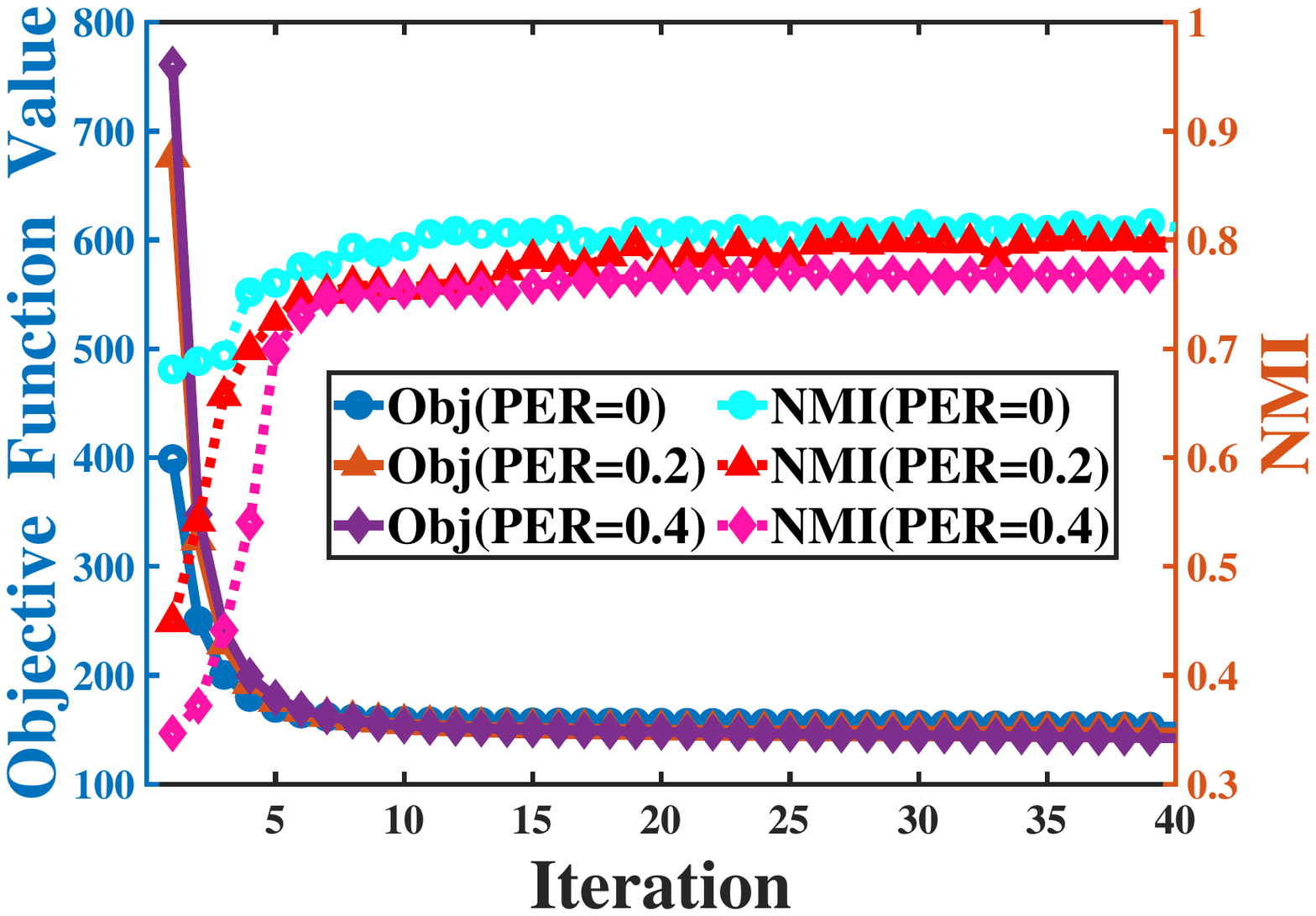}}
\subfigure[Convergence Comparison]{\label{fig:converage_com} \includegraphics[width=0.315\textwidth]{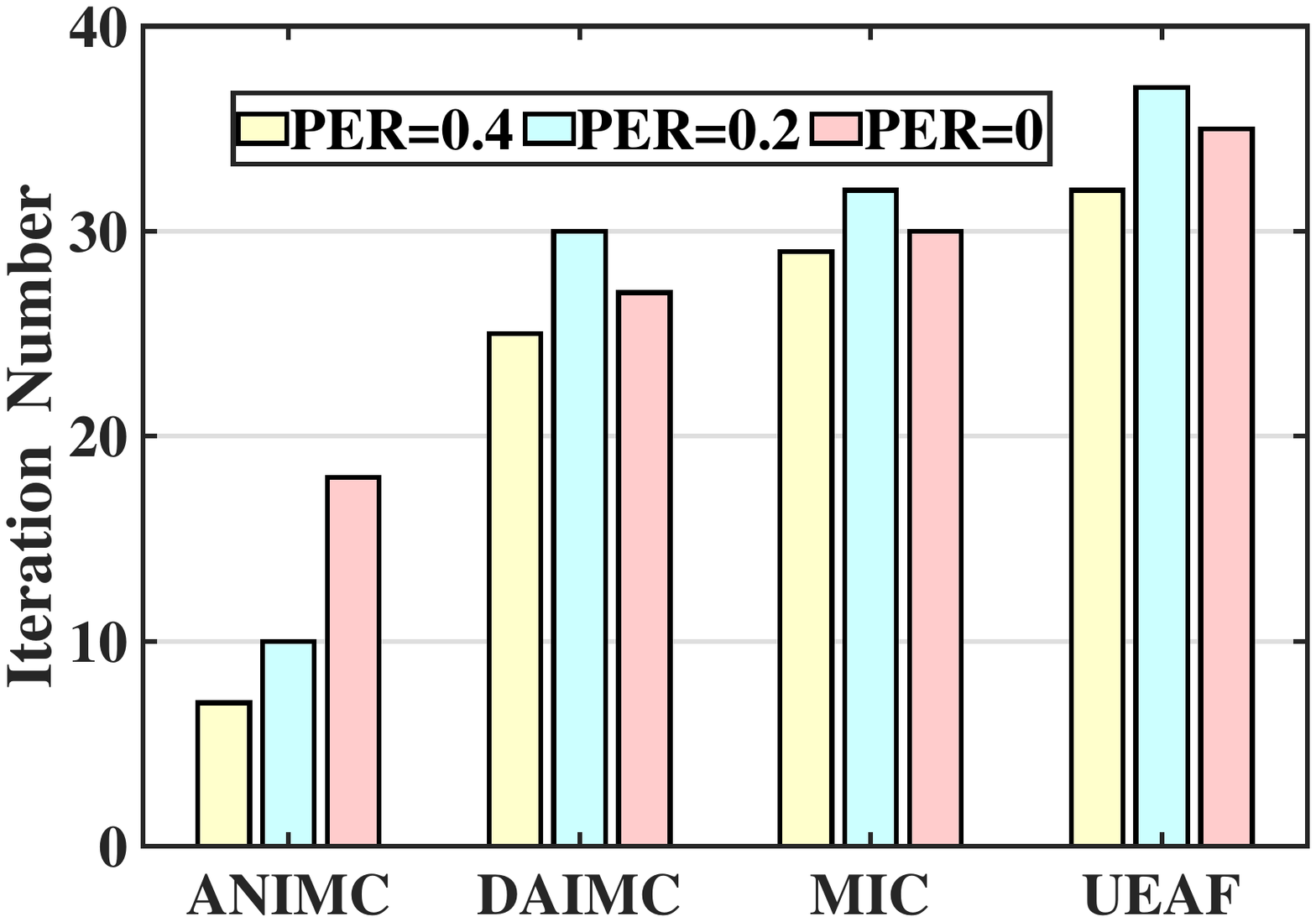} }
\subfigure[Parameter Study]{\label{fig:nmi_1} \includegraphics[width=0.315\textwidth]{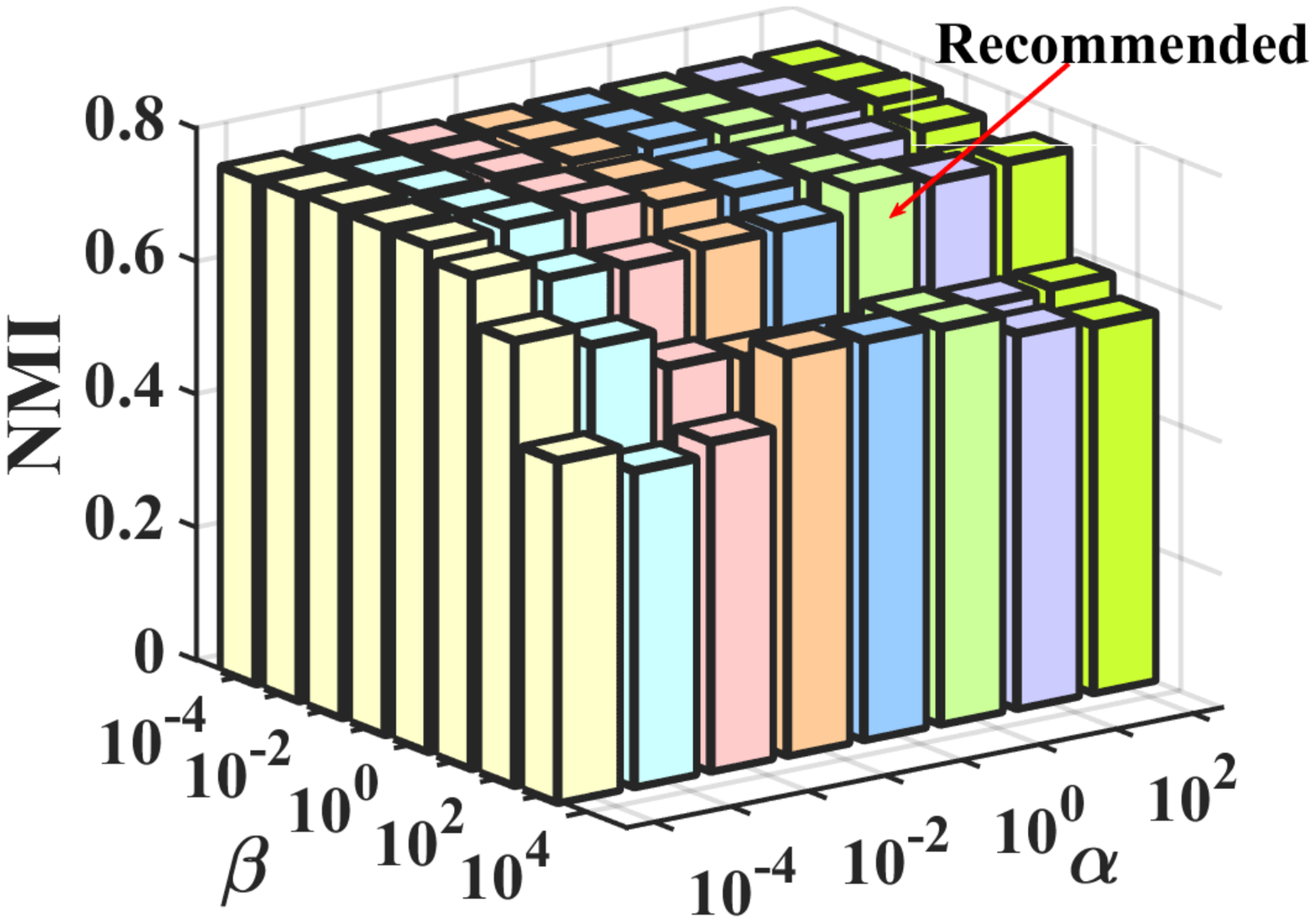} }
%%\subfigure[]{\label{fig:acc_3d} \includegraphics[width=0.23\textwidth]{ACC_3D}}
%\subfigure[Study for ($\beta=1$)]{\label{fig:nmi_1} \includegraphics[width=0.225\textwidth]{figure/b1} }
%%\subfigure[]{\label{fig:acc_3d} \includegraphics[width=0.23\textwidth]{ACC_3D}}
%\subfigure[Study for ($\beta=100$)]{\label{fig:nmi_100} \includegraphics[width=0.225\textwidth]{figure/b100} }
\caption{Convergence and parameter study on the Digit dataset.}
\label{fig:3D_b100}
\end{figure*}
\begin{table*}[ht]
\centering
\caption{Performance of our proposed Soft Boundary for noisy and incomplete multi-view clustering on the Digit dataset, where ``ANIMC(SB)'' denotes our ANIMC with Soft Boundary, ``ANIMC(woSB)'' denotes our ANIMC without Soft Boundary, and `ANIMC(SW)'' denotes our ANIMC with $w_v=1$. \textbf{Bold} numbers denote the best results.}
\scalebox{0.9}{
\begin{tabular}{c|ccc|ccc|ccccc}
\hline
%\multirow{2}*{PER}&BBC&REU&GUA
\multirow{2}*{PER}&\multicolumn{3}{c}{ACC (\%)}&\multicolumn{3}{|c|}{NMI (\%)}&\multicolumn{3}{c}{Purity (\%)}\\\cline{2-10}
~& ANIMC(SB) &ANIMC(woSB)&ANIMC(SW)& ANIMC(SB) &ANIMC(woSB)&ANIMC(SW)& ANIMC(SB) &ANIMC(woSB)&ANIMC(SW)\\\hline
0&\textbf{73.75}&72.63&68.52&\textbf{68.70}&66.09&64.26&\textbf{73.75}&72.03&70.52\\
0.1&\textbf{70.10}&69.87&67.31&\textbf{52.87}&51.74&47.96&\textbf{70.10}&68.31&65.14\\
0.2&\textbf{63.95}&60.88&52.94&\textbf{44.42}&40.48&38.92&\textbf{63.95}&60.17&57.64\\
0.3&\textbf{48.75}&44.93&41.27&\textbf{41.69}&37.61&35.98&\textbf{50.05}&47.35&45.71\\
0.4&\textbf{30.67}&25.39&21.30&\textbf{37.26}&36.94&33.02&\textbf{42.19}&40.16&39.15\\
0.5&\textbf{20.54}&13.75&11.06&\textbf{29.53}&28.55&26.69&\textbf{30.58}&28.39&26.71\\\hline
\end{tabular}
}
\label{ruanbianjieduibi}
\end{table*}

\begin{table}
\centering
\caption{Running time (seconds) of different methods for noisy and incomplete multi-view clustering on all datasets (missing rate=0.5, noise rate=0.2, and noise variance=0.1).}
\scalebox{0.9}{
\begin{tabular}{c|ccccccccc}
\hline
Method  &  BBCSport & BUAA  & Caltech7 &Digit& NH-face& Scene \\\hline
AGC$\_$IMC&26.37&43.80&54.71&38.45&107.94&28.76\\
BSV& 9.83&1.04&8.16&6.84&83.54&10.87 \\
Concat&11.27&1.63&10.21&8.74&90.49&12.36 \\
DAIMC&60.88&39.46&27.81&32.94&208.76&52.85  \\
EE-IMVC&65.53&27.31&26.42&48.59&284.76&20.48\\
EE-R-IMVC&77.24&38.66&19.53&52.87&377.40&15.83\\
MIC&50.71&16.82&31.60&40.72&384.83&37.15\\
MLAN&30.18&7.60&6.83&8.64&116.72&25.56\\
NMF-CC&73.55&10.82&8.54&35.72&120.94&33.67\\
UEAF&40.96&84.57&43.88&24.50&130.26&30.49 \\
UIMC&31.86&90.85&72.61&20.53&195.47&23.75\\
Our ANIMC&33.62&51.84&25.06&30.81&174.62&23.18  \\
\hline
\end{tabular}
}
\label{yunxingshijian}
\end{table}
%Reasons:
%1.fluctuation: the average of ACC and NMI. As the missing rate increases, the evaluation indicators of different views decrease at different speeds, which also indicates that different views should be given different weights.
%2.digit: 5 views, AwIMC and DAIMC can deal with well. Seim-NMF and the regression of basis matrix
%3.bbcsport: the regression of latent feature matrix
%4.buaa: the regression of latent feature matrix, high Incomplete Rate influence the regression.
%5.scene: 4 view, low Incomplete Rate the weight (0-0.3) view do not get a good weight in DAIMC; get a good view in AwIMC. High Incomplete Rate the weight (0.3-0.5) influence learning weight.
%%

%(1) Due to similar performance (ACC and NMI) in the four datasets, our comparison methods can be divided into the following three groups: {PVC, IMG}, {MIC, DAIMC}, {Multi-NMF}. This phenomenon can be explained by these algorithm structure. IMG extends PVC by considering the global structure in the latent space. DAIMC improves MIC with the help of weighted Semi-NMF and regression theory. Multi-NMF can only solve incomplete multi-view clustering problems in specific experimental settings, and it is difficult to apply directly to real-world applications.
%(2) On Digit dataset ($5$ views) and Scene dataset ($4$ views), DAIMC, MIC and our AwIMC are significantly superior to several other algorithms.
%This shows that in the task of incomplete multi-view clustering, the double-view clustering algorithm still has limitations because it ignores the information of some views.
%\paragraph{BBCSport:}
\noindent\textbf{BBCSport:} As shown in Table \ref{buwanzhengjuleijieguo},
compared with the other methods, MLAN often becomes the first worst. For example, when PER=0.5, compared with the other methods, MLAN reduces the performance at least about $1.98\%$ in ACC, $0.28\%$ in NMI and $2.35\%$ in Purity, respectively. The main reason is that although it can automatically learn the proper weights for each view from a complete multi-view dataset, MLAN does not effectively integrate incomplete views. Therefore, the self-weighting strategy in MLAN cannot be directly used for incomplete multi-view clustering.
%It is because PVC simply projects instances from each view into a common subspace, which overlooks the global information among the two views. Since each view of the BBCSport dataset has a large number of features, if the global information is not used effectively, it is difficult to obtain satisfactory clustering results. Thus, when the missing rate is small, the clustering results of PVC are still poor.
Obviously, our proposed ANIMC outperforms all the other methods significantly for various missing rates. Specifically, relative to DAIMC, when the missing rate is relatively small (PER=0.1), ANIMC improves ACC by at least $25.33\%$, NMI by at least $37.10\%$ and Purity by at least $18.86\%$; when the missing rates become larger (PER=0.5), ANIMC still improves ACC by at least $14.16\%$, NMI by at least $13.97\%$ and Purity by at least $16.52\%$.
One reason for ANIMC's outstanding performance is that each view of the BBCSport dataset has many features, and ANIMC can effectively integrate these features by minimizing the disagreement between the common latent feature matrix and the common consensus.
%{\color{blue}{This also shows that it is unreliable to only perform $L_{2,1}$-norm regularized regression for .}}
%Compared with the state-of-the-art incomplete multi-view clustering UEAF,

%the BBCSport dataset contains 5 views. AwIMC effectively uses information from different views, reduces the impact of missing samples, and obtains better experimental results.
%\paragraph{BUAA:}
\noindent\textbf{BUAA:} In the BUAA dataset, DAIMC and UEAF may obtain worse clustering results than other methods.
For instance, when PER=0.3, compared with DAIMC and UEAF, NMF-CC raises the performance at least about $1.15\%$ in ACC, $1.14\%$ in NMI and $2.50\%$ in Purity, respectively. This is mainly because the BUAA dataset contains a large number of noises that cause the learned basis matrix $\bm{U}^{(v)}$ to deviate from the true value, thereby hurting the clustering performance of DAIMC and UEAF.
ANIMC performs much better than the other methods for all various missing rates. When PER=0.5, ANIMC raises ACC by at least $4.63\%$, NMI by at least $3.79\%$ and Purity by at least $7.35\%$, relative to the compared methods. The reason is that our proposed adaptive semi-RNMF model can assign a smaller weight to the view with more noises (larger deviate for $\bm{U}^{(v)}$) to efficiently reduce the impact of noises.
%, ANIMC still obtains satisfactory clustering results.
%The possible reason for this phenomenon is similar to that in BBCSport dataset.
%; when the incomplete rate becomes too large (e.g., 0.4), the performance of AwIMC is almost the same as PVC.
%\paragraph{Digit:}

\noindent\textbf{Caltech7:} Compared with EE-IMVC and EE-R-IMVC, both BSV and Concat perform unsatisfactorily. For example, when PER=0.3, Concat decreases the performance at least about $29.42\%$ in ACC, $17.59\%$ in NMI and $30.86\%$ in Purity, respectively. It is because Concat simply concatenates these views, which makes it difficult to learn consistent information between views.

\noindent\textbf{Digit:} As PER increases, the clustering performance of MIC drops significantly.
Because MIC only simply fills the missing samples with average feature values, which will result in a serious deviation for the dataset with high incompleteness. Thus, this simply filling cannot effectively solve the incomplete multi-view clustering problem.
Obviously, our proposed ANIMC performs much better than other methods in all cases. Especially, when PER=0.1, compared with AGC$\_$IMC and UIMC, ANIMC improves ACC by at least $2.85\%$, NMI by at least $2.47\%$ and Purity by at least $2.68\%$. It is because ANIMC assigns a weight to each view for satisfactory performance.
%Secondly, ANIMC and DAIMC have close performance, and the difference between them is less than $1\%$. The reason is that Digit dataset is sparse and the sparse models ($L_{2,1}$-norm regularized regression in DAIMC and doubly soft regularized regression in ANIMC) can learn the effective information of the dataset for clustering. This close performance also shows that doubly soft regularized regression is not inferior to $L_{2,1}$-norm regularized regression.
%ANIMC inherits the strength of DAIMC with ability of regularized regression.
%\paragraph{Scene:}

\noindent\textbf{NH-face:} For the dataset, our proposed ANIMC outperforms other methods. For instance, when PER=0.5, ANIMC raises the performance at least about $11.16\%$ in ACC, $10.85\%$ in NMI and $15.14\%$ in Purity, respectively. It is because NH-face has a large number of instances and features, and ANIMC can effectively process these instances and features.

\noindent\textbf{Scene:} Obviously, our proposed ANIMC outperforms the other methods in all cases.
When PER=0.1, ANIMC raises ACC by at least $7.30\%$, NMI by at least $4.05\%$ and Purity by at least $6.04\%$. It is because ANIMC assigns a proper weight to each view thereby decreasing the influence of noises.

%\noindent\textbf{4.4 Noisy and Incomplete Multi-view Clustering}
\subsection{Noisy and Incomplete Multi-view Clustering}\label{subsection:result_noisy}
Gaussian noise is one of the most common noises in real-world applications, and many algorithms often use this noise to analyze their de-noising capabilities~\cite{yuan2017ell,elhoseiny2017link}. In this section, we use the Digit, BUAA, Caltech7 datasets. We add Gaussian noise (noisy rate=0.2, noisy variance=0.1) to these datasets. Before adding noises, we normalize each view matrix of the dataset. Since our proposed ANIMC is based on matrix factorization, we compare ANIMC with NMF-CC, MIC, DAIMC and UEAF, which are also based on matrix factorization.

Fig.~\ref{fig:AwIMC_noisy} shows the noisy and incomplete multi-view clustering results. Obviously, ANIMC significantly outperforms other methods in all cases.
%Note that both ANIMC and NMF-CC often obtain better clustering results than MIC, DAIMC and UEAF in most cases. It is because both ANIMC and NMF-CC try to align the latent feature matrix towards the consensus, which verifies the effectiveness and necessity of the doubly soft regularized regression model (in Section~\ref{subsection:doubly}).
For the Digit dataset with PER=0, compared with these methods, ANIMC raises the clustering performance around $13.52\%$ in ACC and $10.63\%$ in NMI. The outstanding performance of ANIMC verifies its ability to deal with both noises and incompleteness because our adaptive semi-RNMF model can balance the impact of noises and incompleteness based on soft auto-weighted strategy (in Section~\ref{subsection:adaptive}). By adaptively assigning a proper weight to each noisy and incomplete view, ANIMC can distinguish the availability of different views for satisfactory performance.

In Table~\ref{yunxingshijian}, we report the running time of different methods on all the datasets with the missing rate of 0.5, the noise rate of 0.2, and the noise  variance of 0.1. We can find that 1) BSV spends the least running time. It is because different from other methods, BSV only deals with the features in one view, which reduces its running time. 2) All methods spend the most running time in the NH-face dataset. The reason is that the dataset has more features and instances, which will increase the running time. 3) In most cases, our proposed ANIMC spends less running time than state-of-the-art UIMC, which shows the efficiency of ANIMC.

Besides, we seek the influence of exponential function with different $r$, shown in Fig.~\ref{fig:rbianhua}. When we scan it in the whole range (from 0 to 2), the clustering performance keeps a high-level value. Therefore, $r$ is insensitive to the missing rate (PER). Note that when $r=0$ (i.e., our adaptive semi-RNMF model is removed $w_v||\bm{G}^{(v)}\odot(\bm{X}^{(v)}-\bm{U}^{(v)}\bm{V}^T)||_F^2=0$), the clustering performance is the worst, which validates the importance of the adaptive semi-RNMF model. When $0.2\leq r\leq 1.3$, ANIMC can obtain relatively good performance. In our experiments, we choose $r=0.2$; for $\theta$, we set $\theta=0.01$ for $||\bm{V}||_{\theta}$ and $\theta=100$ for $||\bm{A}^{(v)}||_{\theta}$.

In summary, the performance of all the methods is analyzed as follows. Based on the auto-weighted strategy, MLAN can perform well in complete multi-view clustering tasks. But MLAN is difficult to be directly extended to incomplete multi-view clustering because missing instances will cause the graphs learned by MLAN to be unavailable. As the missing rate increases, the clustering results of both MIC and NMF-CC drop significantly because MIC and NMF-CC neglect the hidden information of the missing instances. Both DAIMC and UEAF rely too much on alignment information. When clustering the dataset without enough alignment information, they always obtain unsatisfactory clustering results because the loss of alignment information will reduce their availability.
%they learn the consensus representation by aligning these views.
These drawbacks make these methods difficult to be widely used in real-world applications. By assigning a proper weight to each view via soft auto-weighted strategy and learning the global structure via doubly soft regularized regression model, our proposed ANIMC can obtain satisfactory performance in most cases, which shows that its potential to cluster different real-world datasets well.
%In summary,
%When the incomplete rate is small ($0$ to $0.3$), DAIMC outperforms IMG and PVC. The main reason is that Scene contains 4 views and DAIMC can uses information from different views to improve clustering results.
%In summary, when dealing with text data or multi-view data that contains less alignment information, IMG often gets poor result. Meanwhile, although MIC can handle the clustering problem with more than two-views, simply filling the missing instances with the global feature average will lead to a deviation, especially when the incomplete rate is large. By utilizing the information of instance alignment and enforcing the alignment among basis matrices, the proposed ANIMC can get better performances no matter whether it is text dataset or image dataset. Especially when the number of views is large, ANIMC yields more better results.
%\subsection{Weight, Convergence and Parameter Study}
%\paragraph{Weight Study:}

%\noindent\textbf{4.5 Weight Study}
\subsection{View Availability Study}\label{subsection:availability}
%Since our proposed ANIMC is the first auto-weighted noisy and incomplete multi-view clustering framework, it is necessary to analyze the effects of noises and incompleteness on the view availability.
To comprehensively analyze the availability of different views, we compare the incomplete multi-view clustering performance of each view on two datasets (Digit and Scene) with different PERs. We compare our proposed ANIMC (multi-view clustering based on ANIMC) with single-view clustering based on ANIMC (e.g., ``FAC" in Fig.~\ref{fig:usability} denotes the clustering performance using ANIMC on view FAC.). The results of the comparison are shown in Fig.~\ref{fig:usability} and Fig.~\ref{fig:usability_scene}.

Obviously, ANIMC obtains better clustering performance in all cases. For example, when PER=0.5, ANIMC improves ACC by at least $42.90\%$, NMI by at least $38.80\%$ and Purity by at least $42.01\%$, which shows that ANIMC can effectively integrate different views for clustering. Moreover, different views have distinct clustering performance on single-view clustering. For instance, KAR and PIX have higher availability than FAC and ZER (see Fig.~\ref{fig:usability}).
%Besides, as PER increases, the clustering results of these views decline differently. When clustering the Digit dataset, compared with FAC, FOU can obtain better performance on PER=0 but get worse clustering results on PER=0.5 (see Fig.~\ref{fig:usability}).
Besides, as PER changes, distinct views have different relative availability. When clustering the Digit dataset, compared with FAC, FOU can obtain better performance on PER=0 but get worse clustering results on PER=0.5 (see Fig.~\ref{fig:usability}). The comparison of relative availability also shows that we are difficult to choose proper parameters to weight views, which illustrates the significance of our proposed soft auto-weighted strategy (in Section~\ref{subsection:adaptive}).

\subsection{Soft Boundary Study}\label{subsection:boundary}
To verify the effectiveness of our soft boundary, we perform experiments to verify the effectiveness of our proposed soft boundary. In this section, we will compare the performance of our proposed ANIMC with the soft boundary (represented by ANIMC(SB)) and without the soft boundary (represented by ANIMC(woSB)). For ANIMC without soft boundaries, we use Eq.~\eqref{wgengxin} to learn the view weight $w_v$. For ANIMC with soft boundaries, we use Eq.~\eqref{wzuixin} to obtain $w_v$. To better demonstrate the effectiveness of our soft auto-weighted strategy, we also compared a popular strategy that all views have the same weight (i.e., $w_v=1$, represented by ANIMC(SW)). We compared them on the Digit dataset, and report the comparison results in Table~\ref{ruanbianjieduibi}.

As shown in Table~\ref{ruanbianjieduibi}, ANIMC(SW) performs the worst, while ANIMC(SB) obtains the best clustering results in all cases. It is because ANIMC(SW) assigns all views the same weight, which cannot distinguish the importance and availability of each view. This also illustrates the importance of distinguishing the availability of each view. Although ANIMC(woSB) can also distinguish the importance of different views, it cannot measure the impact of noise and incompleteness on each view. Therefore, this method may not be able to obtain satisfactory clustering performance.
%\noindent\textbf{4.6 Convergence Study}
\subsection{Convergence Study and Parameter Sensitivity}\label{subsection:convergence}
By perform incomplete multi-view clustering on the Digit dataset, we study the convergence with different PERs, i.e., PER=0, PER=0.2 and PER=0.4. We set the parameters ${\alpha,\beta}$ as $0.1, 100$, respectively. Also, we set the maximum iteration time $i_{\max}=40$.
Fig.~\ref{fig:digit_compare} shows that the convergence curve and the NMI values versus the iteration number, where ``Obj'' denotes ``objective function value''. Note that our proposed ANIMC has converged just after $20$ iterations and NMI achieves the best at the convergence point for all PERs.
Fig.~\ref{fig:converage_com} shows the number of iterations when ANIMC, DAIMC, MIC and UEAF converge. ANIMC converges faster than the others. The reason is that ANIMC can make the falling direction of the objective function close to the gradient direction by adaptively learning a proper weight for each view (in Eq.~\eqref{w}), which shows the efficiency of our proposed soft auto-weighted strategy (in Section~\ref{subsection:adaptive}). Especially, when PER increases, the convergence of  ANIMC is accelerated. This is because as PER increases, each incomplete view is assigned a larger weight $w_v$, which increases the difference of the objective function value between before and after each iteration (i.e., the declining rate of the objective function curve increases). Therefore, the objective function can converge to a stable value faster as PER increases.
%\subsection{Parameter Sensitivity}\label{subsection:parameter}

For different datasets, it is difficult to adaptively select the optimal values for these parameters. Thus, we consider a simple method to obtain the optimal combination of two parameters for experiments.
In terms of $\{\alpha, \beta\}$ used in our model, we conduct the parameter experiments on the Digit dataset.
Similar to \cite{wen2019unified}, we set PER=0.3 and report the clustering performance of ANIMC versus $\alpha$ and $\beta$ within the set of $\{$$10^{-4}$, $10^{-3}$, $10^{-2}$, $10^{-1}$, $10^{0}$, $10^{1}$, $10^{2}$, $10^{3}$, $10^{4}$$\}$.
As shown in Fig.~\ref{fig:nmi_1}, ANIMC not only achieves excellent clustering results but also is insensitive to these parameters. Moreover, ANIMC obtains a relatively good performance when $\alpha=0.1$ and $\beta= 10$. Therefore, we give the recommended values ($\alpha=0.1$ and $\beta= 10$) and set the recommended values for all the datasets.
%Taking $\frac{obj^{(i)}-obj^{(i-1)}}{obj^{(i)}}\leq 10^{-6}$($obj^{(i)}$ is the value of the objective function obtained from the $i$-th iteration) as the convergence condition, we compare the convergence of these algorithms in four data sets.
%\begin{figtab}[t]
%\begin{minipage}[b]{0.45\linewidth}
%\centering
% \includegraphics[width=0.23\textwidth]{ACC_digit_compare}}
% \figcaption{Result on Digit}
% \end{minipage}\quad
%
% \begin{minipage}[b]{0.45\linewidth}
%    \centering
%    \tabcaption{Example Tabular}
%    \begin{tabular}{cc}
%    \hline
%      Method & Iteration for convergence \\
%      AwIMC & 10 \\
%      DAIMC & 30 \\
%      MIC & 10 \\
%    \end{tabular}
%  \end{minipage}
%\subfigure[Result on Scene]{\label{fig:scene_compare} \includegraphics[width=0.23\textwidth]{scene_value_compare} }
%
%
%\caption{Convergency study on Digit dataset.}
%% \setlength{\abovecaptionskip}{1in}
%\label{fig:compare}
%\end{figtab}

%\noindent\textbf{4.7 Hyper-parameter Study}

%\begin{figure}[t]
%\centering
%\subfigure[Result on Digit]{\label{fig:digit_compare} \includegraphics[width=0.23\textwidth]{NMI_value}}
%\subfigure[Converge Comparison]{\label{fig:converage_com} \includegraphics[width=0.23\textwidth]{converage} }
%\caption{Convergency study on Digit dataset.}
%% \setlength{\abovecaptionskip}{1in}
%\label{fig:compare}
%\end{figure}

\section{Conclusion and Future Work}  \label{section:con}
In this paper, we propose a novel and soft auto-weighted approach ANIMC for noisy and incomplete multi-view clustering.
ANIMC consists of a soft auto-weighted strategy and a doubly soft regularized regression model. On the one hand,
the soft auto-weighted strategy is designed to automatically assign a proper weight to each view. On the other hand, the doubly soft regularized regression model is proposed to align the same instances in all views.
%To reduce the effects of noises, ANIMC assigns proper weights to noisy and incomplete views by learning the optimal common latent feature matrix automatically. Moreover, doubly regularized regression is designed to align the basis matrices of different view and push the common latent feature matrix towards its consensus, which decreases the influence of missing instances.
Extensive experiments on four real-world multi-view datasets demonstrate the effectiveness of ANIMC.

%Our proposed ANIMC is based on machine learning.
In the future, we will incorporate our approach with deep learning algorithms.
By the incorporation, we can effectively deal with the missing instances and noises in a large-scale recommendation system.
% for incomplete multi-view clustering on large-scale data with noises.

%\appendices
%% you can choose not to have a title for an appendix
%% if you want by leaving the argument blank
%\section{}
%% use section* for acknowledgment
%\section*{Acknowledgment}

\ifCLASSOPTIONcaptionsoff
  \newpage
\fi
\bibliographystyle{IEEEtran}
\bibliography{IEEEtran}
\end{document}